\documentclass{article}

\usepackage[numbers,sort&compress]{natbib}
\usepackage[preprint]{neurips_2026}

\usepackage[utf8]{inputenc}
\usepackage[T1]{fontenc}
\usepackage{bm}
\usepackage{hyperref}
\usepackage{url}
\usepackage{booktabs}
\usepackage{amsfonts}
\usepackage{nicefrac}
\usepackage{microtype}
\usepackage{xcolor}
\usepackage{amsmath}
\usepackage{amssymb}
\usepackage{mathtools}
\usepackage{amsthm}
\usepackage{float}

\usepackage{enumitem}

\usepackage{caption}
\usepackage{subcaption}
\usepackage{csquotes}
\MakeOuterQuote{"}

\theoremstyle{plain}
\newtheorem{theorem}{Theorem}[section]

\newtheorem{lemma}[theorem]{Lemma}
\newtheorem{corollary}[theorem]{Corollary}
\theoremstyle{definition}

\newtheorem{remark}[theorem]{Remark}
\newtheorem{example}[theorem]{Example}
 
\title{Learnability and Competition in High-Dimensional Multi-Component ICA}

\author{%
  \textbf{Eser İlke Genç\textsuperscript{1}} \hspace{1.5cm} 
  \textbf{Samet Demir\textsuperscript{1}} \hspace{1.5cm} 
  \textbf{Zafer Doğan\textsuperscript{1,2}\thanks{Corresponding author}} \\[0.3cm]
  \textsuperscript{1}MLIP Research Group, KUIS AI Center, Koç University \quad
  \textsuperscript{2}Department of EEE, Koç University
 \\
  \texttt{\{egenc25,sdemir20,zdogan\}@ku.edu.tr}
}

\usepackage{amsmath}
\usepackage{amssymb}
\usepackage{mathtools}
\usepackage{amsthm}

\usepackage[capitalize,noabbrev]{cleveref}

\usepackage{float}

\usepackage{enumitem}

\usepackage{caption}
\usepackage{subcaption}
\usepackage{csquotes}

\usepackage{amsmath,amsfonts,bm}

\def\eqref#1{(\ref{#1})}

\def\1{\bm{1}}

\DeclareMathAlphabet{\mathsfit}{\encodingdefault}{\sfdefault}{m}{sl}
\SetMathAlphabet{\mathsfit}{bold}{\encodingdefault}{\sfdefault}{bx}{n}

\begin{document}

\maketitle

\begin{abstract}
Independent Component Analysis (ICA) is a foundational tool for unsupervised representation learning, yet its high-dimensional theory remains largely limited to single-component recovery. We develop an asymptotically exact mean-field theory for multi-component online ICA, capturing the coupling induced by simultaneous learning and orthogonalization. In the high-dimensional limit, the joint empirical distribution of learned estimates and ground-truth components converges to a deterministic process, yielding a closed ODE system for the overlap matrix between learned directions and true components. This characterization reveals a genuinely multi-component, initialization-driven phase structure: a \emph{decoupled regime}, where estimates align with distinct components and evolve nearly independently, and a \emph{competition regime}, where overlapping initializations induce orthogonality-driven conflicts, slow reorientation, and delayed convergence. Our steady-state analysis gives explicit learnability boundaries and competition conditions linking step size, data moments, and initialization. These conditions show that larger higher-order moments and competition shrink the stable learning-rate window, increase convergence times, and predict a staircase phenomenon in which the number of recoverable components changes discretely with the learning rate. Experiments on synthetic data and hyperspectral remote sensing data validate the predicted trajectories and phase behavior.
\end{abstract}

\section{Introduction}
Independent Component Analysis (ICA) is a foundational \emph{feature learning} framework for recovering latent independent components from linear mixtures \cite{fransizICA1985,comon1994independent, amari1995new, hyvarinen2000independent}. After whitening, second-order statistics are removed by centering and decorrelating the data, so the remaining identifiable structure is necessarily \emph{non-Gaussian}. As a result, ICA relies on nonlinear score functions that extract information from \emph{higher-order} moments \cite{ bell1995information,fastica,hyvarinen_oja_2001_book}. In modern streaming settings, ICA is naturally implemented by stochastic gradient updates \cite{ablin2019stochastic, li2021stochastic}, placing it within the broader theory of high-dimensional stochastic gradient descent (SGD) \cite{Saad1995Dynamics,wang2017scaling,JMLR:v21:19-245,paquette21Large,arous2023highdimensionallimittheoremssgd, jagannath2026highdimensional, collins2024hitting}. Understanding ICA in this regime is therefore not only a classical problem in unsupervised learning, but also a testbed for developing exact theories of high-dimensional online learning.

Despite substantial progress, a rigorous high-dimensional theory of online ICA remains largely limited to the single-component setting. Existing analyses either track the evolution of a single vector \cite{li2016online,wang2017ica} or study the recovery of one non-Gaussian direction \cite{li2021stochastic,ricci2025}. A prior high-dimensional analysis of multi-component online ICA studied transient trapping in a particular Hebbian algorithm \cite{basalyga2003}, but did not provide a full-trajectory limit for coupled online learning dynamics. This leaves open a central question: \emph{what are the exact high-dimensional dynamics when multiple ICA components are learned simultaneously?} The question is not a straightforward extension of the single-component case. Practical ICA estimates many directions at once, and orthogonalization couples their updates. Consequently, the learning dynamics depend not only on the marginal evolution of each component, but also on competition between estimates, their initialization geometry, the step size, and the higher-order moments of the data.

This paper develops an asymptotically exact high-dimensional theory of multi-component online ICA. We prove that projected SGD converges, in the high-dimensional limit, to a deterministic mean-field PDE governing the orthogonality-coupled trajectory. Our first contribution is an exact mean-field characterization of the joint training dynamics. We introduce the joint empirical measure over the coordinates of the learned and ground-truth vectors and prove that, as the ambient dimension diverges, this empirical measure converges to a deterministic limit governed by a partial differential equation (PDE). A key consequence is that the overlap matrix between the learned directions and the true components satisfies a closed system of ordinary differential equations (ODEs). These ODEs provide a finite-dimensional, tractable description of the macroscopic learning trajectories.

The limiting ODEs reveal a phase structure that is invisible in single-component analyses. Depending on the initialization, the dynamics separate into two qualitatively different regimes: \emph{decoupling} and \emph{competition}. In the decoupling regime, orthogonality constraints are inactive at leading order: each estimate aligns with a distinct component, and the macroscopic dynamics mimic independent single-component learning. In the competition regime, by contrast, multiple estimates initially pursue the same component. Orthogonalization then creates conflicts between the estimates, forcing reorientation, slowing convergence, and increasing sensitivity to the step size and higher-order moments. Thus, multi-component ICA exhibits genuinely collective behavior that cannot be reduced to parallel copies of the single-component theory.

Beyond transient dynamics, our ODEs characterize steady states and phase transitions. We derive explicit \emph{learnability boundaries} relating the learning rate, initialization, and higher-order moments, giving the precise functional dependence of the maximum stable learning rate on the data moments. For the canonical case of cubic nonlinearities, these boundaries admit closed-form expressions involving fourth- and sixth-order moments that determine when nontrivial fixed points exist. The resulting theory predicts a \emph{staircase-like} phenomenon: as the learning rate increases, the number of recoverable components drops discretely. In the competition regime, the staircase persists but shifts to smaller permissible learning rates, quantifying the additional instability induced by competition.

Finally, we validate the theory both synthetically and on real data. In simulations with controllable higher-order moments, Monte Carlo trajectories closely match the mean-field ODEs in both the decoupling and competition regimes, confirming the asymptotic predictions at finite dimension. We further evaluate the theory on hyperspectral remote sensing data from the Indian Pines dataset \cite{Baumgardner15}, which is naturally modeled as a linear mixture of non-Gaussian components corresponding to distinct land-cover classes \cite{lupu2022stochastic}. Despite the moderate dimensionality of this dataset, the ODE predictions accurately track the empirical learning dynamics and explain the observed recovery behavior.

Overall, our contributions are as follows:

\begin{enumerate}[nosep, leftmargin=*]

\item We give an exact mean-field theory for the \textit{joint training dynamics of multiple ICA components}: the joint empirical measure converges to a deterministic weak-form PDE, and the component overlaps obey a closed ODE system.

\item We reveal an initialization-dependent phase structure: a \textit{decoupled regime}, where estimates align with distinct components and evolve nearly independently, and a \textit{competition regime}, where overlapping initializations induce orthogonality-driven conflicts and slow learning.

\item We derive explicit steady-state \textit{learnability boundaries} and \textit{competition conditions} in terms of data moments, learning rate, and initialization, showing when larger higher-order moments or competition force smaller learning rates and longer convergence times.
\end{enumerate}

\paragraph{Notation.} We denote conditional expectation with respect to the filtration $\{\mathcal{F}_k\}_{k\ge 0}$ by $\mathbb{E}_k[\cdot] := \mathbb{E}[\cdot \mid \mathcal{F}_k]$, where $\mathcal{F}_k$ is the $\sigma$-algebra generated by all randomness up to iteration $k$, including the initialization and the samples $\{\bm{y}_s\}_{s < k}$. For a matrix $\bm{A}_t \in \mathbb{R}^{m \times n}$ indexed by iteration $t$, we denote its $(i,j)$-th entry by $A_{t,i,j}$, and its $i$-th row and $j$-th column by $\bm{A}_{t,i,:}$ and $\bm{A}_{t,:,j}$, respectively. We let $\mathrm{diag}(\bm{A}_t)$ be the diagonal matrix whose diagonal entries coincide with those of the matrix $\bm{A}_t$ or $\mathrm{diag}(\boldsymbol{\upsilon}_t)$ for the diagonal matrix formed by the entries of the vector $\boldsymbol{\upsilon}_t$. We define the operator $\mathcal{T}(\bm{A}) \coloneqq \operatorname{tril}(\bm{A} + \bm{A}^\top) - \operatorname{diag}(\bm{A})$, where $\operatorname{tril}(\cdot)$ extracts the lower-triangular matrix by zeroing out all strictly upper-triangular entries of $\bm{A}$. The first derivative of a function $f$ is denoted $f^{\prime}$. Finally, $\text{sgn}(x)$ is used for the sign function and $\operatorname{Tr}(\cdot)$ is the trace operator for matrices.

\section{Related work}

\paragraph{Independent component analysis.}
ICA originated in blind source separation \cite{fransizICA1985} and was formalized through independence, mutual information, and cumulant-based contrast functions \cite{comon1994independent}, including fourth-order cumulant-tensor methods \cite{frieze1996learninglinear}. This line of work led to Infomax and neural learning rules \cite{linskerinfomax1988,bell1995information}, online mutual-information and natural-gradient algorithms \cite{amari1995new,amari1998natural}, higher-order algebraic and cumulant-based contrasts \cite{DELFOSSE199559,Cardoso1999HighOrderCF}, and negentropy-based fixed-point methods, most notably FastICA \cite{hyvarnien1997difentropy,hyv_oja_fastfixed,HYVARINEN1998,hyvarnien1997_familyof,fastica,hyvarinen2000independent, hyvarinen_oja_2001_book}. Despite this extensive literature, the high-dimensional streaming theory of ICA remains largely restricted to the single-component setting: prior works establish diffusion limits \cite{li2016online}, finite-sample guarantees \cite{li2021stochastic}, scaling limits without orthogonalization-induced coupling \cite{wang2017ica}, and recent analyses of FastICA and online SGD for recovering one non-Gaussian direction \cite{ricci2025}. Complementary batch results characterize computational--statistical tradeoffs in large-dimensional ICA \cite{auddy2023large}, but do not describe online learning trajectories. The closest multi-component dynamical analysis \cite{basalyga2003} studies an online Hebbian ICA algorithm \cite{HYVARINEN1998} locally near metastable fixed points. In contrast, we provide a full-trajectory high-dimensional limit for orthogonality-coupled multi-component online ICA, yielding an empirical-measure PDE and closed overlap ODEs that capture both transient competition and steady-state learnability.

\paragraph{High-dimensional dynamics of SGD.}
Our work is also connected to recent advances in deriving closed dynamical descriptions of stochastic gradient methods in high dimensions. Depending on the model and scaling, these descriptions take the form of deterministic ODE or PDE systems for order parameters, homogenized stochastic processes, or empirical-measure limits \cite{paquette2022homogenizationsgdhighdimensionsexact,veiga2022phase,collins2024hitting,arous2023highdimensionallimittheoremssgd,wang2017scaling}. Empirical-measure approaches trace back to the mean-field theory of interacting particle systems \cite{mckean1967propagation,snitzman1991propcgaos}, and have recently been used to analyze SGD and online-learning dynamics in linear regression \cite{balasubramanian2025high,li2025statistical}, streaming PCA and subspace tracking \cite{wang2016online,demir2025implicitly,balzano2018,wang2018}, contrastive learning \cite{meng2024training}, high-dimensional GANs \cite{NEURIPS2019_6b3c49bd,bond2024}, and neural-network training \cite{Goldt_2020,songmei_meanfieldview_twolayer,sirignano2019mean,arnaboldi2023high}. We bring this perspective to multi-component ICA, where the central challenge is resolving the orthogonalization-induced interaction between learned directions. This interaction produces an initialization-dependent transition between decoupled learning and competition, a phenomenon that cannot arise in the single-component setting.

\section{Problem setting}
\subsection{Data model}

We consider a high-dimensional online ICA model with ambient dimension $n$ and a fixed number $p$ of latent components, with $p$ finite as $n\to\infty$. This scaling captures settings in which the number of sources is moderate while the observations are high-dimensional \cite{Calhoun2001,lupu2022stochastic}. At each iteration $k$, we observe
\begin{align}\label{data_generation}
\bm{y}_k=\frac{1}{\sqrt{n}}\bm{U}\bm{c}_{k}+\bm{a}_k ,
\end{align}
where $\bm{U}=[\bm{u}_1,\dots,\bm{u}_p]\in\mathbb{R}^{n\times p}$ contains the unknown components to be recovered, and $\bm{c}_k=(c_{k,1},\dots,c_{k,p})^\top\in\mathbb{R}^p$ contains the latent source coefficients. The component vectors are deterministic, mutually orthogonal, and normalized as $\|\bm{u}_i\|=\sqrt{n}$. The coefficients $\{c_{k,i}\}$ are independent across both $i$ and $k$, have zero mean and unit variance, and are non-Gaussian; their distributions may differ across components and are distinguished by higher-order moments.

The Gaussian term $\bm{a}_k\sim\mathcal{N}(\bm{0},\bm{I}-\tfrac{1}{n}\bm{U}\bm{U}^\top)$ lies in the orthogonal complement of the signal subspace and serves only to whiten the observations. In particular, $\mathbb{E}_k[\bm{y}_k\bm{y}_k^\top]=\bm{I}$, so second-order statistics are uninformative about $\bm{U}$. Consequently, the components are identifiable only through the higher-order non-Gaussian structure of the latent coefficients. This construction extends the standard single-spike high-dimensional ICA model to a finite multi-component setting while preserving the whitening assumption central to classical ICA.

\subsection{Online ICA via projected SGD}

We estimate the components with a strictly online projected-SGD algorithm, processing one sample at a time without mini-batching. Let $\bm{X}_k\in\mathbb{R}^{p\times n}$ collect the estimate vectors row-wise, with rows $\bm{x}_{k,i}$ satisfying $\|\bm{x}_{k,i}\|=\sqrt{n}$ for all $i\in\{1,\dots,p\}$. Given $\bm{y}_k$, we first take the SGD step
\begin{equation}\label{eq:online_sgd}
    \tilde{\bm{X}}_k
    =
    \bm{X}_k
    +
    \frac{\tau}{\sqrt{n}}\, f\Big(\frac{1}{\sqrt{n}}\bm{X}_k\bm{y}_k\Big)\, \bm{y}_k^{\top} - \frac{\tau}{n}\phi(\bm{X}_k),
\end{equation}
where $f$ is applied entrywise to the normalized projections $\frac{1}{\sqrt{n}}\bm{X}_k\bm{y}_k$ and extracts higher-order non-Gaussian structure, e.g., $f(x)=\pm x^3$ for symmetric sources and $f(x)=\pm x^2$ for asymmetric sources. The map $\phi$ is applied entrywise and is the gradient of the regularizer $\Phi(x)=\int \phi(x)\,dx$, allowing structural priors such as sparsity to enter the dynamics; $\tau$ is the learning rate. We then project back onto the row-orthogonal constraint,
\begin{equation}\label{eq:seq_GS}
    \bm{X}_{k+1} = \mathrm{Orthogonalize}(\tilde{\bm{X}}_k),
\end{equation}
where $\mathrm{Orthogonalize}(\cdot)$ returns mutually orthogonal rows, each with norm $\sqrt{n}$. This projection is the source of the multi-component coupling: even though the stochastic-gradient step acts row-wise, orthogonalization makes each estimate interact with the others.

\begin{remark}
For concreteness, we analyze Gram--Schmidt orthogonalization \cite{gram_schmidt1,schmidt1907theorie}. The resulting decoupling and competition regimes, however, are not artifacts of this choice: they persist under both sequential and symmetric orthogonalization schemes, as shown in Appendix \ref{app:sym_orth}.
\end{remark}

\section{Main results}
We begin by defining the central object of our analysis: the joint empirical measure of the learned components and the ground-truth feature vectors. For each coordinate index \(\alpha \in \{1,\ldots,n\}\), denote the stacked estimate  $\bar{\bm{x}}_{k}^{(\alpha)} \coloneqq \bm{X}_{k,:,\alpha}$ and ground-truth vectors $\bar{\bm{u}}^{(\alpha)} \coloneqq \bm{U}_{\alpha,:}$. Then, we define the joint empirical measure \(\mu_k^n\) as a probability measure on
\(\mathbb{R}^p \times \mathbb{R}^p\) by
\begin{equation}
\mu_k^n(\bm{x},\bm{u})
:=
\frac{1}{n}
\sum_{\alpha=1}^n
\delta\!\left(\bm{x}-\bar{\bm{x}}_{k}^{(\alpha)},\,\bm{u}-\bar{\bm{u}}^{(\alpha)}\right),
\label{eq:mu_stacked_dirac}
\end{equation}
where \(\delta(\bm{x}-\bar{\bm{x}}_{k}^{(\alpha)},\bm{u}-\bar{\bm{u}}^{(\alpha)})\) denotes the Dirac measure at \((\bar{\bm{x}}_{k}^{(\alpha)},\bar{\bm{u}}^{(\alpha)})\).
The random measure \(\mu_k^n\) takes values in \(\mathcal{M}(\mathbb{R}^{2p})\), the space of probability measures on \(\mathbb{R}^{2p}\). Intuitively, \(\mu_k^n\) encodes the joint empirical law of the coordinates of the current estimates and the ground truth.
This object serves as a convenient device: many macroscopic quantities of interest can be expressed as expectations with respect to the empirical measure \(\mu_k^n\). The following theorem characterizes the time evolution of the empirical measure, which is the main theoretical result of our analysis.
\begin{figure}[t]
  \centering
  \begin{subfigure}[t]{0.30\textwidth}
    \centering
    \includegraphics[width=\linewidth]{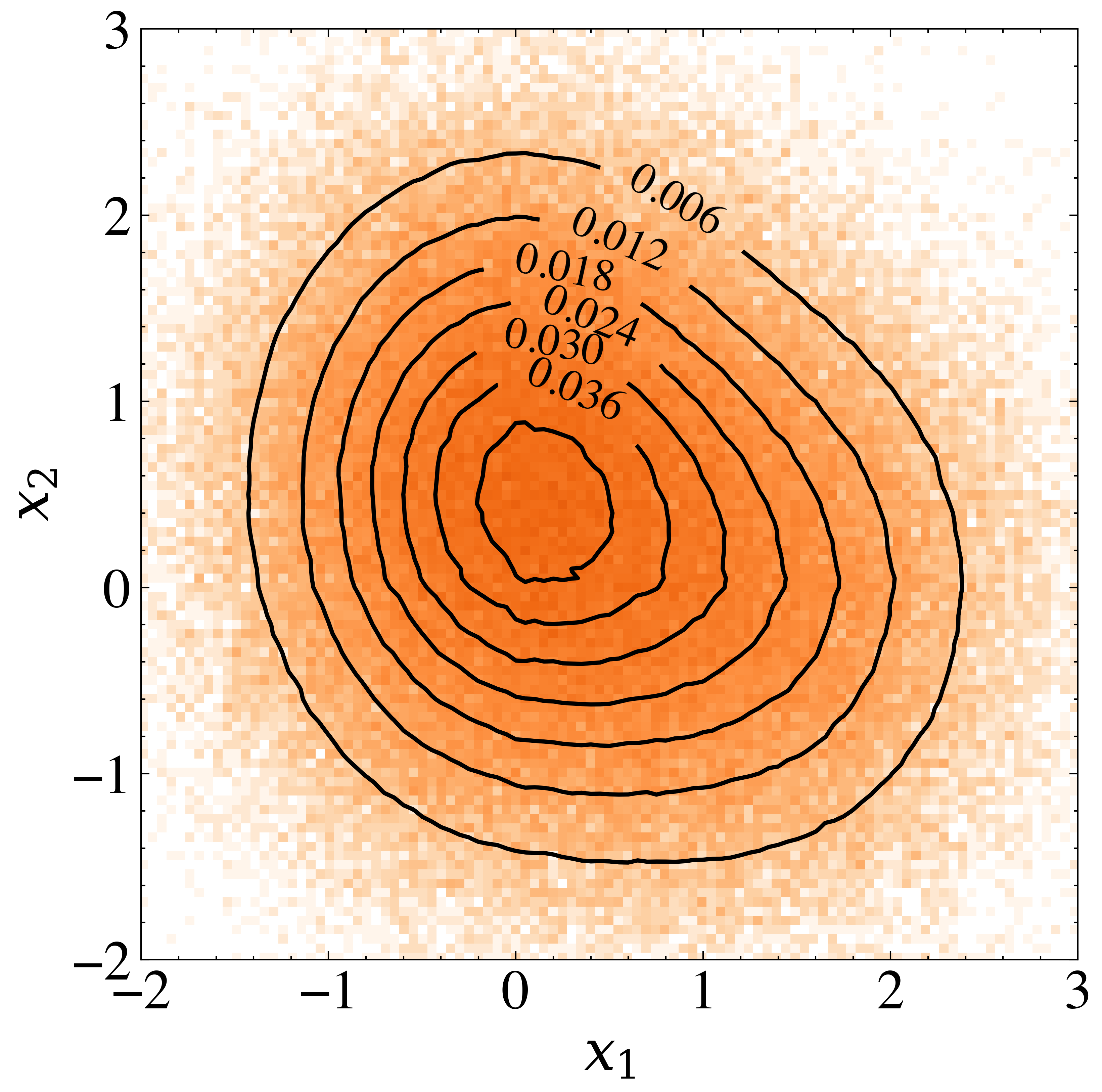}
    \caption{$t=0$}
    \label{fig:1a}
  \end{subfigure}\hfill
  \begin{subfigure}[t]{0.30\textwidth}
    \centering
    \includegraphics[width=\linewidth]{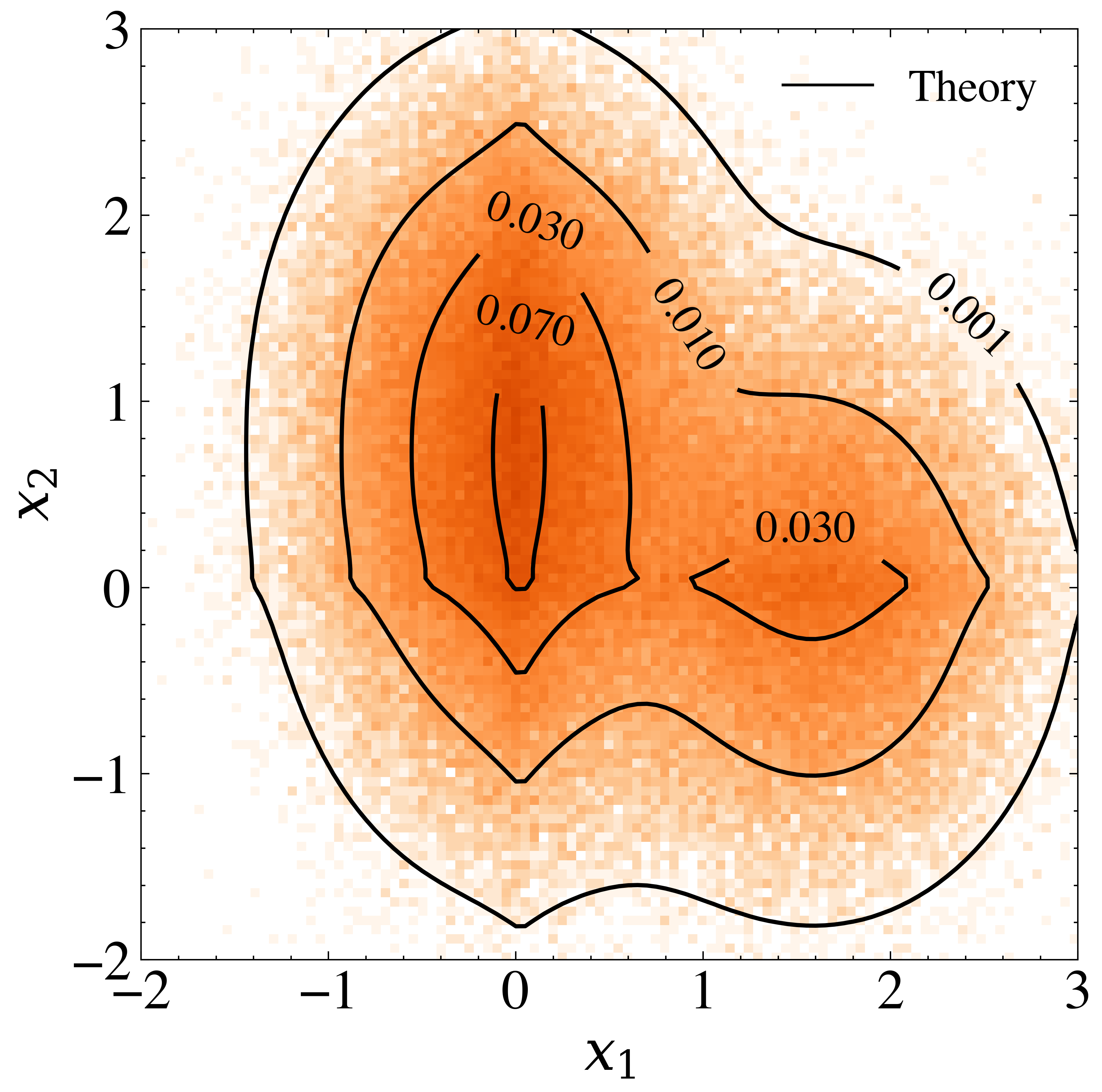}
    \caption{$t=100$}
    \label{fig:1b}
  \end{subfigure}\hfill
  \begin{subfigure}[t]{0.385\textwidth}
    \centering
    \includegraphics[width=\linewidth]{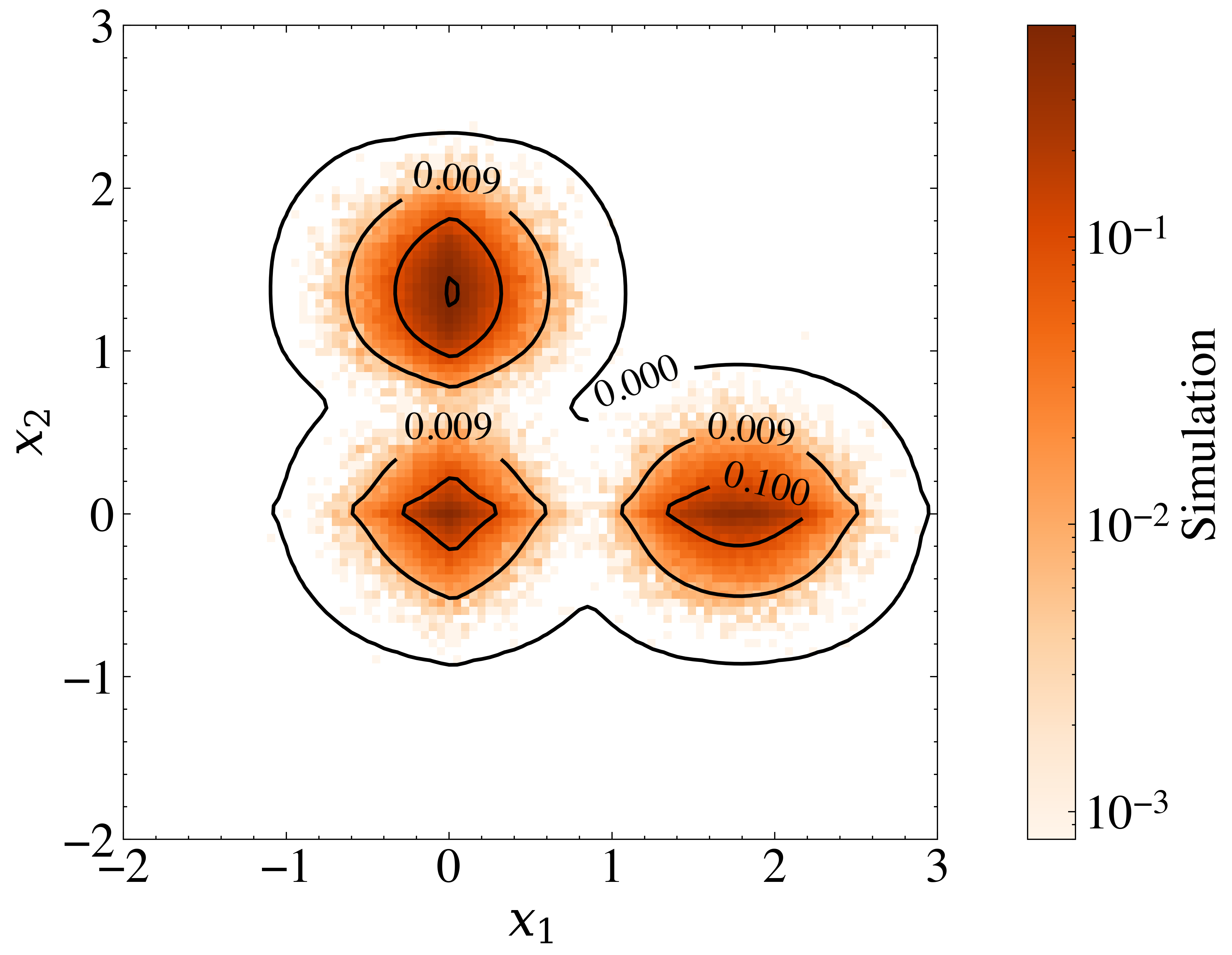}
    \caption{$t=200$}
    \label{fig:1c}
  \end{subfigure}
  \caption{Evolution of the joint probability limiting density for $p=2$. Comparison between theoretical predictions (contours) and Monte Carlo simulations (heatmaps) for $P_t(x_1,x_2,u_1,u_2)$ at times $t=0,100,200$. Here, we use the setting of Example \ref{example:cubic}. We set $\beta_1=1$ for $c_1$ while $\beta_2$ is set to $0$ for $c_2$. In simulations, component vectors $\bm{u}_1$ and $\bm{u}_2$ are drawn from sparse distributions: $\mathbb{P}(u=1/\sqrt{\rho_i}) = \rho_i$ and $\mathbb{P}(u=0) = 1-\rho_i$ with sparsity levels $\rho_1 = 0.5$ and $\rho_2 = 0.3$. This choice enables a clear visualization of the convergence behavior, with the components evolving toward three distinct regions in the $(\mathbf{x}_1,\mathbf{x}_2)$ space. We employ $\phi(x) = 0.1\,\operatorname{sgn}(x)$, corresponding to $L_1$ regularization commonly used in practice to induce sparsity. Here, $n=5000$, and $\tau = 0.01$.}
  \label{figure1}
\end{figure}
\begin{theorem}
\label{theorem}
Suppose that the initial empirical measure $\mu_0^n$ converges weakly to a deterministic measure $\mu_0$ as $n \rightarrow \ \infty$, the function $f$ and its first derivative are at most of polynomial growth, and $\phi$ is Lipschitz. We further assume that all moments of  $\bm{u}_i$ and $\bm{c}_k$ are bounded, as well as the moments of the initial estimates $\bm{x}_{0,i}$. Then, as $n \xrightarrow{} \infty$, the empirical measure process $\{\mu_{\lfloor tn\rfloor}^n\}_{t\ge 0}$ converges weakly to a deterministic limiting measure-valued process $\{\mu_t\}_{t\ge 0}$. 

Using the limiting process, we define 
\begin{equation}
\bm{Q}_t
\coloneqq
\int_{\mathbb{R}^{2p}} \bm{x}\,\bm{u}^{\top}\,\mu_t\mathrm{d}\bm{x}\,\mathrm{d}\bm{u}, \qquad \bm{R}_t \coloneqq \int_{\mathbb{R}^{2p}} \bm{x}\phi(\bm{x})^\top \mu_t\mathrm{d}\bm{x}\,\mathrm{d}\bm{u}.
\end{equation}
Then, for any bounded test function $\varphi$ that is three-times continuously differentiable on $\mathbb{R}^{2p}$, the limiting dynamics can be written in the weak integral form
\begin{equation}
\label{eq:weak_pde_main}
\begin{aligned}
\langle \varphi,\mu_t\rangle
=
\langle \varphi,\mu_0\rangle
+ \int_0^t
\Big\langle
\nabla_{\bm{x}}\varphi^{\top}\,\boldsymbol{\omega}_s, \mu_s \Big\rangle\,\mathrm{d}s + \int_0^t\frac{1}{2}\,\Big\langle\mathrm{Tr}\!\left(\mathbf{\Lambda}_s\,\nabla_{\bm{x}}^2\varphi\right),
\mu_s
\Big\rangle\,\mathrm{d}s,
\end{aligned}
\end{equation}
where $\langle \varphi,\mu_t\rangle := \int \varphi(\bm{x},\bm{u})\,\mu_t \,\mathrm{d}\bm{x}\,\mathrm{d}\bm{u}$.
The time-dependent coefficients $\boldsymbol{\omega}_t$ and $\mathbf{\Lambda}_t$ can be written in compact form as
\begin{align}
\bm{\omega}_t
&=
- \frac{\tau^2}{2}
\, \mathcal{T}(\bm{C_t})
\bm{x}  
- \tau \mathcal{T}(\bm{M_t})\bm{x} + \tau \mathcal{T}(\bm{R_t})\bm{x} + \tau\ \bm{\Psi}_t^\top\ \bm{u} - \tau \phi(\bm{x}), \\
\bm{\Lambda}_t &= \tau^2 \bm{C}_t,
\end{align}
where
\begin{align}
\bm{\Psi}_{t} &\coloneqq \mathbb{E}_{\bm{c}_t,\bm{e}} \Big[\bm{c}_t f(\bm{\upsilon}_{t})^\top - \bm{Q}_{t}^\top \mathrm{diag}(f^\prime(\bm{\upsilon}_{t})) \Big],  \ \quad 
\bm{M}_t \coloneqq \bm{Q}_t\bm{\Psi}_t, \ \quad \bm{C}_{t} \coloneqq \mathbb{E}_{\bm{c}_t, \bm{e}} \Big[ f(\bm{\upsilon}_{t})f(\bm{\upsilon}_{t})^\top \Big], \nonumber\\
\bm{\upsilon}_{t} &\coloneqq \bm{Q}_{t} \bm{c}_t+ \Big(\mathrm{diag}(\bm{I} - \bm{Q}_t \bm{Q}_t^\top)\Big)^{1/2} \bm{e}, \quad  \text{with}\quad \bm{e} \sim\mathcal N(\bm{0},\bm{I}) \nonumber 
\end{align}
and $\bm{c}_t$ denotes the mean-field scaled $\bm{c}_k$ in \eqref{data_generation}. 
\begin{proof}
We first derive drift $\boldsymbol{\omega}_t$ and diffusion $\mathbf{\Lambda}_t$ coefficients by computing the first two moments of the incremental step, conditioned on the filtration $\mathcal{F}_k$. After establishing exchangeability and the necessary uniform bounds, we show convergence to the deterministic process as $n\rightarrow\infty$. Explicit coefficient derivations and the formal proof are provided in Appendices~\ref{app:driftanddiffusion} and \ref{app:proof-of-theorem}, respectively.
\end{proof}
\end{theorem}
In the theorem, the used time-embedding $k=\lfloor tn\rfloor$ with $t\ge 0$ and the considered assumptions are consistent with the prior mean-field literature \cite{wang2017ica, wang2017scaling}. Furthermore, we note that the Lipschitz assumption on $\phi$ is a byproduct of our proof technique, whereas our experiments show that our results remain accurate beyond the Lipschitz class. Theorem \ref{theorem} yields the following two corollaries, which are useful to track the time-evolutions of the joint density $P_t$ and overlap matrix $\bm{Q}_t$, respectively.
\begin{corollary}
\label{corollary:strong_form}
If $\mu_t$ admits a density $P_t(\bm{x},\bm{u})$ with respect to Lebesgue measure on $\mathbb{R}^{2p}$, then $P_t$ satisfies the strong form of the PDE:
\begin{align}
\frac{\partial P_t}{\partial t}
=
-\nabla_{\bm{x}}\big(\boldsymbol{\omega}_t\,P_t\big)
+ \frac{1}{2}\,\mathrm{Tr}\!\left(\mathbf{\Lambda}_t\,\nabla_{\bm{x}}^2 P_t\right).
\end{align}
\end{corollary}
\begin{figure}[t]
  \centering
  \begin{subfigure}[t]{0.32\textwidth}
    \centering
    \includegraphics[width=\linewidth]{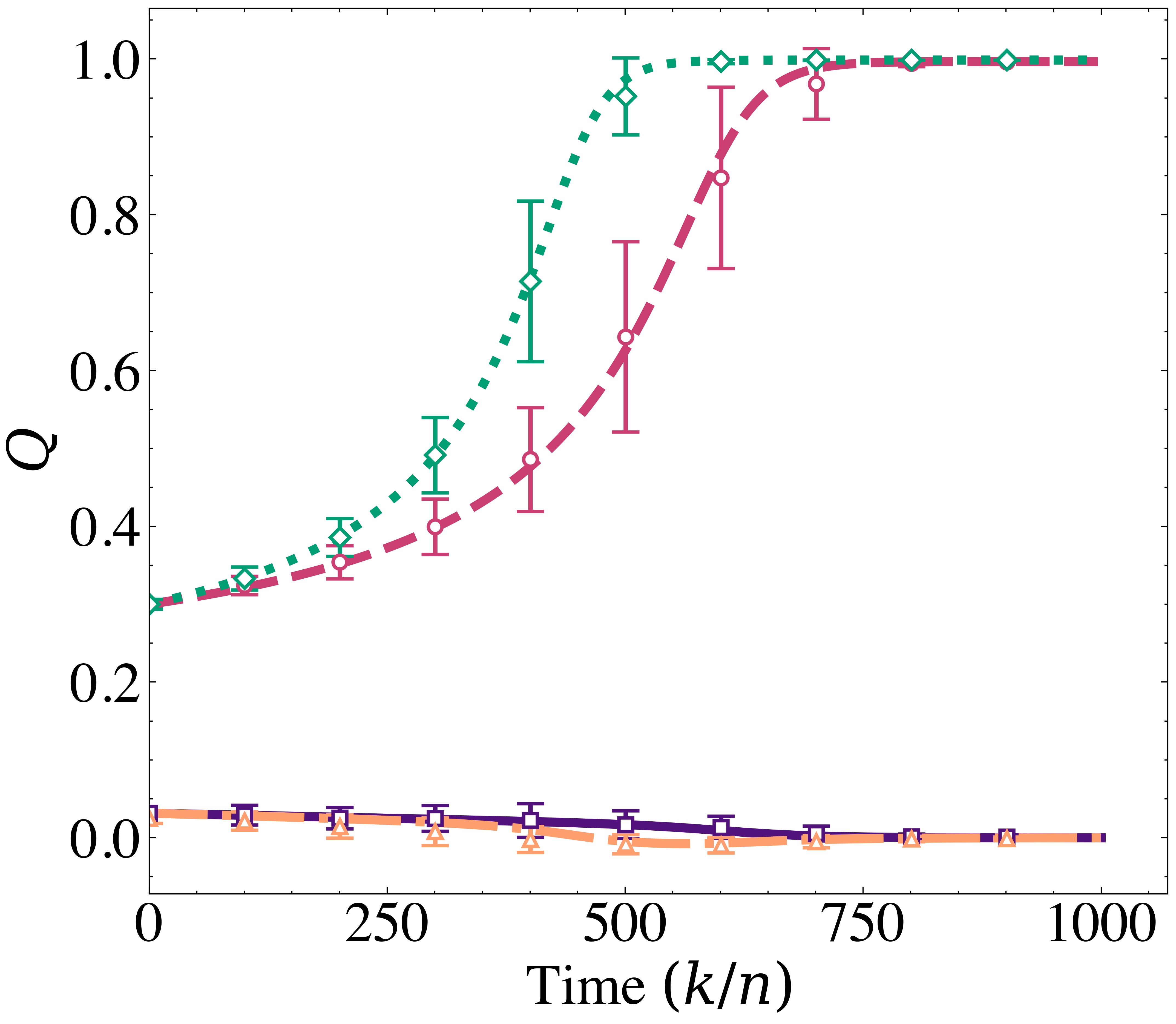}

    \caption{Decoupled alignment dynamics of estimates to true components. } 
    \label{fig:2a}
  \end{subfigure}\hfill
  \begin{subfigure}[t]{0.32\textwidth}
    \centering
    \includegraphics[width=\linewidth]{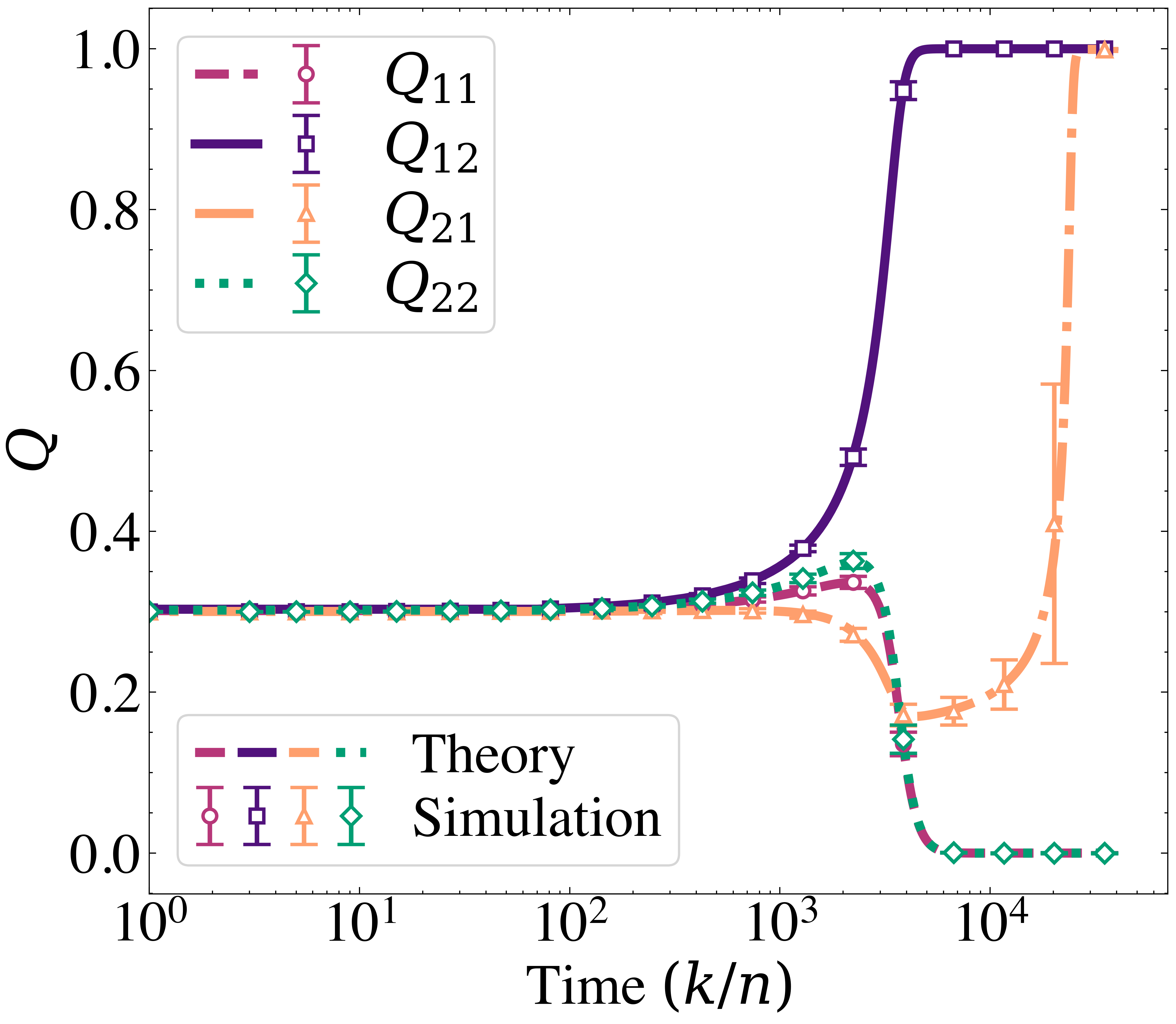}
    \caption{Competition between different estimates for the same component.}
    \label{fig:2b}
  \end{subfigure}\hfill
  \begin{subfigure}[t]{0.32\textwidth}
    \centering
    \includegraphics[width=\linewidth]{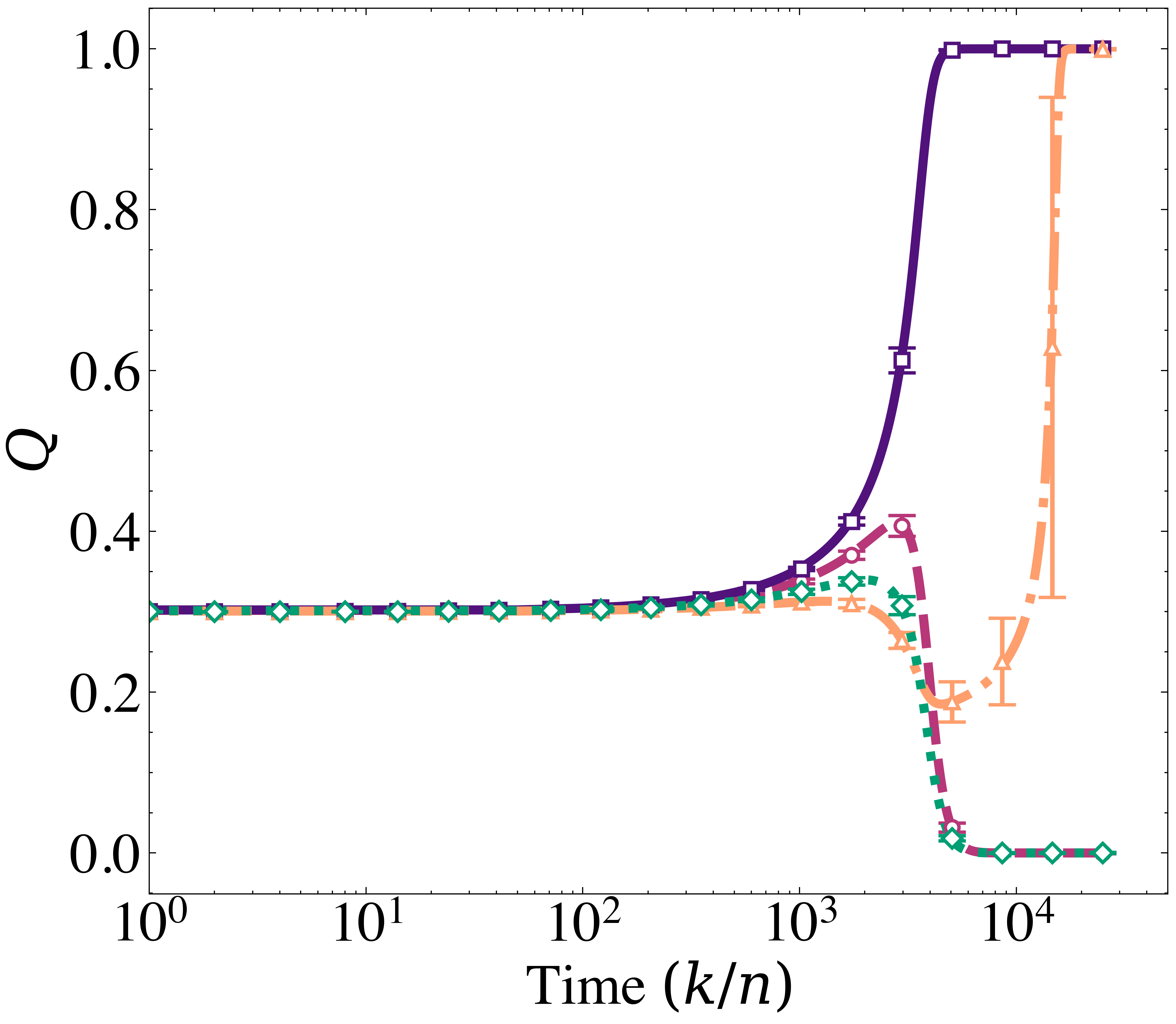}
    \caption{One of the estimates is trying to learn different components.}
    \label{fig:2c}
  \end{subfigure}
  \caption{Learning dynamics in decoupled and competition regimes. Lines denote ODE predictions, and markers denote Monte Carlo simulations with error bars indicating two standard deviations across 20 runs for the overlap matrix $\bm{Q}_t \in \mathbb{R}^{2 \times 2}$. In (a), initialization with $Q_{0,i,i}=0.3$ yields decoupled dynamics at $\tau=0.01$, $\beta_1 = 0.95$, $\beta_2 = 1$. In (b), competitive regime with $Q_{0,i,j}=0.3$; $Q_{1,2}$ and $Q_{2,2}$ compete for the second component, $Q_{1,2}$ wins, at $\tau=0.001$, $\beta_1 = 0.2, \beta_2 = 1$. Finally, in (c), we use the same initialization as (b) but different higher-order moments: first estimate competes across both components, $Q_{1,2}$ wins, at $\tau=0.001$, $\beta_1 = 0.95, \beta_2 = 1$.}
  \label{fig:learning_curves}
\end{figure}

Figure~\ref{figure1} demonstrates that the strong-form PDE in Corollary \ref{corollary:strong_form} accurately tracks the empirical evolution of the joint density in simulations. Marginal density plots are deferred to  Appendix~\ref{marginaldensities}.

In the setting we consider, the overlap matrix $\bm{Q}_t$ captures all the learning dynamics of interest. For example, it can be used to answer which row of our estimate matrix $\bm{X}_k$ learns which component. Therefore, in the rest of this paper, we analyze the learning dynamics through the lens of $\mathbf{Q}_t$.

\begin{corollary}\label{corrolaryode}
For $\phi(x)=0$, the weak formulation above immediately yields a closed evolution for the overlap matrix by choosing the appropriate test functions. In particular, the macroscopic dynamics of the limiting overlap matrix $\bm{Q}_t$ form a closed system of ordinary differential equations:
\begin{equation}
\begin{aligned}
\frac{d\bm{Q}_t}{dt}\,
&=
-\frac{\tau^2}{2}
\mathcal{T}(\bm{C}_t)\bm{Q}_t
-\tau
\mathcal{T}(\bm{M}_t)\bm{Q}_t + \tau\,\boldsymbol{\Psi}_t^\top,
\end{aligned}
\label{eq:Q_ODE}
\end{equation}
where $\bm{C}_t$, $\bm{M}_t$, and $\bm{\Psi}_t$ are defined in Theorem \ref{theorem}.
\end{corollary}

\begin{example}
\label{example:cubic}
While Theorem \ref{theorem} holds for any nonlinearity satisfying the stated assumptions, we focus on $f(x) = \pm x^3$ and $p = 2$ as an analytically tractable example.  Furthermore, as the dynamics depend explicitly on high-order moments, we adopt a data model used in the literature \cite{Gultekin2025Learning} that allows precise control of higher-order statistics through a single parameter. Specifically, the data samples  $ \bm{y}_k$ are generated using coefficients of the form:
$
    c_{k,i} = \beta_i \varepsilon_{k,1} + \sqrt{1-\beta_i^2}\, \varepsilon_{k,2},
$
where $\varepsilon_{k,1} \sim \mathrm{Rademacher}$ and $\varepsilon_{k,2} \sim \mathcal{U}(-\sqrt{3},\sqrt{3})$ for $\beta_1, \beta_2 \in [0,1]$.
\end{example}
\begin{remark}
 In the setting of Example \ref{example:cubic}, the ODE \eqref{eq:Q_ODE} for $\bm{Q}_t$ can be written explicitly in terms of the moments of $\bm{c}_t$. The resulting system is derived in Appendix \ref{app:derivationforcube}.
\end{remark}

Figure \ref{fig:learning_curves} confirms a strong agreement between our theoretical ODE \eqref{eq:Q_ODE} in Corollary \ref{corrolaryode} and empirical simulations across all evaluated initial conditions. Furthermore, this figure highlights two distinct regimes, defined by the initial overlap matrix $\bm{Q}_0$ as follows. 
\paragraph{Regime I (decoupled dynamics).}

If each row and column of the initial overlap matrix $\bm{Q}_0$ has a single dominant $\mathcal{O}(1)$ entry, with all others negligible, then each estimator locks onto one ground-truth component. Consequently, $\bm{Q}_t$ becomes permutationally diagonal, and the coupled $p\times p$ mean-field dynamics reduce, to leading order, to $p$ independent one-dimensional ODEs.
\paragraph{Regime II (competition dynamics).}

If some row or column of $\bm{Q}_0$ contains multiple $\mathcal{O}(1)$ entries, Gram--Schmidt coupling remains active at leading order and the $p\times p$ ODE system does not decouple. Multiple estimators may chase the same component, or one estimator may chase multiple components, producing competition, reorientation, and slower transients.

\begin{figure}[t]
  \centering
  \begin{subfigure}[t]{0.245\textwidth}
    \centering
    \includegraphics[width=\linewidth]{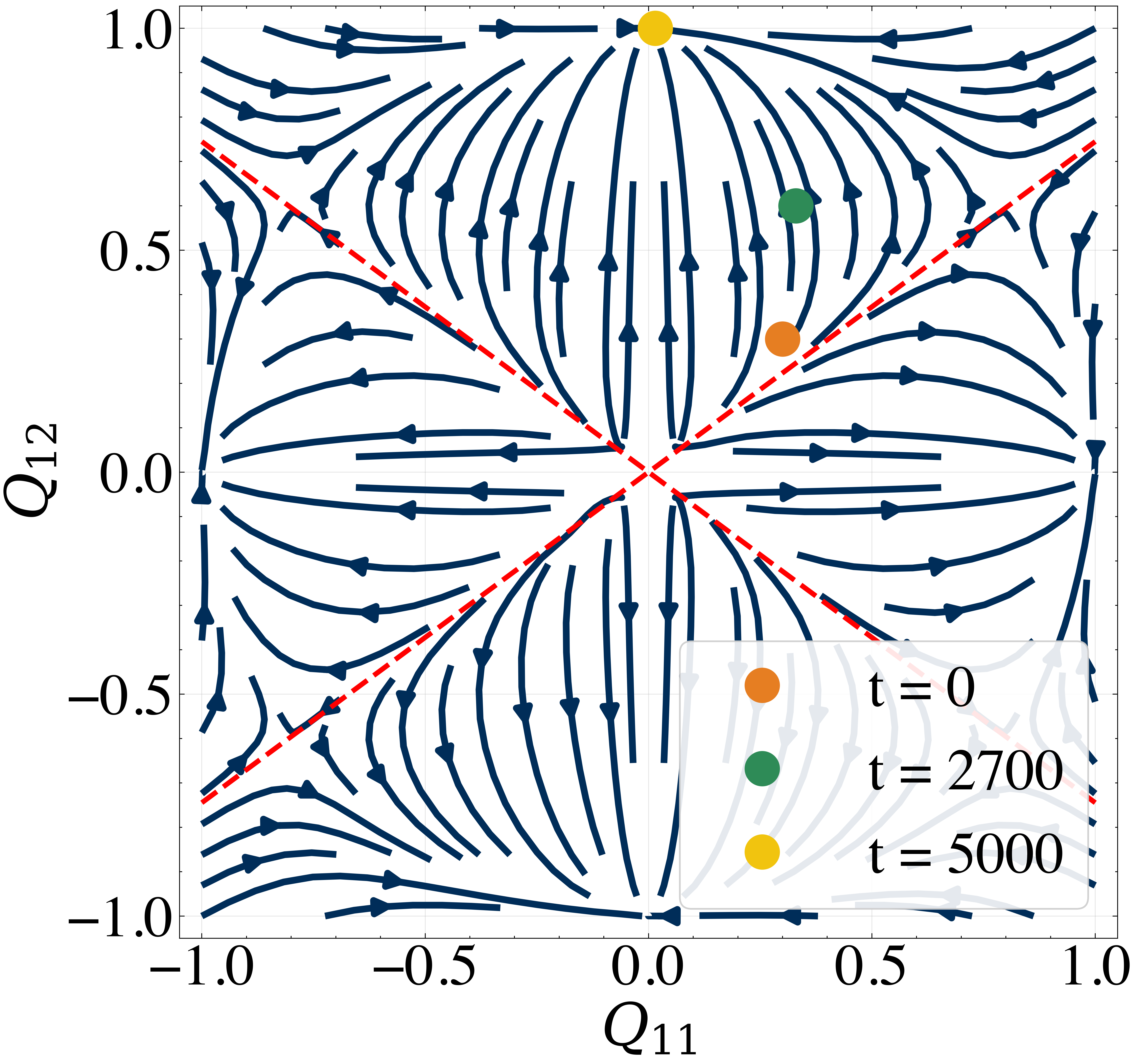}
    \caption{}
    \label{fig:3a}
  \end{subfigure}\hfill
  \begin{subfigure}[t]{0.245\textwidth}
    \centering
    \includegraphics[width=\linewidth]{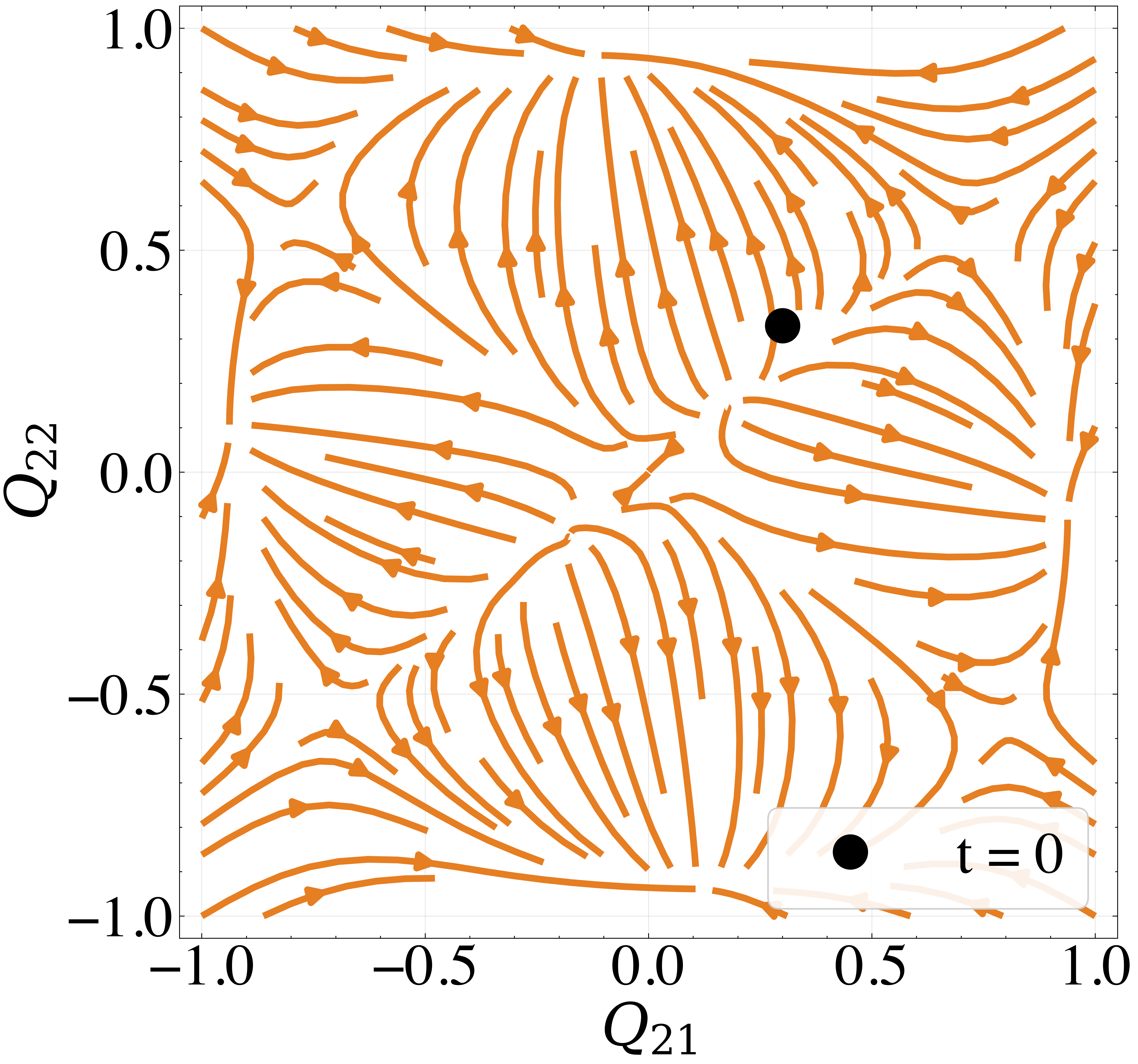}
        \caption{}
    \label{fig:3b}
  \end{subfigure}\hfill
  \begin{subfigure}[t]{0.245\textwidth}
    \centering
    \includegraphics[width=\linewidth]{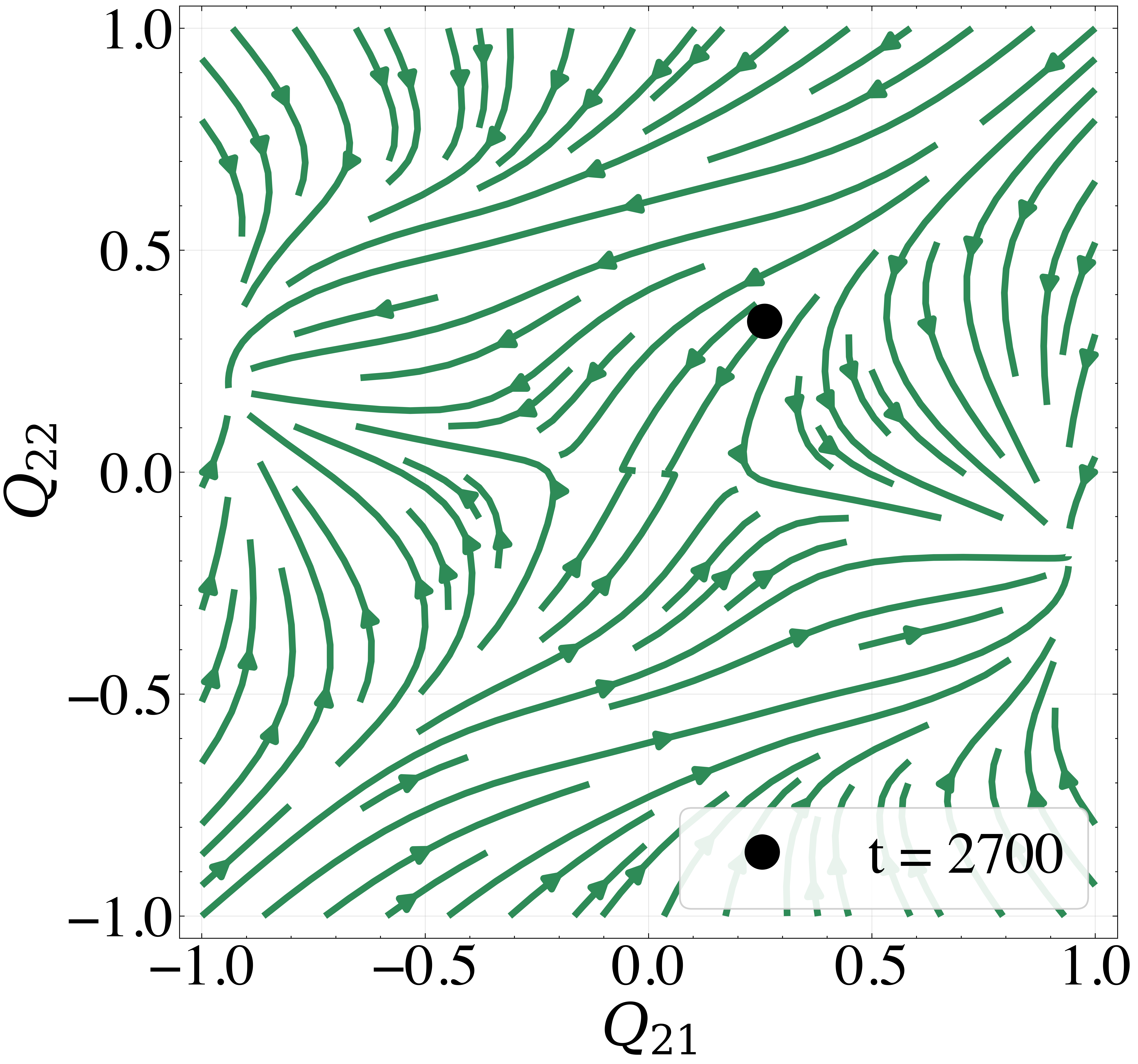}
    \caption{}
    \label{fig:3c}
  \end{subfigure}\hfill
  \begin{subfigure}[t]{0.245\textwidth}
    \centering
    \includegraphics[width=\linewidth]{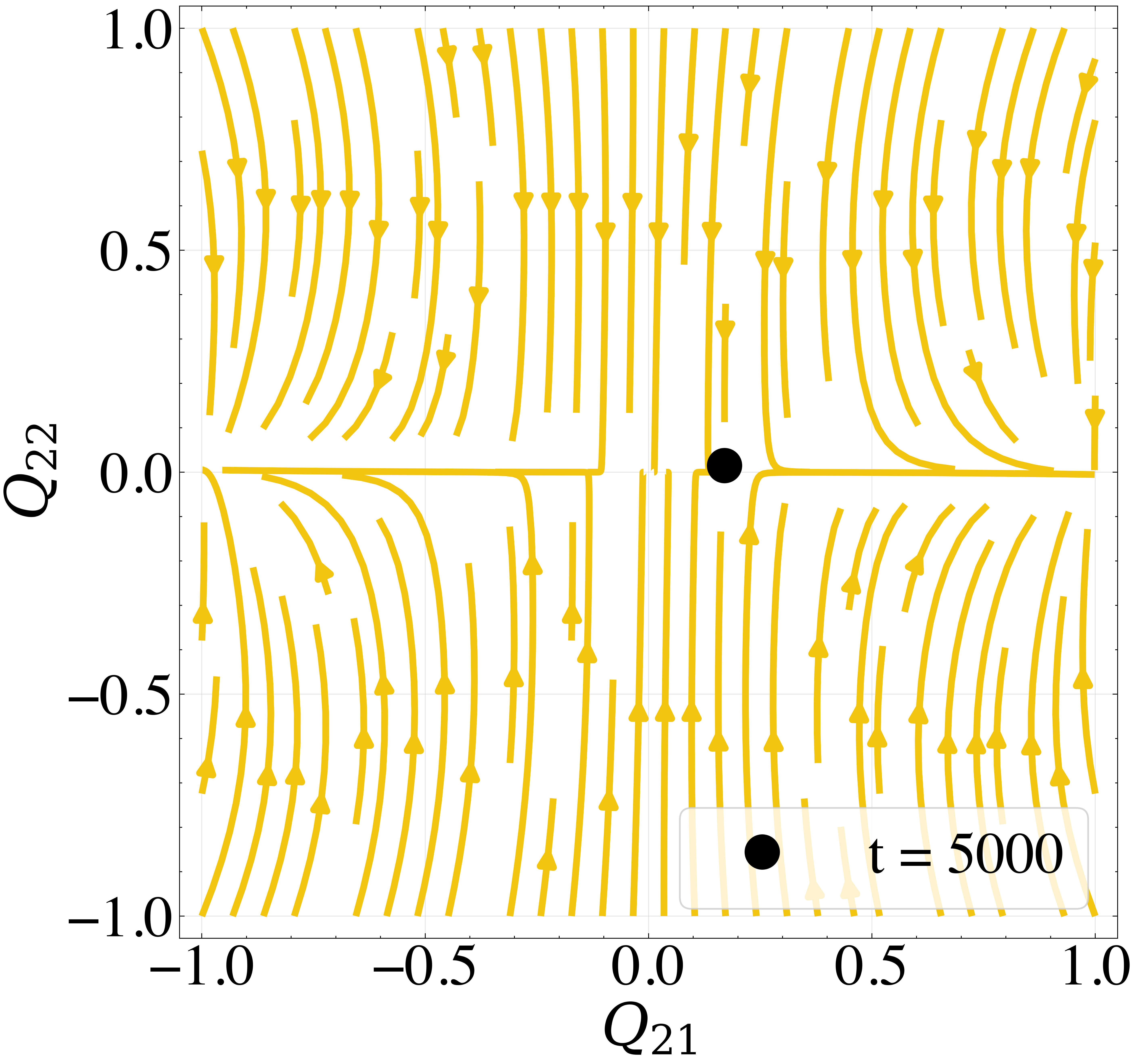}
    \caption{}
    \label{fig:3d}
  \end{subfigure}
  \caption{Phase portraits of the coupled $2 \times 2$ system in Figure \ref{fig:2b}. First row dynamics ($Q_{1,1}, Q_{1,2}$) evolve independently in (a), with markers representing specific timestamps: $t=0$ (orange), $t=2700$ (green), and $t=5000$ (yellow), and the dashed red line represents the theoretically predicted competition boundary in \eqref{competitionboundaryforcubic}. In (b), (c), and (d), we plot the second  row dynamics ($Q_{2,1}, Q_{2,2}$) at the corresponding times, with black markers indicating the ODE-predicted state. As the first row evolves, it progressively tilts the stability landscape of the second row. Vector field colors in (b–d) match the first-row temporal states for visual alignment.}
  
  \label{fig:phase_portrait}
\end{figure}

A diagonal initial overlap matrix (Figure~\ref{fig:2a}) yields a decoupled phase, where each estimate locks onto a distinct true feature without interference. In contrast, uniform initialization induces competition (Figure~\ref{fig:2b}): both estimates initially target the same component, and after one aligns strongly, orthogonalization forces the other to unlearn it before recovering the remaining feature. This coupling requires a smaller learning rate $\tau$ and slows convergence. Changing the data moments under the same initialization changes the competition pattern (Figure~\ref{fig:2c}): one estimate first aligns with \emph{both} components before specializing, while the other captures the remaining feature asymptotically.

The competition dynamics in Figure~\ref{fig:2b} are clarified by the phase portraits in Figure~\ref{fig:phase_portrait}. Figure~\ref{fig:3a} shows that initialization determines the attraction basin of the first estimator. For the second estimator, Figure~\ref{fig:3b} initially points toward the second component, but the later snapshot at $t=2700$ (Figure~\ref{fig:3c}) shows that the vector field redirects it toward the first component, forcing it to unlearn the second component before converging, as seen in Figure~\ref{fig:3d}.

\section{Steady state analysis}

Next, we derive insights from the ODEs describing the learning dynamics. Specifically, steady-state analysis of the limiting ODEs reveals a nontrivial interplay between the learning rate $\tau$, initialization, and the higher-order moment structure of the components. By analyzing the equilibrium $\frac{d\bm{Q}}{dt} = 0$ in \eqref{eq:Q_ODE}, we give general conditions for learnability and competition. In the rest of this section, we first study the learnability boundary, then show a staircase behavior in the total number of learned components, derive a condition (a boundary) for competition to occur, and finally, we demonstrate the instability induced by competition.

\subsection{Learnability boundary (decoupled initialization).}

We first consider the decoupled regime, where $\bm{Q}_t$ remains nearly diagonal (up to some permutations and negligible off-diagonal entries). Consequently, for the steady state of the overlap matrix denoted by $\bm{Q}_s$, the equilibrium condition for \eqref{eq:Q_ODE} reduces to 
\begin{align}\label{eq:Q1,1_dynamics}
\tau \bm{\Psi} (\bm{I} - \bm{Q}_s^2) - \frac{\tau^2}{2} \mathrm{diag}(\mathbb{E}_{\bm{c_t},\bm{e}} \left[ f^2(\bm{\upsilon}) \right])\bm{Q}_s = 0. 
\end{align}

Equation~\eqref{eq:Q1,1_dynamics} holds for any nonlinearity $f$, under our assumptions given in Theorem \ref{theorem}. However, since the boundary depends explicitly on $f$, a closed-form condition in terms of the learning rate and the higher-order moments is intractable. To obtain an explicit criterion, we revisit the cubic $f(x) = \pm x^3$; which reduces the dependence to finitely many moments, yielding an exact learnability boundary.

For the remainder of this analysis, we adopt $f(x) = -x^3$ as an example case. 
Following the detailed derivation in Appendix~\ref{app:boundaryproof}, we arrive at the following learnability boundary. Define

        \begin{align}
            \xi &:= (m_6 -15) - 15(m_4 - 3), \ \ 
            \eta  := (m_4 - 3) \left(15 - \frac{2}{\tau}\right), \ \ 
            \zeta := \frac{2(m_4 - 3)}{\tau}, \ \ 
            \varpi := 15. \notag
        \end{align}
    Then, \emph{learning occurs if}
    \begin{align}
        \eta^2 \zeta^2 - 4\xi\zeta^3 - 4\eta^3\varpi - 27\xi^2\varpi^2 + 18\xi\eta\zeta\varpi > 0 \qquad \quad \text{and}\qquad \quad 3 \varpi + 2 \zeta + \eta < 0.
    \end{align}

The simultaneous satisfaction of these inequalities defines \emph{the learnable region}, a subset of the $(\tau,m_4,m_6)$ space. To isolate the role of $\tau$, Figure~\ref{fig:4a} illustrates cross-sections in the $(m_4,m_6)$ plane for various values of $\tau$. As $\tau$ increases, the learnable region becomes increasingly constrained, persisting only for a narrow set of moment combinations. Conversely, as $\tau\to 0$, the constraint relaxes and the stability boundary approaches a vertical asymptote $m_4 = 3$, so that the boundary is controlled solely by the sign of the fourth cumulant, $\kappa_4 = m_4 -3$, (alternative characterization of non-Gaussianity to raw moments \cite{novak2014three, fisher1932derivation}).

\begin{figure}[t]
  \centering
  \begin{subfigure}[t]{0.31\textwidth}
    \centering
    \includegraphics[width=\linewidth]{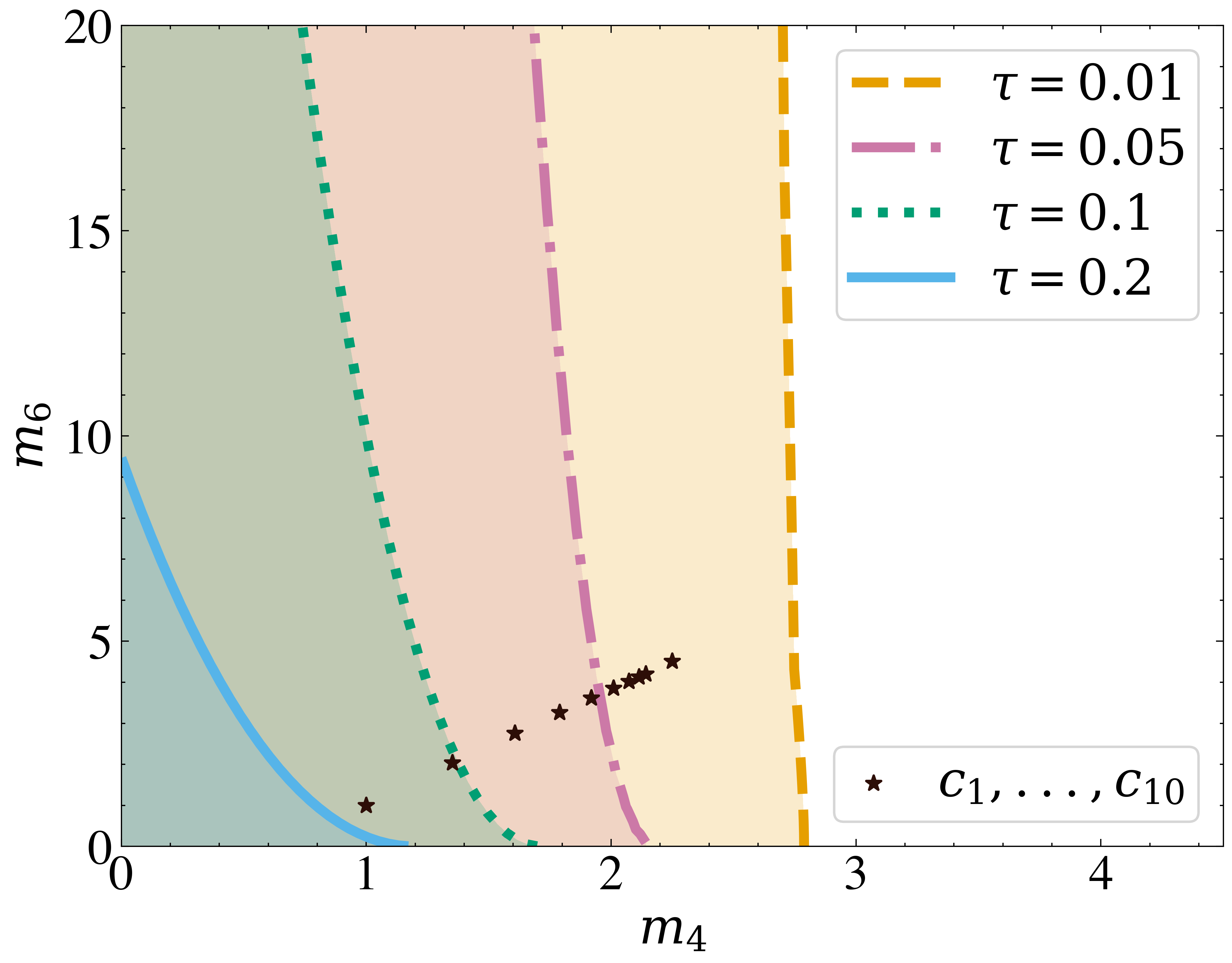}
    \caption{Theoretical learnability boundaries for different learning rates.}

    \label{fig:4a}
  \end{subfigure}\hfill
  \begin{subfigure}[t]{0.31\textwidth}
    \centering
    \includegraphics[width=\linewidth]{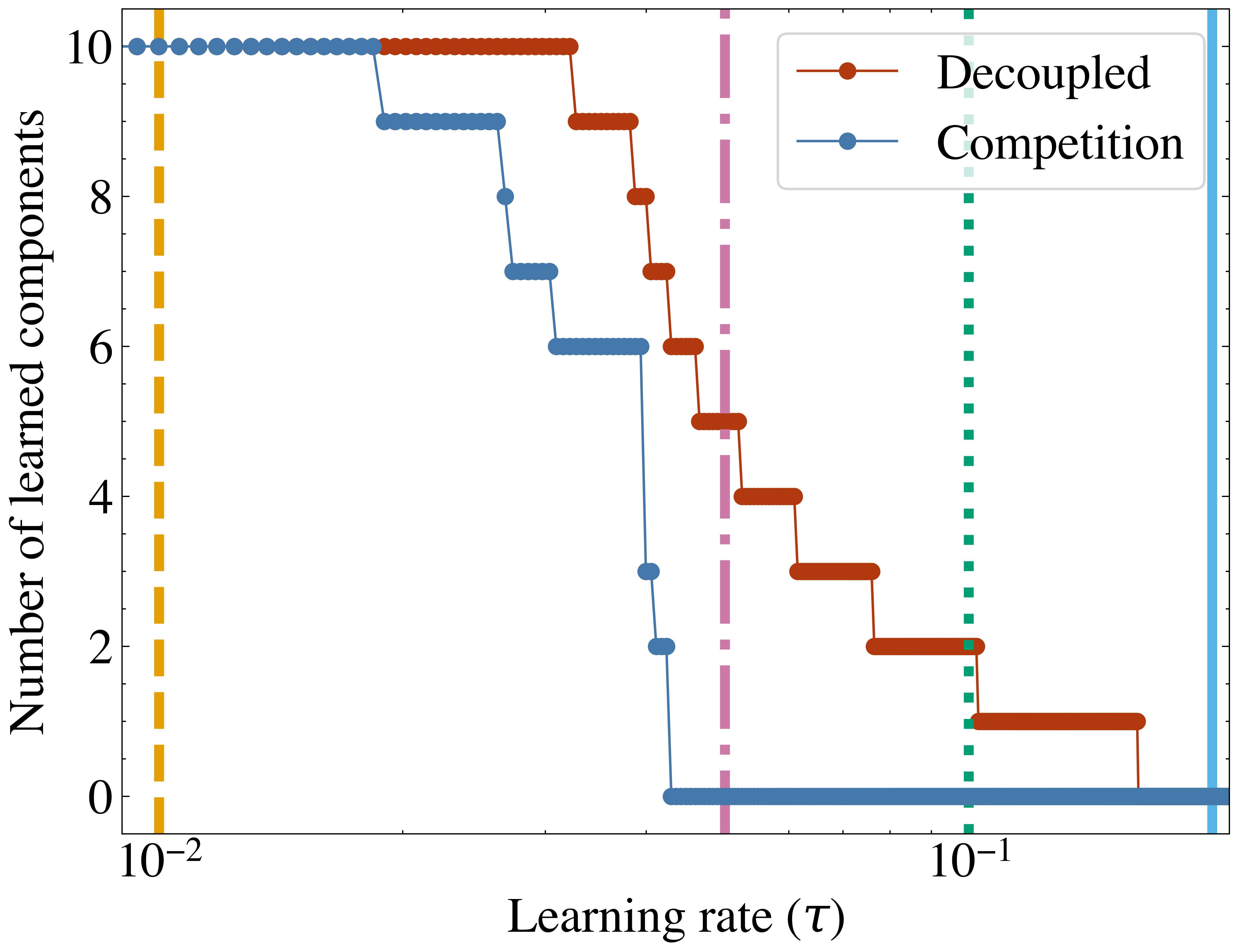}
    \caption{Number of learned components as a function of the learning rate.}

    \label{fig:4b}
  \end{subfigure}\hfill
  \begin{subfigure}[t]{0.33\textwidth}
    \centering
    \includegraphics[width=\linewidth]{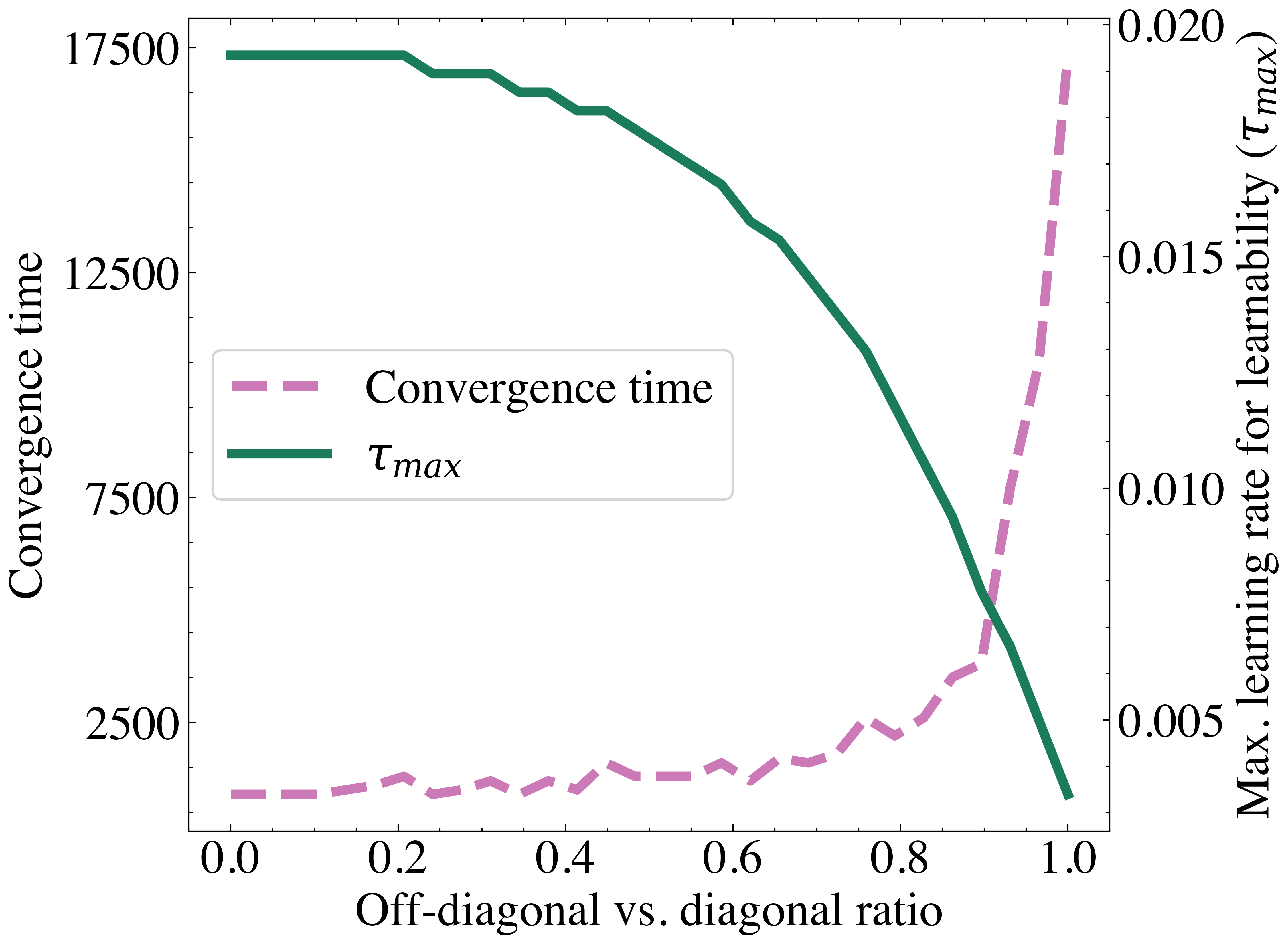}
    \caption{Instability of competition: effect of initial off-diagonal/diagonal ratio.}

    \label{fig:4c}
  \end{subfigure}
  
  \caption{Steady-state analysis reveals learnability boundaries, staircase behavior, and competition-induced instability. In (a)--(b), component moments $m_4,m_6$ decrease with the component index. (a) Learnability regions for different learning rates, with scattered points showing the moments of $c_1,\ldots,c_{10}$; curves denote phase boundaries. (b) Number of recovered components versus $\tau$ for decoupled and competition regimes with $p=10$; vertical lines mark the theoretical boundaries from Figure~\ref{fig:4a}. (c) As the off-diagonal/diagonal initialization ratio increases, the largest learning rate recovering all components, $\tau_{\max}$, decreases, while convergence time increases.}
\end{figure}

\subsection{Staircase behavior}

Figure~\ref{fig:4b} shows that, as the component moment pairs $(m_4,m_6)$ cross the learnability boundaries, the number of recoverable components changes in discrete steps with the learning rate $\tau$. The thresholds in Figure~\ref{fig:4b} match the boundary crossings in Figure~\ref{fig:4a}, confirming that the mean-field theory captures the phase transitions. Competition shifts the staircase to smaller values of $\tau$: although the discrete structure persists, stable recovery requires more conservative learning rates. Thus, learning more components requires smaller $\tau$, with the precise thresholds determined jointly by the data moments and the initialization regime.

\subsection{Competition boundary (coupled initialization).}
To analytically characterize the precise conditions for the competition regime, we return to the example case of $p=2$. Following the derivation detailed in Appendix \ref{app:subsec_sompeition_boundary}, we identify the boundary at which the ratio between components is stationary, given by the condition:
\begin{align}
        \frac{\mathbb{E}_{\bm{c},\bm{e}}[c_1f(\upsilon_1)]}{Q_{1,1}} =\frac{\mathbb{E}_{\bm{c},\bm{e}}[c_2f(\upsilon_1)]}{Q_{1,2}}.
    \label{eq:competition_condition}
\end{align}
For $f(x) = \pm x^3$, once the learnability conditions hold, the learned component is selected by initialization through a boundary in the $(Q_{1,1},Q_{1,2})$ phase plane. Specifically, Equation \eqref{eq:competition_condition} yields the theoretical competition boundary:
\begin{align}\label{competitionboundaryforcubic}
    Q_{1,2} = \pm \sqrt{\frac{(m_{1,4}-3)}{(m_{2,4}-3)}} \, Q_{1,1},
\end{align}
where $m_{j,4}$ denotes the fourth moment of component $c_j$. This relation defines an implicit cross-shaped separatrix, which can be observed in Figure~\ref{fig:3a} as the red dashed lines, partitioning the phase space into distinct basins of attraction. While the general competition condition holds for complex nonlinearities like $\tanh(x)$, the contribution of all higher-order moments of $c_j$ yields more intricate boundaries, whose analysis lies beyond the scope of this paper.

\subsection{Instability induced by competition}
The degree of coupling in the initialization governs both the efficiency and the permissible learning rate of the recovery process. In Figure~\ref{fig:4b}, we already observed that coupled initialization shifts the staircase curve toward smaller values of $\tau$ relative to the decoupled case, indicating that the competition regime necessitates more conservative learning rates to recover all components. Figure~\ref{fig:4c} demonstrates this effect directly. As the ratio off-diagonal to diagonal entries in the initial overlap matrix increases (transitioning from decoupled to fully coupled regime) the maximum permissible learning rate $\tau_{\max}$ required for full recovery decreases significantly. This required reduction in learning rate, along with component competition, results in a sharp nonlinear increase in convergence time. The close agreement between our theoretical predictions and these empirical trajectories confirms that the high-dimensional dynamics are fundamentally constrained by an initialization-induced trade-off. Stronger initial coupling necessitates smaller learning rates, which in turn prolongs the time required for all components to converge.

\section{Multi-component online ICA for hyperspectral remote sensing}\label{section:hyperspectrexepr}
To validate the predictive power of our theory in a practical setting, we conduct an experiment on hyperspectral remote sensing \citep{lupu2022stochastic}. Specifically, we use the Indian Pines hyperspectral dataset \citep{Baumgardner15}, where each pixel is represented by a 200-dimensional spectral vector. Given this dimensionality, the system operates near the high-dimensional regime, allowing for a robust test of our theory. Following the pre-processing steps detailed in Appendix \ref{app:experimental}, we extract canonical signal profiles by averaging the ground-truth masked vectors within each class and these spectral signatures serve as the latent components that we then mix for the ICA problem. As illustrated in Figure \ref{fig:5a}, the ODE predictions accurately characterize the learning dynamics. Figure \ref{fig:5b} and \ref{fig:5c} provide a comparison between the learned masks and the corresponding ground-truth masks. The empirical trajectories remain well within the theoretically predicted regions, affirming our theoretical results. 
\begin{figure}[t]
  \centering
  \begin{subfigure}[t]{0.29\textwidth}
    \centering
    \includegraphics[width=\linewidth]{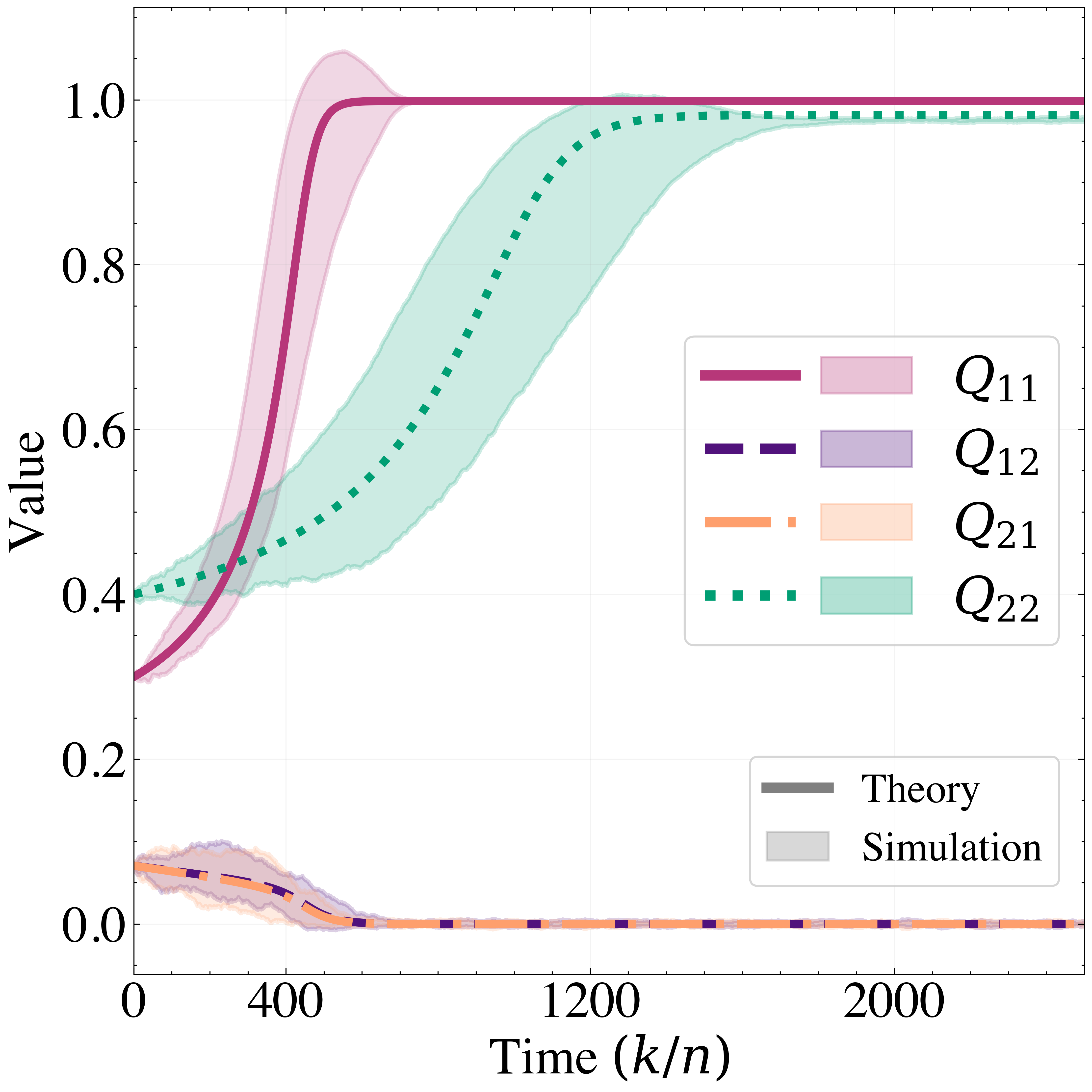}    
    \caption{Cosine similarity dynamics.}
    \label{fig:5a}
  \end{subfigure}\hfill
  \begin{subfigure}[t]{0.29\textwidth}
    \centering
    \includegraphics[width=\linewidth]{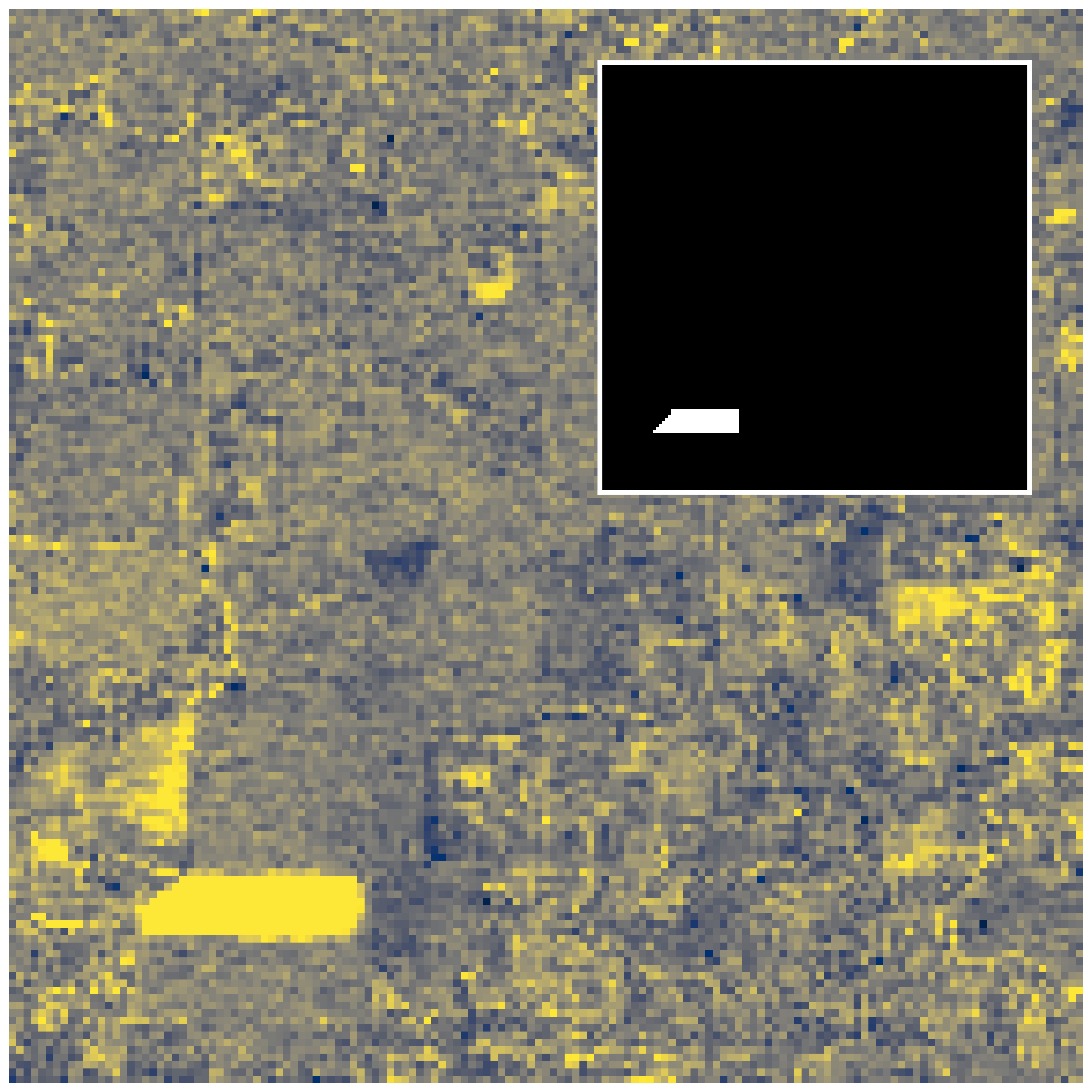}
    \caption{Learned mask — Class 1.}

    \label{fig:5b}
  \end{subfigure}\hfill
  \begin{subfigure}[t]{0.345\textwidth}
    \centering
    \includegraphics[width=\linewidth]{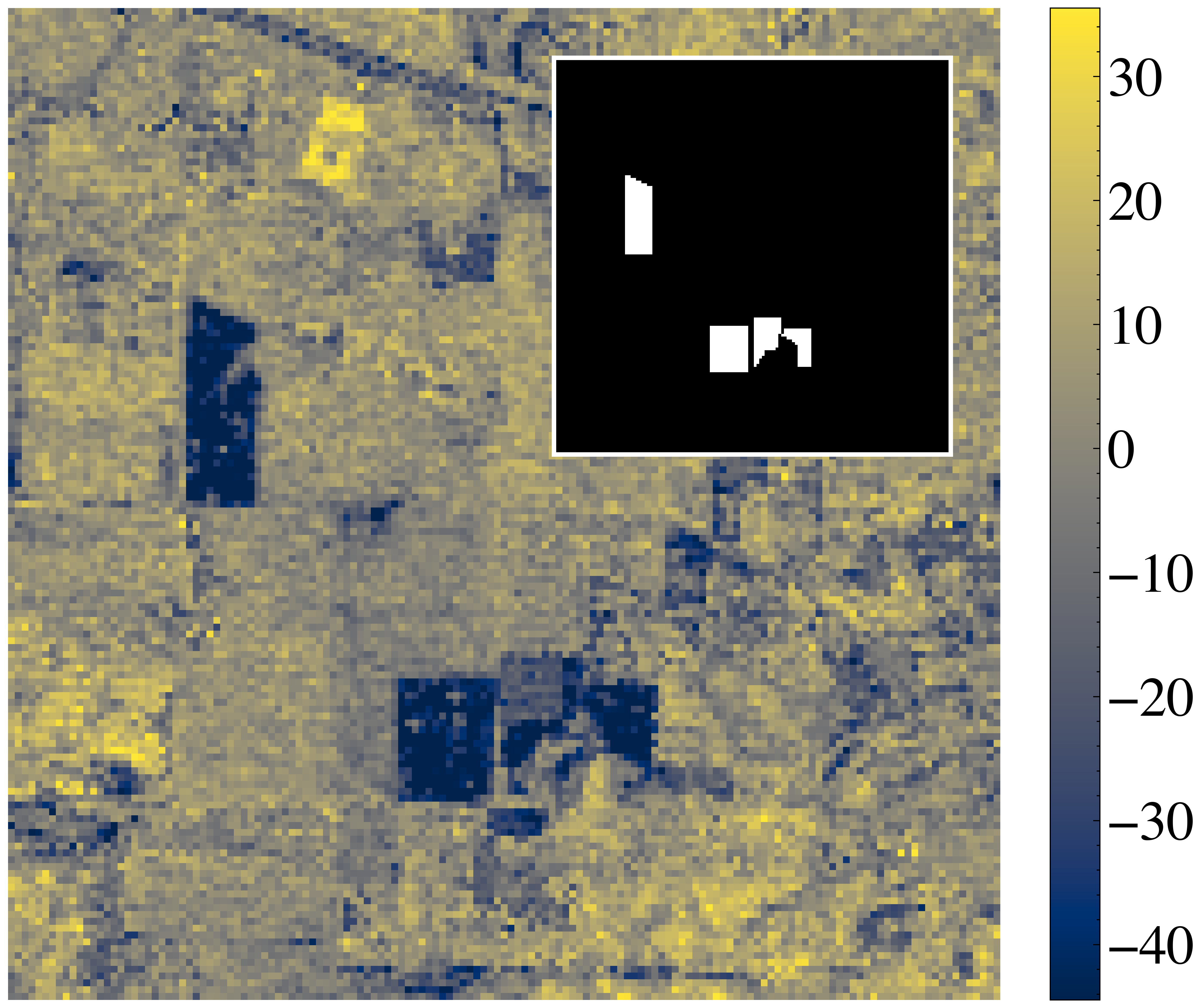}
    \caption{Learned mask — Class 2.}

    \label{fig:5c}
  \end{subfigure}
  \caption{Multi-component ICA on the Indian Pines dataset.
In (a) ODE predictions accurately track the empirical dynamics for $\tau=0.01$, $\beta_1=1$, and $\beta_2=0.6$. The shaded regions indicate two standard deviations across 20 runs. In (b), learned spatial masks align with the underlying true classes 1 and 2; ground-truth annotations are shown as black\&white insets (top right).}
\end{figure}

\section{Conclusion}
We developed an asymptotically exact mean-field theory for multi-component online ICA in high dimensions. Our main theorem shows that the joint empirical measure of learned and ground-truth coordinates converges to a deterministic limit governed by a nonlinear PDE, from which we derive a closed ODE system for the overlap matrix between estimates and true components. This characterization reveals an initialization-dependent phase structure with two distinct regimes: a decoupled regime, where components evolve nearly independently, and a competition regime, where orthogonality constraints induce conflicts that slow learning.
Our steady-state analysis yields explicit closed-form learnability boundaries linking the step size, initialization, and higher-order moments. These boundaries predict a staircase phenomenon in which the number of recoverable components decreases discretely as the learning rate increases. They also imply that competition and larger higher-order moments shrink the stable learning-rate window for recovering all components and lead to longer convergence times. Experiments on synthetic data and hyperspectral remote sensing data confirm that the mean-field predictions remain accurate even at moderate dimensions.
\section*{Acknowledgements}
We acknowledge that this work is supported partially by TUBİTAK under project 124E063 in the ARDEB 1001 program. E.İ.G. is supported by an MS Scholarship (BIDEB 2210) from TUBİTAK and a Graduate Fellowship provided by Koç University. S.D. is supported by an AI Fellowship provided by KUIS AI Research Center and a PhD Scholarship (BIDEB 2211) from TUBİTAK.
\bibliographystyle{unsrt}
\bibliography{ref}

@inproceedings{Gultekin2025Learning,
  author    = {M. O\u{g}uzhan G\"ultekin and Samet Demir and Zafer Do\u{g}an},
  title     = {{Learning Rate Should Scale Inversely with High-Order Data Moments in High-Dimensional Online Independent Component Analysis}},
  booktitle = {Proceedings of the IEEE International Workshop on Machine Learning for Signal Processing},
  year      = {2025},
}

@misc{Baumgardner15,
        title = {220 Band {AVIRIS} Hyperspectral Image Data Set: June 12, 1992 Indian Pine Test Site 3},
        year = {2015},
        author = {Marion F. Baumgardner  and Larry L. Biehl and David A. Landgrebe },
}

@article{wang2017scaling,
  title={Scaling limit: Exact and tractable analysis of online learning algorithms with applications to regularized regression and PCA},
  author={Wang, Chuang and Mattingly, Jonathan and Lu, Yue M},
  journal={arXiv preprint arXiv:1712.04332},
  year={2017}
}

@article{Edelman1998,
  title   = {The geometry of algorithms with orthogonality constraints},
  author  = {Edelman, Alan and Arias, Tom{\'a}s A. and Smith, Steven T.},
  journal = {SIAM Journal on Matrix Analysis and Applications},
  volume  = {20},
  number  = {2},
  pages   = {303--353},
  year    = {1998}
}

@article{WenYin2013,
author = {Wen, Zaiwen and Yin, Wotao},
title = {A feasible method for optimization with orthogonality constraints},
year = {2013},
issue_date = {2013},
publisher = {Springer-Verlag},
volume = {142},
number = {1–2},
journal = {Math. Program.},
pages = {397–434},
numpages = {38},
}

@book{Boumal2023,
  title     = {An Introduction to Optimization on Smooth Manifolds},
  author    = {Boumal, Nicolas},
  publisher = {Cambridge University Press},
  year      = {2023}
}

@ARTICLE{balzano2018,
  author={Balzano, Laura and Chi, Yuejie and Lu, Yue M.},
  journal={Proceedings of the IEEE}, 
  title={Streaming {PCA} and Subspace Tracking: The Missing Data Case}, 
  year={2018},
  volume={106},
  number={8},
  pages={1293-1310}}

@article{wang2018,
author = {Wang, Chuang and Eldar, Yonina C. and Lu, Yue M.},
year={2018},
  volume={12},
  number={6},
  pages={1240-1252},
title = {Subspace Estimation From Incomplete Observations: A High-Dimensional Analysis},
journal = {IEEE Journal of Selected Topics in Signal Processing},
}

@inproceedings{bond2024,
      title={Exploring the Precise Dynamics of Single-Layer {GAN} Models},
      booktitle={Advances in Neural Information Processing Systems},
      author={Bond, Andrew and Dogan, Zafer},
      year={2024}
}

@inproceedings{wang2017ica,
  title={The scaling limit of high-dimensional online independent component analysis},
  author={Wang, Chuang and Lu, Yue M.},
  booktitle={Advances in Neural Information Processing Systems},
  year={2017}
}

@inproceedings{ricci2025,
  title={Feature learning from non-{G}aussian inputs: the case of {I}ndependent {C}omponent {A}nalysis in high dimensions},
  author={Ricci, Fabiola and Bardone, Lorenzo and Goldt, Sebastian},
  booktitle={Proceedings of International Conference on Machine Learning},
  year={2025}
}

@article{comon1994independent,
  title={Independent component analysis, a new concept?},
  author={Comon, Pierre},
  journal={Signal processing},
  volume={36},
  number={3},
  pages={287--314},
  year={1994}
}

@article{hyvarinen2000independent,
  title={Independent component analysis: algorithms and applications},
author={Hyv{\"a}rinen, A and Oja, E},
  journal={Neural networks},
  volume={13},
  number={4-5},
  pages={411--430},
  year={2000},
  publisher={Elsevier}
}

@ARTICLE{fastica,
  author={Hyv{\"a}rinen,, A.},
  journal={IEEE Transactions on Neural Networks}, 
  title={Fast and robust fixed-point algorithms for independent component analysis}, 
  year={1999},
  volume={10},
  number={3},
  pages={626-634}}

@inproceedings{li2016online,
  author    = {Li, Chris Junchi and Wang, Zhaoran and Liu, Han},
  title     = {Online {ICA}: Understanding Global Dynamics of Nonconvex Optimization via Diffusion Processes},
  booktitle = {Advances in Neural Information Processing Systems},
  year      = {2016}
}

@inproceedings{ablin2019stochastic,
  title     = {Stochastic algorithms with descent guarantees for {ICA}},
  author    = {Ablin, Pierre and Gramfort, Alexandre and Cardoso, Jean-Fran{\c{c}}ois and Bach, Francis},
  booktitle = {Proceedings of International Conference on Artificial Intelligence and Statistics},
  year      = {2019},
}

@article{lupu2022stochastic,
  author    = {Lupu, Daniela and Necoara, Ion and Garrett, Joseph L. and Johansen, Tor Arne},
  title     = {Stochastic Higher-Order Independent Component Analysis for Hyperspectral Dimensionality Reduction},
  journal   = {IEEE Transactions on Computational Imaging},
  volume    = {8},
  pages     = {1184--1194},
  year      = {2022},
  publisher = {IEEE}
}

@inproceedings{li2021stochastic,
  title     = {Stochastic Approximation for Online Tensorial Independent Component Analysis},
  author    = {Li, Chris Junchi and Jordan, Michael I.},
  booktitle = {Proceedings of Conference on Learning Theory},
  year      = {2021}
}

@article{collins2024hitting,
  author    = {Collins-Woodfin, Elizabeth and Paquette, Courtney and Paquette, Elliot and Seroussi, Inbar},
  title     = {Hitting the High-dimensional notes: an {ODE} for {SGD} learning dynamics on {GLMs} and multi-index models},
  journal   = {Information and Inference: A Journal of the IMA},
  volume    = {13},
  number    = {4},
  year      = {2024},
}

@article{auddy2023large,
author = {Arnab Auddy and Ming Yuan},
title = {{Large-dimensional independent component analysis: Statistical optimality and computational tractability}},
volume = {53},
journal = {The Annals of Statistics},
number = {2},
publisher = {Institute of Mathematical Statistics},
pages = {477 -- 505},
year = {2025},

}

@article{sirignano2019mean,
    author = {Sirignano, Justin and Spiliopoulos, Konstantinos},
    title = {Mean Field Analysis of Neural Networks: A Law of Large Numbers},
    journal = {SIAM Journal on Applied Mathematics},
    volume = {80},
    number = {2},
    pages = {725-752},
    year = {2020},
}

@article{demir2025implicitly,
  title={Implicitly Normalized Online {PCA}: A Regularized Algorithm with Exact High-Dimensional Dynamics},
  author={Demir, Samet and Dogan, Zafer},
  journal={arXiv preprint arXiv:2512.01231},
  year={2025},
  url={https://arxiv.org/abs/2512.01231}
}

@inproceedings{wang2016online,
  author={Chuang Wang and Lu, Yue M.},
  booktitle={2016 IEEE Information Theory Workshop}, 
  title={Online learning for sparse {PCA} in high dimensions: Exact dynamics and phase transitions}, 
  year={2016},
}

@inproceedings{arnaboldi2023high,
  title={From high-dimensional \& mean-field dynamics to dimensionless {ODEs}: A unifying approach to {SGD} in two-layers networks},
  author={Arnaboldi, Luca and Stephan, Ludovic and Krzakala, Florent and Loureiro, Bruno},
  booktitle={Proceedings of Conference on Learning Theory},
  year={2023}
}

@inproceedings{arous2023highdimensionallimittheoremssgd,
title={High-dimensional limit theorems for {SGD}: Effective dynamics and critical scaling},
author={Gerard Ben Arous and Reza Gheissari and Aukosh Jagannath},
booktitle={Advances in Neural Information Processing Systems},
year={2022},
}

@inproceedings{li2025statistical,
  title={Statistical Guarantees for High-Dimensional Stochastic Gradient Descent},
  author={Li, Jiaqi and Lou, Zhipeng and Schmidt-Hieber, Johannes and Wu, Wei Biao},
  booktitle={Advances in Neural Information Processing Systems},
  year={2025}
}

@inproceedings{amari1995new,
  title={A New Learning Algorithm for Blind Signal Separation},
  author={Amari, Shun-ichi and Cichocki, Andrzej and Yang, Howard Hua},
  booktitle={Advances in Neural Information Processing Systems},
  year={1995}
}

@incollection{mckean1967propagation,
  title={Propagation of chaos for a class of non-linear parabolic equations},
  author={McKean, Henry P.},
  booktitle={Stochastic Differential Equations (Lecture Series in Differential Equations, Session 7)},
  pages={41--57},
  year={1967},
  publisher={Catholic University}
}

@article{meng2024training,
  title={Training Dynamics of Nonlinear Contrastive Learning Model in the High Dimensional Limit},
  author={Meng, Linghuan and Wang, Chuang},
  journal={IEEE Signal Processing Letters},
  volume={31},
  pages={2535--2539},
  year={2024},
  publisher={IEEE},
  doi={10.1109/LSP.2024.3465133},
  url={https://ieeexplore.ieee.org/document/10685361}
}

@article{bell1995information,
  author    = {Bell, Anthony J. and Sejnowski, Terrence J.},
  title     = {An Information-Maximization Approach to Blind Separation and Blind Deconvolution},
  journal   = {Neural Computation},
  volume    = {7},
  number    = {6},
  pages     = {1129--1159},
  year      = {1995}
}

@article{chen2020rosenthal,
  title     = {Rosenthal Type Inequalities for Random Variables},
  author    = {Chen, Pingyan and Sung, Soo Hak},
  journal   = {Journal of Mathematical Inequalities},
  volume    = {14},
  number    = {2},
  pages     = {305--318},
  year      = {2020},
  doi       = {10.7153/jmi-2020-14-20}
}

@article{lodvin,
  title={Singular Value Decomposition and the Centrality of {L}\"{o}wdin Orthogonalizations},
  author={Annavarapu, Ramesh Naidu},
  journal={American Journal of Computational and Applied Mathematics},
  volume={3},
  number={1},
  pages={33--35},
  year={2013},
}

@InProceedings{paquette21Large,
  title = 	 {{SGD} in the Large: Average-case Analysis, Asymptotics, and Stepsize Criticality},
  author =       {Paquette, Courtney and Lee, Kiwon and Pedregosa, Fabian and Paquette, Elliot},
  booktitle = 	 {Proceedings of Conference on Learning Theory},
  year = 	 {2021},

}

@inproceedings{Saad1995Dynamics,
  author    = {David Saad and Sara A. Solla},
  title     = {Dynamics of On-Line Gradient Descent Learning for Multilayer Neural Networks},
  year = {1995}, 
  booktitle = {Advances in Neural Information Processing Systems},
}

@inproceedings{
jagannath2026highdimensional,
title={High-dimensional limit theorems for {SGD}: Momentum and Adaptive Step-sizes},
author={Aukosh Jagannath and Taj Jones-McCormick and Varnan Sarangian},
booktitle={Proceedings of International Conference on Learning Representations},
year={2026},
}

@article{JMLR:v21:19-245,
  author  = {Yazhen Wang and Shang Wu},
  title   = {Asymptotic Analysis via Stochastic Differential Equations of Gradient Descent Algorithms in Statistical and Computational Paradigms},
  journal = {Journal of Machine Learning Research},
  year    = {2020},
  volume  = {21},
  number  = {199},
  pages   = {1--103},
}

@article{paquette2022homogenizationsgdhighdimensionsexact,
      title={Homogenization of {SGD} in high-dimensions: Exact dynamics and generalization properties}, 
      author={Courtney Paquette and Elliot Paquette and Ben Adlam and Jeffrey Pennington},
      journal={Mathematical Programming},
      volume={214},
      pages={1--90},
      year={2025},
      publisher={Springer},
}

@inproceedings{veiga2022phase,
  title={Phase diagram of Stochastic Gradient Descent in high-dimensional two-layer neural networks},
  author={Veiga, Rodrigo and Stephan, Ludovic and Loureiro, Bruno and Krzakala, Florent and Zdeborov{\'a}, Lenka},
  booktitle={Advances in Neural Information Processing Systems},
  year={2022}
}

@article{Goldt_2020,
   title={Dynamics of stochastic gradient descent for two-layer neural networks in the teacher–student setup},
   journal={Journal of Statistical Mechanics: Theory and Experiment},
   author={Goldt, Sebastian and Advani, Madhu S and Saxe, Andrew M and Krzakala, Florent and Zdeborová, Lenka},
volume={2020},
  number={12},
  pages={124010},
  year={2020},
}

@article{Rosenthal_1970,
   title={On the subspaces of $\mathrm{L}^p (p>2)$ spanned by sequences of independent random variables},
   journal={Israel J. Math.},
   author={Rosenthal, H.P.},
volume={8},
  number={3},
  pages={273-303},
  year={1970},
}

@article{Calhoun2001,
  author          = {Calhoun, V. D. and Adali, T. and Pearlson, G. D. and Pekar, J. J.},
  title           = {Spatial and temporal independent component analysis of functional {MRI} data containing a pair of task-related waveforms},
  journal         = {Human Brain Mapping},
  volume          = {13},
  number          = {1},
  pages           = {43--53},
  year            = {2001},
}

@article{gram_schmidt1,
author = {Leon, Steven J. and Björck, Åke and Gander, Walter},
title = {Gram-Schmidt orthogonalization: 100 years and more},
journal = {Numerical Linear Algebra with Applications},
volume = {20},
number = {3},
pages = {492-532},
year = {2013}
}

@article{schmidt1907theorie,
  title={Theorie der linearen und nichtlinearen Integralgleichungen I. Teil: Entwicklung willkürlicherFunktionen nach Systemen vorgeschriebener},
  author={Schmidt, Erhard},
  journal={Mathematische Annalen},
  volume={63},
  pages={433--476},
  year={1907},
  publisher={Springer Berlin/Heidelberg}
}

@incollection{novak2014three,
  author    = {Novak, Jonathan},
  title     = {Three lectures on free probability},
  booktitle = {Random Matrices},
  series    = {MSRI Publications},
  volume    = {65},
  year      = {2014},
  pages     = {309--383},
}

@article{fisher1932derivation,
  author  = {Fisher, Ronald A and Wishart, John},
  title   = {The derivation of the pattern formulae of two-way partitions from those of simpler patterns},
  journal = {Proceedings of the London Mathematical Society},
  pages   = {195--208},
  year    = {1931},
  
}

@inproceedings{fransizICA1985,
  author    = {Ans, Bernard and H{\'e}rault, Jeanny and Jutten, Christian},
  title     = {Architectures neuromim{\'e}tiques adaptatives~: D{\'e}tection de primitives},
  booktitle = {Cognitiva 85},
  volume    = {2},
  pages     = {593--597},
  year      = {1985},
}

@article{amari1998natural,
  author    = {Amari, Shun-ichi},
  title     = {Natural Gradient Works Efficiently in Learning},
  journal   = {Neural Computation},
  volume    = {10},
  number={2},
  pages     = {251--276},
  year      = {1998},
}

@inproceedings{linskerinfomax1988,
 author = {Linsker, Ralph},
 booktitle = {Advances in Neural Information Processing Systems},
 title = {An Application of the Principle of Maximum Information Preservation to Linear Systems},
 year = {1988},
}

@article{Cardoso1999HighOrderCF,
  title={High-Order Contrasts for Independent Component Analysis},
  author={J.-F. Cardoso},
  journal={Neural Computation},
  year={1999},
  volume={11},
  number={1},
  pages={157-192},
}

@inproceedings{NEURIPS2019_6b3c49bd,
 author = {Wang, Chuang and Hu, Hong and Lu, Yue},
 booktitle = {Advances in Neural Information Processing Systems},
 title = {A Solvable High-Dimensional Model of {GAN}},
 year = {2019}
}

@article{ songmei_meanfieldview_twolayer,
author = {Song Mei  and Andrea Montanari  and Phan-Minh Nguyen },
title = {A mean field view of the landscape of two-layer neural networks},
journal = {Proceedings of the National Academy of Sciences {U.S.A}},
volume = {115},
number = {33},
pages = {E7665-E7671},
year = {2018},
}

@article{balasubramanian2025high,
  author  = {Balasubramanian, Krishnakumar and Ghosal, Promit and He, Ye},
  title   = {High-dimensional scaling limits and fluctuations of online least-squares {SGD} with smooth covariance},
  journal = {The Annals of Applied Probability},
  volume  = {35},
  number  = {5},
  pages   = {2983--3045},
  year    = {2025},
}

@inproceedings{hyvarnien1997difentropy,
 author = {Hyv\"{a}rinen, Aapo},
 booktitle = {Advances in Neural Information Processing Systems},
 title = {New Approximations of Differential Entropy for Independent Component Analysis and Projection Pursuit},
 year = {1997}
}

@INPROCEEDINGS{hyvarnien1997_familyof,
  author={Hyv\"{a}rinen, A.},
  booktitle={IEEE International Conference on Acoustics, Speech, and Signal Processing}, 
  title={A family of fixed-point algorithms for independent component analysis}, 
  year={1997},
}

@article{HYVARINEN1998,
title = {Independent component analysis by general nonlinear Hebbian-like learning rules},
journal = {Signal Processing},
volume = {64},
number = {3},
pages = {301-313},
year = {1998},
author = {Aapo Hyvärinen and Erkki Oja}}

@article{snitzman1991propcgaos,
  author = {Sznitman, A.},
  title = {Topics in propagation of chaos},
  journal = {Lecture Notes in Mathematics},
  year = {1991},
  pages = {165-251},
}

@article{basalyga2003,
  author  = {Basalyga, Gleb and Rattray, Magnus},
  title   = {Statistical Dynamics of On-line Independent Component Analysis},
  journal = {Journal of Machine Learning Research},
  year    = {2003},
  volume  = {4},
  pages   = {1393--1410},
}

@INPROCEEDINGS{frieze1996learninglinear,
  author={Alan Frieze and Mark Jerrum and Ravi Kannan},
  booktitle={Proceedings of Conference on Foundations of Computer Science}, 
  title={Learning linear transformations.}, 
  year={1996},
}

@article{hyv_oja_fastfixed,
    author = {Hyvärinen, Aapo and Oja, Erkki},
    title = {A Fast Fixed-Point Algorithm for Independent Component Analysis},
    journal = {Neural Computation},
    volume = {9},
    number = {7},
    pages = {1483-1492},
    year = {1997},
    }

@article{DELFOSSE199559,
title = {Adaptive blind separation of independent sources: A deflation approach},
journal = {Signal Processing},
volume = {45},
number = {1},
pages = {59-83},
year = {1995},
author = {Nathalie Delfosse and Philippe Loubaton},}

@book{hyvarinen_oja_2001_book,
  title={Independent Component Analysis},
  author={Hyv{\"a}rinen, Aapo and Karhunen, Juha and Oja, Erkki},
  year={2001},
  publisher={John Wiley \& Sons}
}


\newpage

\appendix

\section*{Additional notation used in the Appendix}
For vectors $\bm{x}_{k,i}$, the first index $k$ denotes the iteration (time), and the second index $i\in\{1,\dots,p\}$ denotes the component. For any vector $\bm{a}$, we write $a^{(\alpha)}$ for its $\alpha$-th entry. Throughout the Appendix, $\|\cdot\|$ denotes the standard Euclidean ($\ell_2$) norm for vectors, while $\|\cdot\|_2$ and $\|\cdot\|_F$ denote the spectral and Frobenius norms for matrices, respectively.

To streamline the lengthy derivations in the Appendix, expectations of the random variables $c_i$ and $e_i \sim \mathcal{N}(0,1)$ are denoted by $\langle \cdot \rangle$ rather than the $\mathbb{E}_{\bm{c},\bm{e}}[\cdot]$ notation used in the main text.

Finally, we use standard asymptotic notation to describe scaling. We use Big~$O$ notation, $O(g(n))$, to denote that a term scales at most at the same order as $g(n)$. Conversely, we use Little~$o$ notation, $o(g(n))$, to denote terms that are asymptotically smaller than $g(n)$.

\section{Explicit derivations of drift and diffusion} \label{app:driftanddiffusion}

This section provides a detailed derivation of the drift and diffusion terms, which characterize the conditional mean and  variance of the increments $\bm{x}_{k+1,i} - \bm{x}_{k,i}$. To facilitate the derivation, we transition from the compact matrix-form update to a row-wise vector representation. For clarity, we first derive the explicit expressions for the two-component case ($p=2$), when $\phi(x) = 0$ before generalizing to an arbitrary number of components $p$ and an arbitrary $\phi(x)$. 

The complete online update for $p=2$ is defined by first computing the stochastic gradient increments:
\begin{align}
    \tilde{\bm{x}}_{k,1}
    = 
    \bm{x}_{k,1}
    +
    \frac{\tau}{\sqrt{n}}
    f\!\left(
        \frac{1}{\sqrt{n}} \bm{y}_k^\top \bm{x}_{k,1}
    \right)
    \bm{y}_k,
\end{align}
\begin{align}
    \tilde{\bm{x}}_{k,2}
    =
    \bm{x}_{k,2}
    +
    \frac{\tau}{\sqrt{n}}
    f\!\left(
        \frac{1}{\sqrt{n}} \bm{y}_k^\top \bm{x}_{k,2}
    \right)
    \bm{y}_k.
\end{align}
Using these increments, we immediately apply the Gram–Schmidt procedure:
\begin{align}
    \tilde{\tilde{\bm{x}}}_{k,1} = \tilde{\bm{x}}_{k,1},
\end{align}
\begin{align}\label{appendixgramschmit}
    \tilde{\tilde{\bm{x}}}_{k,2}
    =
    \tilde{\bm{x}}_{k,2}
    -
    \frac{
        \tilde{\bm{x}}_{k,2}^\top \tilde{\bm{x}}_{k,1}
    }{
        \lVert \tilde{\bm{x}}_{k,1} \rVert^2
    }
    \tilde{\bm{x}}_{k,1}.
\end{align}

Finally, we normalize both vectors to obtain the final updated iterates:

\begin{align}
\bm{x}_{k+1,1} & = \frac{\tilde{\bm{x}}_{k,1}}{\|\tilde{\bm{x}}_{k,1}\|}\sqrt{n}, \\  
\bm{x}_{k+1,2} &= \frac{\tilde{\tilde{\bm{x}}}_{k,2}}{\|\tilde{\tilde{\bm{x}}}_{k,2}\|} \sqrt{n} .
\end{align}

We proceed by explicitly analyzing the second component first, as it exposes the key cross-interaction terms that come from the Gram-Schmidt orthogonalization step in \eqref{appendixgramschmit}, which are necessary for the general inductive step. 
 
 The remainder of this section is organized as follows: In Appendix~\ref{app:secondestimate}, we derive the dynamics for the second component. We then provide the derivations for the first component in Appendix~\ref{app:firstestimate}, followed by the analysis of the cross terms in Appendix \ref{app:crossterms}. Finally, in Appendix \ref{app:regularization} we derive the additional regularization terms and in Appendix \ref{app:generalization}, we generalize these results to the case of $p$ components and provide the final explicit expressions for the drift and diffusion terms in Equations \eqref{appendixDRIFT_final}, and \eqref{appendixDIFF_final}.

\subsection{Dynamics of the second estimate}\label{app:secondestimate} 

First, for notational convenience, we define the stochastic gradient increments as 

\begin{align}\label{gradients_before_reg}
\bm{g}_{k,1} &\coloneqq \frac{\tau}{\sqrt{n}} f\left( \frac{1}{\sqrt{n}} \bm{y}_k^\top \bm{x}_{k,1} \right) \bm{y}_k , \\ 
\bm{g}_{k,2} &\coloneqq \frac{\tau}{\sqrt{n}} f\left( \frac{1}{\sqrt{n}} \bm{y}_k^\top \bm{x}_{k,2} \right) \bm{y}_k.
\end{align}

Then, substituting the update expressions into the orthogonalization step \eqref{appendixgramschmit} yields the explicit form below for $\bm{x}_{k,2}$:

\begin{align}
    \tilde{\tilde{\bm{x}}}_{k,2} &=
   (\bm{x}_{k,2} + \bm{g}_{k,2}) - \frac{(\bm{x}_{k,2} + \bm{g}_{k,2})^\top(\bm{x}_{k,1} + \bm{g}_{k,1})}{\|\bm{x}_{k,1} + \bm{g}_{k,1}\|^2}(\bm{x}_{k,1} + \bm{g}_{k,1}), \notag \\
   &= (\bm{x}_{k,2} + \bm{g}_{k,2}) - \frac{(\overbrace{\bm{x}_{k,2}^\top \bm{x}_{k,1}}^{=\,0} + \bm{x}_{k,2}^\top\bm{g}_{k,1} + \bm{g}_{k,2}^\top\bm{x}_{k,1} + \bm{g}_{k,2}^\top\bm{g}_{k,1}   )}{\underbrace{\|\bm{x}_{k,1}\|^2}_{=n} + \|\bm{g}_{k,1}\|^2 + 2 \bm{x}^\top_{k,1}\bm{g}_{k,1}}  (\bm{x}_{k,1} + \bm{g}_{k,1}), \notag \\
   & = 
   (\bm{x}_{k,2} + \bm{g}_{k,2}) - \frac{( \bm{x}_{k,2}^\top\bm{g}_{k,1} + \bm{g}_{k,2}^\top\bm{x}_{k,1} + \bm{g}_{k,2}^\top\bm{g}_{k,1}   )}{n( 1 + \|\bm{g}_{k,1}\|^2 /n + 2\bm{x}^\top_{k,1}\bm{g}_{k,1}/n)}  (\bm{x}_{k,1} + \bm{g}_{k,1}).
\end{align}

We used the orthogonality condition $\bm{x}_{k,2}^\top \bm{x}_{k,1} = 0 $, which holds by construction from the previous update step, as well as the normalization $ \|\bm{x}_{k,1}\|^2 = n$. Next, we expand the denominator to first order in  $1/n $. Higher-order terms involving products of three or more gradient increments contribute only at order $o(\frac{1}{n})$ and are omitted from the expansion:

\begin{align}
\tilde{\tilde{\bm{x}}}_{k,2}
&=
\bm{x}_{k,2} + \bm{g}_{k,2}
- \frac{
    \bm{x}_{k,2}^\top \bm{g}_{k,1}
    + \bm{g}_{k,2}^\top \bm{x}_{k,1}
    + \bm{g}_{k,2}^\top \bm{g}_{k,1}
}{n}\,
\bm{x}_{k,1} \notag \\ & \qquad
- \frac{
    \bm{x}_{k,2}^\top \bm{g}_{k,1}
    + \bm{g}_{k,2}^\top \bm{x}_{k,1}
}{n}\, 
\Big(
    - 2\,\bm{x}_{k,1}^\top \bm{g}_{k,1} / n
\Big) \bm{x}_{k,1} \notag \\ & \qquad
- \frac{
    \bm{x}_{k,2}^\top \bm{g}_{k,1}
    + \bm{g}_{k,2}^\top \bm{x}_{k,1}
}{n}\,
\bm{g}_{k,1}
+ o\!\left(\frac{1}{n}\right).
\end{align}

We define the orthogonalized increment as $\tilde{\tilde{\bm{\Delta}}}_{k,2} \coloneqq \tilde{\tilde{\bm{x}}}_{k,2} - \bm{x}_{k,2} $. Thus, we arrive at

\begin{align}
\tilde{\tilde{\bm{\Delta}}}_{k,2}
&=
\bm{g}_{k,2}
- \frac{
    \bm{x}_{k,2}^\top \bm{g}_{k,1}
    + \bm{g}_{k,2}^\top \bm{x}_{k,1}
    + \bm{g}_{k,2}^\top \bm{g}_{k,1}
}{n}\,
\bm{x}_{k,1} \notag \\[0.4em]
&\qquad
- \frac{
    \bm{x}_{k,2}^\top \bm{g}_{k,1}
    + \bm{g}_{k,2}^\top \bm{x}_{k,1}
}{n}\,
\Big(
    - 2\,\bm{x}_{k,1}^\top \bm{g}_{k,1} / n
\Big)\bm{x}_{k,1}
\notag \\[0.4em]
&\qquad
- \frac{
    \bm{x}_{k,2}^\top \bm{g}_{k,1}
    + \bm{g}_{k,2}^\top \bm{x}_{k,1}
}{n}\,
\bm{g}_{k,1} + o\!\left(\frac{1}{n}\right).
\end{align}

Following orthogonalization, the update is rescaled to norm $\sqrt{n}$,  and the corresponding normalization factor is expanded around $\bm{x}_{k,2}$:

\begin{align}
    \bm{x}_{k+1,2} &= \frac{\sqrt{n}}{\| \tilde{\tilde{\bm{x}}}_{k,2} \|} \tilde{\tilde{\bm{x}}}_{k,2} \ ,\notag \\  &= \frac{\sqrt{n}}{\| \bm{x}_{k,2} +\tilde{\tilde{\bm{\Delta}}}_{k,2} \|} (\bm{x}_{k,2} +\tilde{\tilde{\bm{\Delta}}}_{k,2}) \ ,  \notag \\
    &= (1 - \|\tilde{\tilde{\bm{\Delta}}}_{k,2}\|^2 /2n -\bm{x}_{k,2}^\top\tilde{\tilde{\bm{\Delta}}}_{k,2}/n ) (\bm{x}_{k,2} + \tilde{\tilde{\bm{\Delta}}}_{k,2}) + o\!\left(\frac{1}{n}\right), \\[3ex] 
    \bm{x}_{k+1,2} - \bm{x}_{k,2} &= \tilde{\tilde{\bm{\Delta}}}_{k,2} +  ( - \|\tilde{\tilde{\bm{\Delta}}}_{k,2}\|^2 /2n- \bm{x}_{k,2}^\top\tilde{\tilde{\bm{\Delta}}}_{k,2}/n ) \bm{x}_{k,2} + o\!\left(\frac{1}{n}\right) .
\end{align}

Recall that $x_{k,i}^{(j)}$ denotes the $j-th$ entry/element of $\bm{x}_{k,i}$. With this convention, we obtain

\begin{align}
\tilde{\tilde{\Delta}}_{k,2}^{(j)} &= 
     g_{k,2}^{(j)}
    - \frac{
        \bm{x}_{k,2}^\top \bm{g}_{k,1}
        + \bm{g}_{k,2}^\top \bm{x}_{k,1}
        + \bm{g}_{k,2}^\top \bm{g}_{k,1}
    }{n}\
    x_{k,1} ^{(j)} \notag \\[0.4em]
&\qquad + \frac{
        \bm{x}_{k,2}^\top \bm{g}_{k,1}
        + \bm{g}_{k,2}^\top \bm{x}_{k,1}
    }{n}\,
    \Big(
        2\,\bm{x}_{k,1}^\top \bm{g}_{k,1} / n
    \Big)x_{k,1}^{(j)}
\notag \\[0.4em]
&\qquad 
    - \frac{
         \bm{x}_{k,2}^\top \bm{g}_{k,1}
        + \bm{g}_{k,2}^\top \bm{x}_{k,1}
    }{n}\,
    g_{k,1}^{(j)}.
\label{deltatilttil}
\end{align}

After the normalization step, the entry-wise update satisfies

\begin{align}\label{equationx2-x2}
    x_{k+1,2}^{(j)} - x_{k,2}^{(j)} =  \tilde{\tilde{\Delta}}_{k,2}^{(j)} - \|\tilde{\tilde{\bm{\Delta}}}_{k,2}\|^2  x_{k,2}^{(j)}/2n - x_{k,2}^{(j)} (\bm{x}_{k,2}^\top\tilde{\tilde{\bm{\Delta}}}_{k,2})/n .
\end{align}

Taking expectations yields the first moment of the entry-wise update for the second component:

\begin{align} \label{beforecalculation,drift2}
    \mathbb{E}_k[x_{k+1,2}^{(j)} - x_{k,2}^{(j)}] &=  \mathbb{E}_k[\tilde{\tilde{\Delta}}_{k,2}^{(j)}] - \frac{\mathbb{E}_k[\|\tilde{\tilde{\bm{\Delta}}}_{k,2}\|^2]  x_{k,2}^{(j)} }{2n}- \frac{x_{k,2}^{(j)} \mathbb{E}_k[\bm{x}_{k,2}^\top\tilde{\tilde{\bm{\Delta}}}_{k,2}] }{n},\notag \\
    & =  \mathbb{E}_k[\tilde{\tilde{\Delta}}_{k,2}^{(j)}] -\frac{\mathbb{E}_k[\|\tilde{\tilde{\bm{\Delta}}}_{k,2}\|^2]  x_{k,2}^{(j)}}{2n} - \frac{x_{k,2}^{(j)} \left(\sum_{\alpha = 1}^n x_{k,2}^{(\alpha)}\mathbb{E}_k[\tilde{\tilde{\Delta}}_{k,2}^{(\alpha)}]\right)}{n}.
\end{align}

Consequently, it remains to compute 
\begin{align}\label{eq:neededdeltaexpectations}
\mathbb{E}_k[\tilde{\tilde{\Delta}}_{k,2}^{(j)}], \ \ \ \ \  \mathbb{E}_k[\|\tilde{\tilde{\bm{\Delta}}}_{k,2}\|^2 ] .
\end{align}

It follows from \eqref{deltatilttil}  that the analysis of the above terms require evaluating the following expectations:

\begin{align}
\mathbb{E}_k[\bm{g}_{k,1}^{(j)}] , \ \ \
\mathbb{E}_k[\bm{g}_{k,2}^{(j)}],\ \ \ 
\mathbb{E}_k[\bm{g}_{k,2}^\top \bm{g}_{k,1}], \ \ \
\mathbb{E}_k[\bm{g}_{k,1}\bm{g}_{k,2}^\top] . 
\end{align}

In the following Subsections~\ref{app:subsection:Eg} through~\ref{app:subsection:Eg1g2transpose}, we derive the expressions for these individual terms. Building on these results, we derive the expressions in \eqref{eq:neededdeltaexpectations} in Subsections~\ref{app:subsection:Edeltaj} and~\ref{app:subsection:Edeltasqr}. Finally in Subsection \ref{app:subsectionfinalexpforsecondcomponent}, we state the final expression for \eqref{beforecalculation,drift2}.

\subsubsection{Derivation of \texorpdfstring{$\mathbb{E}_k[g^{(j)}_{k,i}]$}{E_k[g_{k,i}^{(j)}]}} \label{app:subsection:Eg}

For the gradient $g_{k,i}^{(j)}$ we have 

\begin{align}
g_{k,i}^{(j)} &= \frac{\tau}{\sqrt{n}}f\!\left(\frac{1}{\sqrt{n}}\,\bm{y}_k^\top \bm{x}_{k,i}\right)y_k^{(j)},
\notag \\ &= \frac{\tau}{\sqrt{n}}f\!\left(\frac{1}{\sqrt{n}}
\Big(\tfrac{1}{\sqrt{n}}\big(c_{k,1}\bm{u}_1^\top + c_{k,2}\bm{u}_2^\top\big) + \bm{a}_k^\top\Big)\bm{x}_{k,i}\right)y_k^{(j)} ,\notag
\\
&=\frac{\tau}{\sqrt{n}} f\!\left(c_{k,1}Q_{k,i,1} + c_{k,2}Q_{k,i,2}
+ \frac{1}{\sqrt{n}}\,\bm{a}_k^\top \bm{x}_{k,i}\right)y_k^{(j)}.
\end{align}
 
Recall that we defined the cosine similarities between the estimates $\bm{x}_{k,i}$ and the true components $\bm{u}_l$, for $i,l \in \{1,\dots,p\}$, by $
Q_{k,i,l} \coloneqq \frac{\bm{x_{k,i}}^\top \bm{u_l} }{n} $. We define $e_{k,i}^{\setminus j} \coloneqq \frac{1}{\sqrt{n}}(\bm{a}_k^\top \bm{x}_{k,i}  - a_k^{(j)}x_{k,i}^{(j)})$. Thus we arrive at

\begin{align}
 g_{k,i}^{(j)} = \frac{\tau}{\sqrt{n}}f\left(c_{k,1}Q_{k,i,1} + c_{k,2}Q_{k,i,2} + e_{k,i}^{\setminus j} + \frac{a_k^{(j)}x_{k,i}^{(j)}}{\sqrt{n}}\right)y_{k}^{(j)}.
\end{align}

In the limit $n \to \infty$, the contribution of any single coordinate is negligible compared to the contribution of all coordinates. We therefore apply a first order Taylor expansion of $f(\cdot)$ around the leading-order term $c_{k,1} Q_{k,i,1} + c_{k,2} Q_{k,i,2} + e_{k,i}^{\setminus j}$:

\begin{align} \label{gradj}
g_{k,i}^{(j)}
&= \frac{\tau}{\sqrt{n}}
f\!\left(
c_{k,1} Q_{k,i,1} + c_{k,2} Q_{k,i,2} +e_{k,i}^{\setminus j}
\right)
\left(
\frac{c_{k,1} u_1^{(j)}}{\sqrt{n}}
+ \frac{c_{k,2} u_2^{(j)}}{\sqrt{n}}
+ a_k^{(j)}
\right)
\notag
\\
&\qquad
+ \frac{\tau}{\sqrt{n}}
f'\!\left(
c_{k,1} Q_{k,i,1} + c_{k,2} Q_{k,i,2} + e_{k,i}^{\setminus j}
\right)
\frac{a_k^{(j)} x_{k,i}^{(j)}}{\sqrt{n}}
\left(
\frac{c_{k,1} u_1^{(j)}}{\sqrt{n}}
+ \frac{c_{k,2} u_2^{(j)}}{\sqrt{n}}
+ a_k^{(j)}
\right) \notag \\ &\qquad  + o\!\left(\frac{1}{n}\right).
\end{align}

\begin{lemma} \label{app:lemmacovmatrix}
  For the random vector
$
\begin{pmatrix}
e_{k,i}^{\setminus j} \\
a_k^{(j)}
\end{pmatrix},
$
its covariance matrix takes the form:
\begin{align}
\operatorname{Cov}
\!\left(
\begin{pmatrix}
e_{k,i}^{\setminus j} \\
a_k^{(j)}
\end{pmatrix}
\right)
=
\begin{bmatrix}
1 - Q_{k,i,1}^2 - Q_{k,i,2}^2
&
-\dfrac{1}{\sqrt{n}}
\big(
u_1^{(j)} Q_{k,i,1}
+
u_2^{(j)} Q_{k,i,2}
\big)
\\[1em]
-\dfrac{1}{\sqrt{n}}
\big(
u_1^{(j)} Q_{k,i,1}
+
u_2^{(j)} Q_{k,i,2}
\big)
&
1
\end{bmatrix}.
\end{align}
\end{lemma}

\begin{proof}

We first compute the cross-correlation between $e_{k,i}^{\setminus j}$ and $a_{k}^{(j)}$. Substituting the definition of $e_{k,i}^{\setminus j}$, we have
\begin{align}
    \mathbb{E}_k[e_{k,i}^{\setminus j} a _{k}^{(j)}] &= \mathbb{E}_k\left[\frac{1}{\sqrt{n}}\left(\sum_{\alpha \neq j}a_k^{(\alpha)}x_{k,i}^{(\alpha)}\right)a_k^{(j)}\right].
\end{align}
Recalling that the vector $\bm{a}_k$ follows the distribution $\bm{a}_k \sim \mathcal{N}\left(\bm{0}, \bm{I}-\frac{1}{n}(\bm{u}_1\bm{u}_1^\top + \bm{u}_2\bm{u}_2^\top)\right)$, we substitute the off-diagonal covariance terms into the expectation:
\begin{align}
    \mathbb{E}_k[e_{k,i}^{\setminus j} a _{k}^{(j)}] &= - \frac{1}{\sqrt{n}}\left(\frac{u_1^{(j)} \sum_{\alpha \neq j} u_1^{(\alpha)} x_{k,i}^{(\alpha)}}{n} + \frac{u_2^{(j)}\sum_{\alpha \neq j} u_2^{(\alpha)} x_{k,i}^{(\alpha)}}{n} \right).
\end{align}
Recognizing the parameters $Q_{k,i,1} = \frac{1}{n} \bm{x}_{k,i}^\top \bm{u}_1 $ and $Q_{k,i,2} = \frac{1}{n} \bm{x}_{k,i}^\top \bm{u}_2 $, this simplifies to
\begin{align}
    \mathbb{E}_k[e_{k,i}^{\setminus j} a _{k}^{(j)}] &= - \frac{1}{\sqrt{n}} \left(u_1^{(j)}Q_{k,i,1} + u_2^{(j)}Q_{k,i,2}\right) + o\left(\frac{1}{\sqrt{n}}\right).
\end{align}

Next, we determine the variance of $e_{k,i}^{\setminus j}$.For notational convenience, let $\bm{x}_{k,i}^{(-j)}$ denote the vector $\bm{x}_{k,i}$ with the $j$-th entry set to zero. This allows us to express the sum as
\begin{align}
    \mathbb{E}_k[(e_{k,i}^{\setminus j})^2] &= \mathbb{E}_k\left[\frac{1}{n} \left(\sum_{\alpha \neq j}a_k^{(\alpha)}x_{k,i}^{(\alpha)}\right)^2 \right] 
    = \mathbb{E}_k\left[\frac{1}{n} \left((\bm{x}_{k,i}^{(-j)})^\top \bm{a}_k\right)^2 \right].
\end{align}
Evaluating the expectation using the covariance matrix of $\bm{a}_k$ yields
\begin{align}
    \mathbb{E}_k[(e_{k,i}^{\setminus j})^2] &= \frac{1}{n}\Bigg( (\bm{x}_{k,i}^{(-j)})^\top \left(\bm{I}-\frac{1}{n}(\bm{u}_1 \bm{u}_1^\top + \bm{u}_2 \bm{u}_2^\top)\right) \bm{x}_{k,i}^{(-j)} \Bigg), \notag  \\
    &= \frac{1}{n} \left(
    \sum_{\alpha \neq j} (x_{k,i}^{(\alpha)})^2
    - \frac{ (\bm{x}_{k,i}^{(-j)})^{\top} \bm{u}_1 \bm{u}_1^{\top} \bm{x}_{k,i}^{(-j)} }{n}
    - \frac{ (\bm{x}_{k,i}^{(-j)})^{\top} \bm{u}_2 \bm{u}_2^{\top} \bm{x}_{k,i}^{(-j)} }{n}
    \right).
\end{align}
Applying $\sum_{\alpha \neq j}(x_{k,i}^{(\alpha)})^2 = \|\bm{x}_{k,i}\|^2 - O(1) = n - O(1)$, and  $\frac{1}{n}\bm{u}_m^\top \bm{x}_{k,i}^{(-j)} = Q_{k,i,m}-O(\frac{1}{n})$ for $m \in \{1,2\}$, we obtain
\begin{align}
\mathbb{E}_k[(e_{k,i}^{\setminus j})^2] &= \frac{1}{n}\Big( n - n Q_{k,i,1}^2 - n Q_{k,i,2}^2 \Big) +O\left(\frac{1}{n}\right),\\
    &= 1 - Q_{k,i,1}^2 - Q_{k,i,2}^2+O\left(\frac{1}{n}\right).
\end{align}
Finally, we note that the variance of $a_{k}^{(j)}$ itself is normalized such that:$\mathbb{E}_k[(a_{k}^{(j)})^2] = 1 - O\left(\frac{1}{n}\right)$. Thus, omitting the terms of $O\left(\frac{1}{n}\right)$ yields the covariance matrix:

\begin{align}
\operatorname{Cov}
\!\left(
\begin{pmatrix}
e_{k,i}^{\setminus j} \\
a_k^{(j)}
\end{pmatrix}
\right)
=
\begin{bmatrix}
1 - Q_{k,i,1}^2 - Q_{k,i,2}^2
&
-\dfrac{1}{\sqrt{n}}
\bigg(
u_1^{(j)} Q_{k,i,1}
+
u_2^{(j)} Q_{k,i,2}
\bigg)
\\[1em]
-\dfrac{1}{\sqrt{n}}
\bigg(
u_1^{(j)} Q_{k,i,1}
+
u_2^{(j)} Q_{k,i,2}
\bigg)
&
1
\end{bmatrix}.
\end{align}
\end{proof}

By Lemma \ref{app:lemmacovmatrix}, the cross-covariance between $e_{k,i}^{\setminus j}$ and $a_k^{(j)}$ is $O(n^{-1/2})$. Therefore  $e_{k,i}^{\setminus j}$ and $a_k^{(j)}$ may be treated as independent up to $O(n^{-1/2})$ accuracy in when evaluating any term whose overall contribution is already $O(n^{-1/2})$. Moreover, we introduce an auxiliary random variable $e_i \sim \mathcal{N}(0,1)$ and represent $e_{k,i}^{\setminus j}$ as
\begin{align}
e_{k,i}^{\setminus j}
=
\sqrt{1 - Q_{k,i,1}^2 - Q_{k,i,2}^2}\, e_i .
\end{align}

It follows from \eqref{gradj} that $\mathbb{E}[g_{k,i}^{(j)}]$ can be written as
\begin{align} 
\mathbb{E}_k[g_{k,i}^{(j)}]
&= \frac{\tau}{\sqrt{n}}
\mathbb{E}\left[f\!\left(
c_{k,1} Q_{k,i,1} + c_{k,2} Q_{k,i,2} +e_{k,i}^{\setminus j}
\right)
\left(
\frac{c_{k,1} u_1^{(j)}}{\sqrt{n}}
+ \frac{c_{k,2} u_2^{(j)}}{\sqrt{n}}
+ a_k^{(j)}
\right)\right]
\notag
\\
&\qquad
+ \frac{\tau}{\sqrt{n}}
\mathbb{E}\Bigg[f'\!\left(
c_{k,1} Q_{k,i,1} + c_{k,2} Q_{k,i,2} + e_{k,i}^{\setminus j}
\right)
\frac{a_k^{(j)} x_{k,i}^{(j)}}{\sqrt{n}} \notag \\ & \qquad \qquad\qquad\qquad\qquad  \times
\left(
\frac{c_{k,1} u_1^{(j)}}{\sqrt{n}}
+ \frac{c_{k,2} u_2^{(j)}}{\sqrt{n}}
+ a_k^{(j)}
\right)\Bigg] + o\!\left(\frac{1}{n}\right), \notag \\& = \frac{\tau}{\sqrt{n}}L_1 + \frac{\tau}{\sqrt{n}}L_2 + o\left(\frac{1}{n}\right),
\end{align}

where 
\begin{align}
L_2 &\coloneqq\mathbb{E}_k\left[f^{\prime}\left(c_{k,1}Q_{k,i,1} + c_{k,2}Q_{k,i,2}
+ e_{k,i}^{\setminus j}\right) \frac{a_k^{(j)}x_{k,i}^{(j)}}{\sqrt{n}}\left(\frac{c_{k,1}u_1^{(j)}}{\sqrt{n}}+\frac{c_{k,2}u_2^{(j)}}{\sqrt{n}}+ a_k^{(j)}\right)\right],\notag \\ &= \frac{x_{k,i}^{(j)}}{\sqrt{n}} \Bigg\langle f\left(c_{k,1}Q_{k,i,1}+c_{k,2}Q_{k,i,2}+e_i\sqrt{1-Q_{k,i,1}^2-Q_{k,i,2}^2}\right)\Bigg\rangle  + o\left(\frac{1}{\sqrt{n}}\right), \\[5ex]
    L_1 &\coloneqq\mathbb{E}_k\Bigg[
        f \left(c_{k,1}Q_{k,i,1} + c_{k,2}Q_{k,i,2} + e_{k,i}^{\setminus j}\right)
        \Bigg(\frac{c_{k,1}u_1^{(j)}}{\sqrt{n}}+\frac{c_{k,2}u_2^{(j)}}{\sqrt{n}}+ a_k^{(j)}\Bigg)
    \Bigg], \notag \\
    &= \mathbb{E}_k\Bigg[
        f\ \bigg(
            c_{k,1}Q_{k,i,1} + c_{k,2}Q_{k,i,2}
            + \sqrt{1 - Q_{k,i,1}^2 - Q_{k,i,2}^2}\, e_i \notag \\
    &\qquad\qquad\quad
            -\frac{1}{\sqrt{n}}\big(u_1^{(j)} Q_{k,i,1} + u_2^{(j)} Q_{k,i,2}\big)\, a_k^{(j)}
        \bigg)
        (\frac{c_{k,1}u_1^{(j)}}{\sqrt{n}}+\frac{c_{k,2}u_2^{(j)}}{\sqrt{n}}+ a_k^{(j)})
    \Bigg], \notag\\&=\frac{u_1^{(j)}}{\sqrt{n}}\Bigg(\Bigg\langle c_{k,1}f\left(c_{1}Q_{k,i,1} + c_{2}Q_{k,i,2} + e_{i}\sqrt{1-Q_{k,i,1}^2 -Q_{k,i,2}^2}\right)\Bigg\rangle \notag \\ & \qquad \qquad \qquad- Q_{k,i,1} \Bigg\langle f^\prime \left( c_{1}Q_{k,i,1} + c_{2}Q_{k,i,2} + e_{i}\sqrt{1-Q_{k,i,1}^2 -Q_{k,i,2}^2} \right) \Bigg\rangle\Bigg)\notag \\ & \qquad  + \frac{u_2^{(j)}}{\sqrt{n}} \Bigg(\Bigg\langle c_{k,2}f\left(c_{1}Q_{k,i,1} + c_{2}Q_{k,i,2} + e_{i}\sqrt{1-Q_{k,i,1}^2 -Q_{k,i,2}^2}\right)\Bigg\rangle \notag \\ & \qquad \qquad \qquad- Q_{k,i,2} \Bigg\langle f^\prime \left( c_{1}Q_{k,i,1} + c_{2}Q_{k,i,2} + e_{i}\sqrt{1-Q_{k,i,1}^2 -Q_{k,i,2}^2} \right) \Bigg\rangle\Bigg)\notag \\ & \qquad + o\left(\frac{1}{\sqrt{n}}\right).
\end{align}

Henceforth, to streamline the notation, we introduce
\begin{align}
\gamma_1 &\coloneqq f\!\Big(c_{1}Q_{k,1,1} + c_{2}Q_{k,1,2} + e_{1}\sqrt{1-Q_{k,1,1}^2 -Q_{k,1,2}^2}\Big), \\
\gamma_2 &\coloneqq f\!\Big(c_{1}Q_{k,2,1} + c_{2}Q_{k,2,2} + e_{2}\sqrt{1-Q_{k,2,1}^2 -Q_{k,2,2}^2}\Big).
\end{align}
More generally, for $i\in\{1,2\}$ we write
\begin{align}
\gamma_i \coloneqq f\!\Big(c_{1}Q_{k,i,1} + c_{2}Q_{k,i,2} + e_{i}\sqrt{1-Q_{k,i,1}^2 -Q_{k,i,2}^2}\Big).
\end{align}

Throughout, $\gamma_i^\prime$ denotes the derivative of $f$ evaluated at the same argument as in $\gamma_i$. Finally, we arrive at 

\begin{align}
\mathbb{E}_k[g_{k,i}^{(j)}] &= \frac{\tau}{n}\Bigg(x_{k,i}^{(j)} \langle \gamma_i^{\prime} \rangle + u_1^{(j)}\Big(\langle c_1\gamma_i \rangle - Q_{k,i,1}\langle \gamma_i^{\prime} \rangle\Big) +  u_2^{(j)}\Big(\langle c_2\gamma_i \rangle - Q_{k,i,2}\langle \gamma_i^{\prime} \rangle\Big)\Bigg) +o\left(\frac{1}{n}\right).
\end{align}

\subsubsection{Derivation of \texorpdfstring{$\mathbb{E}_k[\bm{g}_{k,1}^\top \bm{g}_{k,2}]$}{E_k[g_{k,1}^T g_{k,2}]}}\label{app:subsection:Eg1g2}

Now, we consider the cross term $\mathbb{E}_k[\bm{g}_{k,1}^\top \bm{g}_{k,2}]$. Substituting the definitions of $\bm{g}_{k,1}$ and $\bm{g}_{k,2}$ into the expectation yields

\begin{align}
\mathbb{E}_k\!\left[\bm{g}_{k,1}^\top \bm{g}_{k,2}\right]
&=
\mathbb{E}_k\!\left[
\frac{\tau^2}{n}\,
f\!\left(\frac{\bm{y}_k^\top \bm{x}_{k,1}}{\sqrt{n}}\right)
f\!\left(\frac{\bm{y}_k^\top \bm{x}_{k,2}}{\sqrt{n}}\right)\,
\bm{y}_k^\top \bm{y}_k
\right]
+ o\!\left(\frac{1}{n}\right),
\notag \\
&=
\frac{\tau^2}{n}\,
\mathbb{E}_k\!\Bigg[
\Big(\gamma_1 + \gamma_1'\,\frac{a_k^{(j)} x_{k,1}^{(j)}}{\sqrt{n}}\Big)
\Big(\gamma_2 + \gamma_2'\,\frac{a_k^{(j)} x_{k,2}^{(j)}}{\sqrt{n}}\Big)
\notag \\ & \qquad \qquad \qquad \qquad \qquad \qquad \qquad \times \Big( c_{k,1}^2 + c_{k,2}^2 + \sum_{\alpha=1}^n a_k^{(\alpha)} a_k^{(\alpha)} \Big)
\Bigg],
\notag \\
&=
\frac{\tau^2}{n}\,
\mathbb{E}_k\!\Bigg[
\Big(
\gamma_1\gamma_2
+ \gamma_1'\gamma_2'\,\frac{(a_k^{(j)})^2 x_{k,1}^{(j)} x_{k,2}^{(j)}}{n}
+ \gamma_2\,\gamma_1'\,\frac{a_k^{(j)} x_{k,1}^{(j)}}{\sqrt{n}} \notag \\ & \qquad \qquad \qquad \quad
+ \gamma_1\,\gamma_2'\,\frac{a_k^{(j)} x_{k,2}^{(j)}}{\sqrt{n}}
\Big)
\Big( c_{k,1}^2 + c_{k,2}^2 + \sum_{\alpha=1}^n a_k^{(\alpha)} a_k^{(\alpha)} \Big) \Bigg] \ .
\end{align}

We begin by noting the following expectations:
\begin{align}
    \mathbb{E}_k\left[(a_{k}^{(j)})^2\right] = 1 + o(1), \quad
    \mathbb{E}_k\left[a_{k}^{(j)}e_{k,i}^{\setminus j}\right] = o(1), \quad
    \sum_{\alpha = 1}^n \mathbb{E}_k\left[(a_k^{(\alpha)})^2\right] = n + O(1). \notag
\end{align}

Following these, we can expand the expectation as follows:
\begin{align}
    \mathbb{E}_k\!\left[\bm{g}_{k,1}^\top \bm{g}_{k,2}\right] &= \frac{\tau^2}{n}\mathbb{E}_k\ \Bigg[\Bigg(\gamma_1 \gamma_2+\gamma_1'\gamma_2'\frac{(a_k^{(j)})^2 x_{k,1}^{(j)}x_{k,2}^{(j)}}{n}+ \gamma_2 \gamma_1'\frac{a_k^{(j)}x_{k,1}^{(j)}}{n} \notag \\
    &\qquad \qquad\qquad \ \ + \gamma_2' \gamma_1 \frac{a_k^{(j)}x_{k,2}^{(j)}}{n} \Bigg) \left( c_{k,1}^2 + c_{k,2}^2 \right) \Bigg] \notag \\
    &\qquad + \frac{\tau^2}{n}\mathbb{E}_k\ \Bigg[\Bigg(\gamma_1 \gamma_2+\gamma_1'\gamma_2'\frac{(a_k^{(j)})^2 x_{k,1}^{(j)}x_{k,2}^{(j)}}{n}+ \gamma_2 \gamma_1'\frac{a_k^{(j)}x_{k,1}^{(j)}}{n} \notag \\ & \qquad \qquad \qquad \qquad \qquad \qquad \qquad \qquad
    + \gamma_2' \gamma_1 \frac{a_k^{(j)}x_{k,2}^{(j)}}{n} \Bigg) \left( \sum_{\alpha = 1}^n (a_k^{(\alpha)})^2\right) \Bigg].
\end{align}

Since the contributions from the auxiliary terms are of order $o(1/n)$, we obtain the final expression:

\begin{align}
    \mathbb{E}_k[\bm{g}_{k,1}^\top \bm{g}_{k,2}] = \tau^2 \langle \gamma_1 \gamma_2 \rangle + o(1).
\label{g1trasnposeg2}
\end{align}
\subsubsection{Derivation of \texorpdfstring{$\mathbb{E}_k[\bm{g}_{k,1}\bm{g}_{k,2}^\top]$}{E_k[g_{k,1} g_{k,2}^T]}}\label{app:subsection:Eg1g2transpose}

Substituting the definitions of $\bm{g}_{k,1}$ and $\bm{g}_{k,2}$ yields
\begin{align}\label{g1g2transpose_before_calculation}
\mathbb{E}_k[\bm{g}_{k,1}\bm{g}_{k,2}^\top] &= \mathbb{E}_k\left[\frac{\tau^2}{n} f\left(\frac{\bm{y}_k^\top \bm{x}_{k,1}}{\sqrt{n}}\right)f\left(\frac{\bm{y}_k^\top \bm{x}_{k,2}}{\sqrt{n}}\right)\bm{y}_k \bm{y}_k^\top\right] + o\left(\frac{1}{n}\right),\notag \\
    &= \frac{\tau^2}{n}\mathbb{E}_k\Bigg[\Bigg(\gamma_1 \gamma_2+\gamma_1'\gamma_2'\frac{(a_k^{(j)})^2 x_{k,1}^{(j)}x_{k,2}^{(j)}}{n} \notag \\ & \qquad \qquad \qquad + \gamma_2 \gamma_1'\frac{a_k^{(j)}x_{k,1}^{(j)}}{n} + \gamma_2' \gamma_1 \frac{a_k^{(j)}x_{k,2}^{(j)}}{n} \Bigg)\bm{y}_k \bm{y}_k^\top \Bigg].
\end{align}

We define the matrices $\bm{K} \coloneqq \bm{c}_k \bm{c}_k^\top$ and  $\bm{A} \coloneqq \bm {a}_k \bm{a}_k^\top$. Their expectations are given by
\begin{align}
    \mathbb{E}_k[\bm{A}] = \bm{I} - \frac{\bm{UU}^\top}{n}, \quad \text{and} \quad \mathbb{E}_k[\bm{K}] = \bm{I}.
\end{align}

Consequently, we arrive at the following for $\bm{y}_k\bm{y}_k^\top$:
\textbf{\begin{align}
    \bm{y}_k \bm{y}_k^\top &= \frac{\bm{U \bm{K} U^\top}}{n} + \bm{A} + \frac{\bm{U c }\bm{a}_k^\top}{\sqrt{n}} + \frac{\bm{a_k c^\top U}}{\sqrt{n}}.
\end{align}}

Substituting this into Equation \eqref{g1g2transpose_before_calculation} yields the final form of the desired expectation:

\begin{align}\label{g1g2transpose}
\mathbb{E}_k[\bm{g}_{k,1}\bm{g}_{k,2}^\top]  &=  \frac{\tau^2}{n}\mathbb{E}_k[\gamma_1 \gamma_2 \bm{A} ] + o\left(\frac{1}{n}\right) ,\notag \\
    &= \frac{\tau^2}{n} \langle \gamma_1 \gamma_2 \rangle \left(\bm{I} - \frac{\bm{UU^\top}}{n}\right)+o\left(\frac{1}{n}\right) ,\notag \\
    &= \frac{\tau^2}{n} \langle \gamma_1 \gamma_2 \rangle \bm{I} + o\left(\frac{1}{n}\right).
\end{align}
\subsubsection{Derivation of $\mathbb{E}_k[\tilde{\tilde{\Delta}}_{k,2}^{(j)}]$}\label{app:subsection:Edeltaj}

Taking the expectation of \eqref{deltatilttil}, we arrive at the following:

\begin{align}
    \mathbb{E}_k[\tilde{\tilde{\Delta}}_{k,2}^{(j)}] 
    &= \mathbb{E}_k[g_{k,2}^{(j)}] 
    - x_{k,1}^{(j)}\frac{\sum_{\alpha = 1}^n x_{k,2}^{(\alpha)}\mathbb{E}_k[g_{k,1}^{(\alpha)}]}{n} 
    - x_{k,1}^{(j)}\frac{\sum_{\alpha = 1}^n x_{k,1}^{(\alpha)}\mathbb{E}_k[g_{k,2}^{(\alpha)}]}{n} \notag \\
    &\qquad + 2x_{k,1}^{(j)} \frac{\sum_{\alpha = 1}^n x_{k,2}^{(\alpha)}x_{k,1}^{(\alpha)}\mathbb{E}_k[(g_{k,1}^{(\alpha)})^2]}{n^2} + 2x_{k,1}^{(j)} \frac{\sum_{\alpha = 1}^n (x_{k,1}^{(\alpha)})^2\mathbb{E}_k[g_{k,1}^{(\alpha)}g_{k,2}^{(\alpha)}]}{n^2} \notag \\
    &\qquad
    - \frac{x_{k,2}^{(j)}\mathbb{E}_k[(g_{k,1}^{(j)})^2]}{n} 
    - \frac{x_{k,1}^{(j)}\mathbb{E}_k[g_{k,2}^{(j)}g_{k,1}^{(j)}]}{n}  - x_{k,1}^{(j)}\frac{\mathbb{E}_k[\bm{g}_{k,2}^\top \bm{g}_{k,1}]}{n}.
\end{align}

 Moreover the expectations $\mathbb{E}_k[(g_{k,i}^{(j)})^2]$, $\mathbb{E}_k[g_{k,1}^{(j)}g_{k,2}^{(j)}]$ and $\mathbb{E}_k[\bm{g}_{k,1}^\top \bm{g}_{k,2}]$ were found to be independent of index $(j)$, they can be taken outside the summations, yielding

\begin{align}
\mathbb{E}_k[\tilde{\tilde{\Delta}}_{k,2}^{(j)}]
&= \mathbb{E}_k[g_{k,2}^{(j)}]
- x_{k,1}^{(j)}\frac{\sum_{\alpha = 1}^n x_{k,2}^{(\alpha)}\mathbb{E}_k[g_{k,1}^{(\alpha)}]}{n}
- x_{k,1}^{(j)}\frac{\sum_{\alpha = 1}^n x_{k,1}^{(\alpha)}\mathbb{E}_k[g_{k,2}^{(\alpha)}]}{n}
 \notag \\
&\qquad \ \  + 2x_{k,1}^{(j)} \frac{\tau^2\langle \gamma_1^2 \rangle \,\overbrace{\bm{x}_{k,2}^\top \bm{x}_{k,1}}^{=\,0}}{n^3} + 2x_{k,1}^{(j)} \frac{\tau^2 \langle \gamma_1 \gamma_2\rangle \, n}{n^3} \notag \\
&\qquad \ \ - \frac{x_{k,2}^{(j)}\tau^2 \langle \gamma_1 ^2\rangle}{n^2}
- \frac{x_{k,1}^{(j)}\tau^2 \langle \gamma_1 \gamma_2\rangle}{n^2} - x_{k,1}^{(j)}\frac{\tau^2 \langle \gamma_1 \gamma_2\rangle}{n}, \notag\\
& = \mathbb{E}_k[{g_{k,2}^{(j)}}] - \underbrace{x_{k,1}^{(j)}\frac{\sum_{\alpha = 1}^n x_{k,2}^{(\alpha)}\mathbb{E}_k[g_{k,1}^{(\alpha)}]}{n} }_{T_1}- \underbrace{x_{k,1}^{(j)}\frac{\sum_{\alpha = 1}^n x_{k,1}^{(\alpha)}\mathbb{E}_k[g_{k,2}^{(\alpha)}]}{n}}_{T_2} \notag \\ & \qquad \ \ 
 + x_{k,1}^{(j)} \tau^2 \langle \gamma_1 \gamma_2\rangle(-\frac{1}{n}+\frac{1}{n^2}) 
 - \frac{x_{k,2}^{(j)}\tau^2 \langle \gamma_1 ^2\rangle}{n^2}.
\end{align}

We evaluate the $\alpha$-summation terms labeled $T_1$ and $T_2$ separately:

\begin{align}
  T_1 &= x_{k,1}^{(j)}\frac{\sum_{\alpha = 1}^n x_{k,2}^{(\alpha)}\mathbb{E}_k[g_{k,1}^{(\alpha)}]}{n}, \notag \\&=
  \frac{x_{k,1}^{(j)}\tau}{n\sqrt{n}}\sum_{\alpha = 1}^n x_{k,2}^{(\alpha)}\Bigg(\frac{1}{\sqrt{n}}x_{k,1}^{(\alpha)} \langle \gamma_1^{'} \rangle + \frac{u_1^{(\alpha)}}{\sqrt{n}}(\langle c_1\gamma_1 \rangle - Q_{k,1,1}\langle \gamma_1^{'} \rangle)  \notag \\ & \qquad\qquad\qquad\qquad\qquad +  \frac{u_2^{(\alpha)}}{\sqrt{n}}(\langle c_2\gamma_1 \rangle - Q_{k,1,2}\langle \gamma_1^{'} \rangle)\Bigg), \notag\\
&=
 \frac{x_{k,1}^{(j)}\tau}{n^{3/2}}\Bigg(\frac{Q_{k,2,1}n}{\sqrt{n}}(\langle c_1\gamma_1 \rangle - Q_{k,1,1}\langle \gamma_1^{'} \rangle) \notag \\ & \qquad\qquad\qquad\qquad\qquad + \frac{Q_{k,2,2}n}{\sqrt{n}}(\langle c_2\gamma_1 \rangle - Q_{k,1,2}\langle \gamma_1^{'} \rangle)\Bigg), \notag\\
&=
 \frac{x_{k,1}^{(j)}\tau}{n}\Bigg(Q_{k,2,1}(\langle c_1\gamma_1 \rangle - Q_{k,1,1}\langle \gamma_1^{'} \rangle)  \notag \\ & \qquad\qquad\qquad\qquad\qquad + Q_{k,2,2}(\langle c_2\gamma_1 \rangle - Q_{k,1,2}\langle \gamma_1^{'} \rangle)\Bigg), \\
 T_2 &= x_{k,1}^{(j)}\frac{\sum_{\alpha = 1}^n x_{k,1}^{(\alpha)}\mathbb{E}_k[g_{k,2}^{(\alpha)}]}{n}, \notag\\  &= \frac{x_{k,1}^{(j)}\tau}{n}\Bigg(Q_{k,1,1}(\langle c_1\gamma_2 \rangle - Q_{k,2,1}\langle \gamma_2^{'} \rangle)  \notag \\ & \qquad\qquad\qquad\qquad\qquad + Q_{k,1,2}((\langle c_2\gamma_2 \rangle - Q_{k,2,2}\langle \gamma_2^{'} \rangle)\Bigg).
\end{align}

Finally, we arrive at the following for the final expression of $\mathbb{E}_k[\tilde{\tilde{\Delta}}_{k,2}^{(j)}]$:

\begin{align}\label{firsttermforx2}
\mathbb{E}_k[\tilde{\tilde{\Delta}}_{k,2}^{(j)}] &=-\frac{x_{k,1}^{(j)}\tau^2 \langle \gamma_1 \gamma_2 \rangle}{n} \notag \\
& \qquad+ \frac{\tau}{\sqrt{n}}\Bigg(\frac{1}{\sqrt{n}}x_{k,2}^{(j)} \langle \gamma_2^{'} \rangle + \frac{u_1^{(j)}}{\sqrt{n}}\left(\langle c_1\gamma_2 \rangle - Q_{k,2,1}\langle \gamma_2^{'} \rangle\right)  \notag \\ & \qquad \qquad \qquad \qquad \qquad \qquad 
+ \frac{u_2^{(j)}}{\sqrt{n}}\left(\langle c_2\gamma_2 \rangle - Q_{k,2,2}\langle \gamma_2^{'} \rangle\right)\Bigg) \notag \\
& \qquad -\frac{x_{k,1}^{(j)}\tau}{n}\Bigg(Q_{k,1,1}\left(\langle c_1\gamma_2 \rangle - Q_{k,2,1}\langle \gamma_2^{'} \rangle\right) + Q_{k,1,2}\left(\langle c_2\gamma_2 \rangle - Q_{k,2,2}\langle \gamma_2^{'} \rangle\right)\Bigg) \notag \\
& \qquad -\frac{x_{k,1}^{(j)}\tau}{n}\Bigg(Q_{k,2,1}\left(\langle c_1\gamma_1 \rangle - Q_{k,1,1}\langle \gamma_1^{'} \rangle\right) + Q_{k,2,2}\left(\langle c_2\gamma_1 \rangle - Q_{k,1,2}\langle \gamma_1^{'} \rangle\right)\Bigg) \notag \\ & \qquad+ o\left(\frac{1}{n}\right).
\end{align}

\subsubsection{Derivation of $\mathbb{E}_k[\|\tilde{\tilde{\bm{\Delta}}}_{k,2}\|^2]$}\label{app:subsection:Edeltasqr}
Following Equation \eqref{deltatilttil}, we can obtain the squared form as follows:
\begin{align}
\tilde{\tilde{\Delta}}_{k,2}^{(j)2} &= (g_{k,2}^{(j)})^2 + (x_{k,1}^{(j)})^2 \left(\frac{\bm{x}_{k,2}^\top \bm{g}_{k,1} + \bm{g}_{k,2}^\top \bm{x}_{k,1}}{n}\right)^2 \notag \\ & \qquad - 2x_{k,1}^{(j)}g_{k,2}^{(j)}\left(\frac{\bm{x}_{k,2}^\top \bm{g}_{k,1} + \bm{g}_{k,2}^\top \bm{x}_{k,1}}{n} \right) + o\left(\frac{1}{n}\right).
\end{align}

Recalling that $\mathbb{E}_k[\bm{g_{k,1}g_{k,2}^\top}]$ is diagonal (as derived in subsection \ref{app:subsection:Eg1g2transpose}) taking the expectation yields

\begin{align}
\mathbb{E}_k[(\tilde{\tilde{\Delta}}_{k,2}^{(j)})^2] &=  \notag
\mathbb{E}_k[(g_{k,2}^{(j)})^2] - 2\frac{x_{k,1}^{(j)}(x_{k,2}^{(j)}\mathbb{E}_k[g_{k,2}^{(j)}g_{k,1}^{(j)}] + x_{k,1}^{(j)}\mathbb{E}_k[(g_{k,2}^{(j)})^2])}{n} \\ 
&\qquad + x_{k,1}^{(j)2} \frac{1}{n^2} \Big(\mathbb{E}_k[(g_{k,1}^{(j)})^2]\underbrace{\sum_\alpha (x_{k,2}^{(\alpha)})^2}_{= \ n} + \mathbb{E}_k[g_{k,2}^{(j)2}] \underbrace{\sum_\alpha (x_{k,1}^{(\alpha)})^2}_{= \ n} \notag \\ & \qquad\qquad\qquad\qquad + 2 \mathbb{E}_k[g_{k,1}^{(j)}g_{k,2}^{(j)}] \underbrace{\sum_\alpha x_{k,1}^{(\alpha)}x_{k,2}^{(\alpha)}}_{ = \ 0} \Big) \ .
\end{align}

Substituting our previous results \eqref{g1trasnposeg2} and \eqref{g1g2transpose} into the expression, we arrive at

\begin{align}
\mathbb{E}_k[(\tilde{\tilde{\Delta}}_{k,2}^{(j)})^2] 
 &= \frac{\tau^2 \langle \gamma_2 ^2\rangle}{n}  - 2\frac{x_{k,1}^{(j)}(x_{k,2}^{(j)}\tau^2 \langle\gamma_2 \gamma_1 \rangle /n+ x_{k,1}^{(j)}\tau^2\langle\gamma_2^2 \rangle/n)}{n} +  \frac{\tau^2(x_{k,1}^{(j)})^2}{n^2}( \langle \gamma_1^2\rangle + \langle \gamma_2^2\rangle  ) ,  \notag \\
 &= \frac{\tau^2 \langle \gamma_2 ^2\rangle}{n}  - 2\frac{x_{k,1}^{(j)}(x_{k,2}^{(j)}\tau^2 \langle\gamma_2 \gamma_1 \rangle )}{n^2} + (x_{k,1}^{(j)})^2 \frac{\tau^2}{n^2}( \langle \gamma_1^2\rangle - \langle \gamma_2^2\rangle  )  \ .
\end{align}

Finally, this allows us to directly compute the expectation $\mathbb{E}_k[||\tilde{\tilde{\bm{\Delta}}}_{k,2}||^2]$ as follows:
\begin{align}\label{secondtermforx2}
\mathbb{E}_k[||\tilde{\tilde{\bm{\Delta}}}_{k,2}||^2] = \mathbb{E}_k[\sum_{\alpha = 1}^n(\tilde{\tilde{\Delta}}_{k,2}^{(\alpha)})^2] &= \tau^2 \langle \gamma_2 ^2\rangle - 2\frac{\tau^2 \langle\gamma_2 \gamma_1 \rangle \overbrace{\sum_{\alpha = 1}^n x_{k,1}^{(\alpha)}x_{k,2}^{(\alpha)}}^{= \ 0}}{n^2}, \notag \\& \qquad+ \frac{\tau^2}{n^2}( \langle \gamma_1^2\rangle - \langle \gamma_2^2\rangle  )\overbrace{\sum_{\alpha = 1} ^n (x_{k,1}^{(\alpha)})^2}^{=\ n}, \notag \\
&= \tau^2 \langle \gamma_2 ^2\rangle + \frac{\tau^2}{n}( \langle \gamma_1^2\rangle - \langle \gamma_2^2\rangle  ),  \notag \\
&= \tau^2 \langle \gamma_2 ^2\rangle + o\left(\frac{1}{n}\right) \ .
\end{align}

\subsubsection{Final expressions for the second component} \label{app:subsectionfinalexpforsecondcomponent}
Recall that we obtained the following in \eqref{beforecalculation,drift2}; we restate the equation here for convenience:
\begin{align}\label{second_estimate_expectation_with_deltas}
    \mathbb{E}_k[x_{k+1,2}^{(j)} - x_{k,2}^{(j)}] =  \mathbb{E}_k[\tilde{\tilde{\Delta}}_{k,2}^{(j)}] - \mathbb{E}_k[\|\tilde{\tilde{\bm{\Delta}}}_{k,2}\|^2]  x_{k,2}^{(j)}/2n - x_{k,2}^{(j)} \Big(\sum_{\alpha = 1}^n x_{k,2}^{(\alpha)}\mathbb{E}_k[\tilde{\tilde{\Delta}}_{k,2}^{(\alpha)}]\Big)/n.
\end{align}

Having determined the first two terms involving the expectation of $\tilde{\tilde{\bm{\Delta}}}_{k,2}$ in Subsections \ref{app:subsection:Edeltaj} and \ref{app:subsection:Edeltasqr}, we now proceed to compute the last term:

\begin{align}
&    x_{k,2}^{(j)}\frac{1}{n} \sum_{\alpha = 1 }^n x_{k,2}^{(\alpha)}\mathbb{E}_k[\tilde{\tilde{\Delta}}_{k,2}^{(\alpha)}] = \notag \\
 & x_{k,2}^{(j)}\frac{1}{n} \sum_{\alpha = 1 }^n x_{k,2}^{(\alpha)} \Bigg(-\frac{x_{k,1}^{(\alpha)}\tau^2 \langle \gamma_1 \gamma_2 \rangle}{n}  + \frac{\tau}{\sqrt{n}}\Big(\frac{1}{\sqrt{n}}x_{k,2}^{(\alpha)} \langle \gamma_2^{'} \rangle + \frac{u_1^{(\alpha)}}{\sqrt{n}}(\langle c_1\gamma_2 \rangle - Q_{k,2,1}\langle \gamma_2^{'} \rangle) \notag \\
&\qquad\qquad\qquad\qquad\qquad\qquad  \qquad \qquad \qquad +  \frac{u_2^{(\alpha)}}{\sqrt{n}}(\langle c_2\gamma_2 \rangle - Q_{k,2,2}\langle \gamma_2^{'} \rangle)\Big) \notag \\
&\qquad \qquad \qquad  \quad -\frac{x_{k,1}^{(\alpha)}\tau}{n}\Big(Q_{k,1,1}(\langle c_1\gamma_2 \rangle - Q_{k,2,1}\langle \gamma_2^{'} \rangle) + Q_{k,1,2}(\langle c_2\gamma_2 \rangle - Q_{k,2,2}\langle \gamma_2^{'} \rangle)\Big) \notag \\
&\qquad \qquad \qquad   \quad -\frac{x_{k,1}^{(\alpha)}\tau}{n}\Big(Q_{k,2,1}(\langle c_1\gamma_1 \rangle - Q_{k,1,1}\langle \gamma_1^{'} \rangle) + Q_{k,2,2}(\langle c_2\gamma_1 \rangle - Q_{k,1,2}\langle \gamma_1^{'} \rangle)\Big)\Bigg).
\end{align}

Realize that the terms involving $\sum_{\alpha = 1}^n x_{k,2}^{(\alpha)}x_{k,1}^{(\alpha)} = 0$ because of the orthogonality constraint of the previous iteration step, thus we arrive at the following for the corresponding term:

\begin{align}\label{thirdtermforx2}
x_{k,2}^{(j)}\frac{1}{n} \sum_{\alpha =1}^n x_{k,2}^{(\alpha)}\mathbb{E}_k[\tilde{\tilde{\Delta}}_{k,2}^{(\alpha)}] 
 &=  x_{k,2}^{(j)}\frac{1}{n}\frac{\tau}{\sqrt{n}}\Bigg(\frac{1}{\sqrt{n}}\langle \gamma_2^{'} \rangle \sum_{\alpha =1}^n x_{k,2}^{(\alpha)}x_{k,2}^{(\alpha)} \notag \\ 
 & \qquad\qquad\qquad\qquad + \frac{\sum_{\alpha =1}^n x_{k,2}^{(\alpha)} u_1^{(\alpha)}}{\sqrt{n}}(\langle c_1\gamma_2 \rangle - Q_{k,2,1}\langle \gamma_2^{'} \rangle) \notag \\ 
 & \qquad\qquad\qquad\qquad   +  \frac{\sum_{\alpha =1}^n x_{k,2}^{(\alpha)}u_2^{(\alpha)}}{\sqrt{n}}(\langle c_2\gamma_2 \rangle - Q_{k,2,2}\langle \gamma_2^{'} \rangle)\Bigg) , \notag \\
   &= x_{k,2}^{(j)}\frac{1}{n}\frac{\tau}{\sqrt{n}}\Bigg(\frac{1}{\sqrt{n}}n \langle \gamma_2^{'} \rangle + \frac{Q_{2,1}n}{\sqrt{n}}(\langle c_1\gamma_2 \rangle - Q_{k,2,1}\langle \gamma_2^{'} \rangle) \notag \\ 
 & \qquad\qquad\qquad\qquad +  \frac{Q_{2,2}n}{\sqrt{n}}(\langle c_2\gamma_2 \rangle - Q_{k,2,2}\langle \gamma_2^{'} \rangle)\Bigg),  \notag \\
   &= x_{k,2}^{(j)}\frac{\tau}{\sqrt{n}}\Bigg(\frac{1}{\sqrt{n}} \langle \gamma_2^{'} \rangle + \frac{Q_{2,1}}{\sqrt{n}}(\langle c_1\gamma_2 \rangle - Q_{k,2,1}\langle \gamma_2^{'} \rangle) \notag \\ 
 & \qquad\qquad\qquad\qquad +  \frac{Q_{2,2}}{\sqrt{n}}(\langle c_2\gamma_2 \rangle - Q_{k,2,2}\langle \gamma_2^{'} \rangle)\Bigg).
\end{align}

Having derived the expression above, we finally substitute all our previous calculations \eqref{firsttermforx2},\eqref{secondtermforx2} and \eqref{thirdtermforx2} into \eqref{second_estimate_expectation_with_deltas} to get
\begin{align}
    \mathbb{E}_k[x_{k+1,2}^{(j)} - x_{k,2}^{(j)}] &= -\frac{x_{k,1}^{(j)}\tau^2 \langle \gamma_1 \gamma_2 \rangle}{n} \notag \\
&\qquad  + \frac{\tau}{\sqrt{n}}\Bigg(\frac{1}{\sqrt{n}}x_{k,2}^{(j)} \langle \gamma_2^{'} \rangle + \frac{u_1^{(j)}}{\sqrt{n}}\Big(\langle c_1\gamma_2 \rangle - Q_{k,2,1}\langle \gamma_2^{'} \rangle\Big) \notag \\
&\qquad\qquad\qquad+  \frac{u_2^{(j)}}{\sqrt{n}}\Big(\langle c_2\gamma_2 \rangle - Q_{k,2,2}\langle \gamma_2^{'} \rangle\Big)\Bigg) \notag \\
&\qquad -\frac{x_{k,1}^{(j)}\tau}{n}\Bigg(Q_{k,1,1}\Big(\langle c_1\gamma_2 \rangle - Q_{k,2,1}\langle \gamma_2^{'} \rangle\Big) + Q_{k,1,2}\Big(\langle c_2\gamma_2 \rangle - Q_{k,2,2}\langle \gamma_2^{'} \rangle\Big)\Bigg)  \notag \\
&\qquad-\frac{x_{k,1}^{(j)}\tau}{n}\Bigg(Q_{k,2,1}\Big(\langle c_1\gamma_1 \rangle - Q_{k,1,1}\langle \gamma_1^{'} \rangle\Big) + Q_{k,2,2}\Big(\langle c_2\gamma_1 \rangle - Q_{k,1,2}\langle \gamma_1^{'} \rangle\Big)\Bigg) \notag \\
&\qquad
- \frac{x_{k,2}^{(j)}}{2n}\tau^2 \langle \gamma_2 ^2\rangle 
 \notag \\
&\qquad -x_{k,2}^{(j)}\frac{\tau}{\sqrt{n}}\Bigg(\frac{1}{\sqrt{n}} \langle \gamma_2^{'} \rangle + \frac{Q_{2,1}}{\sqrt{n}}\Big(\langle c_1\gamma_2 \rangle - Q_{k,2,1}\langle \gamma_2^{'} \rangle\Big) \notag \\
&\qquad\qquad\qquad\qquad
+  \frac{Q_{2,2}}{\sqrt{n}}\Big(\langle c_2\gamma_2 \rangle - Q_{k,2,2}\langle \gamma_2^{'} \rangle\Big)\Bigg) + o\left(\frac{1}{n}\right).
\end{align}

The diffusion term, the second moment of the incremental update for the $j-th$ element, can be stated as follows:

\begin{align}
    \mathbb{E}_k[(x_{k+1,2}^{(j)} - x_{k,2}^{(j)})^2] &= \mathbb{E}_k[(g_{k,2}^{(j)})^2] - 2\frac{x_{k,1}^{(j)}\Big(x_{k,2}^{(j)}\mathbb{E}_k[g_{k,2}^{(j)}g_{k,1}^{(j)}] + x_{k,1}^{(j)}\mathbb{E}_k[(g_{k,2}^{(j)})^2]\Big)}{n} \notag \\
    &\qquad + x_{k,1}^{(j)2} \frac{1}{n^2}\Big(x_{k,2}^{(j)2}\mathbb{E}_k[(g_{k,2}^{(j)})^2] + x_{k,1}^{(j)2}\mathbb{E}_k[(g_{k,2}^{(j)})^2] \notag \\ & \qquad \qquad \qquad \qquad + 2x_{k,1}^{(j)}x_{k,2}^{(j)}\mathbb{E}_k[g_{k,1}^{(j)}g_{k,2}^{(j)}]\Big), \notag \\
 &= \frac{\tau^2\langle \gamma_2^2\rangle}{n} - 2\frac{x_{k,1}^{(j)}(x_{k,2}^{(j)}\frac{\tau^2 \langle \gamma_1 \gamma_2 \rangle}{n} + x_{k,1}^{(j)}\frac{\tau^2 \langle \gamma_2^2\rangle}{n})}{n} \notag \\
 & \qquad+ x_{k,1}^{(j)2} \frac{1}{n^2}\Big(x_{k,2}^{(j)2}\frac{\tau^2 \langle \gamma_1^2 \rangle}{n} 
 + x_{k,1}^{(j)2}\frac{\tau^2 \langle \gamma_2^2 \rangle}{n} + 2x_{k,1}^{(j)}x_{k,2}^{(j)}\frac{\tau^2 \langle \gamma_1 \gamma_2 \rangle}{n}\Big), \notag \\
 &= \frac{\tau^2\langle \gamma_2^2\rangle}{n} + o\left(\frac{1}{n}\right).
\end{align}

Combining the preceding results, we arrive at the final expressions for the second estimate $\bm{x}_{k,2}$:

\begin{align}
\mathbb{E}_k\big[x_{k+1,2}^{(j)} - x_{k,2}^{(j)}\big]
&= -\frac{x_{k,1}^{(j)}\tau^2 \langle \gamma_1 \gamma_2 \rangle}{n}
 + \frac{\tau}{\sqrt{n}}\Big(
      \frac{1}{\sqrt{n}}x_{k,2}^{(j)} \langle \gamma_2' \rangle
    + \frac{u_1^{(j)}}{\sqrt{n}}\big(\langle c_1\gamma_2 \rangle - Q_{k,2,1}\langle \gamma_2' \rangle\big) \notag \\
&\quad + \frac{u_2^{(j)}}{\sqrt{n}}\big(\langle c_2\gamma_2 \rangle - Q_{k,2,2}\langle \gamma_2' \rangle\big)\Big)
 - \frac{x_{k,2}^{(j)}}{2n}\tau^2 \langle \gamma_2^2 \rangle \notag \\
&\quad
 - \frac{x_{k,1}^{(j)}\tau}{n}\Big(
      Q_{k,1,1}\big(\langle c_1\gamma_2 \rangle - Q_{k,2,1}\langle \gamma_2' \rangle\big)
    + Q_{k,1,2}\big(\langle c_2\gamma_2 \rangle - Q_{k,2,2}\langle \gamma_2' \rangle\big)\Big) \notag \\
&\quad
 - \frac{x_{k,1}^{(j)}\tau}{n}\Big(
      Q_{k,2,1}\big(\langle c_1\gamma_1 \rangle - Q_{k,1,1}\langle \gamma_1' \rangle\big)
    + Q_{k,2,2}\big(\langle c_2\gamma_1 \rangle - Q_{k,1,2}\langle \gamma_1' \rangle\big)\Big) \notag \\
&\quad -x_{k,2}^{(j)}\frac{\tau}{\sqrt{n}}\Big(\frac{1}{\sqrt{n}} \langle \gamma_2^{'} \rangle + \frac{Q_{2,1}}{\sqrt{n}}(\langle c_1\gamma_2 \rangle - Q_{k,2,1}\langle \gamma_2^{'} \rangle) \notag \\ & \qquad\qquad\qquad\qquad +  \frac{Q_{2,2}}{\sqrt{n}}(\langle c_2\gamma_2 \rangle - Q_{k,2,2}\langle \gamma_2^{'} \rangle)\Big) + o\left(\frac{1}{n}\right),
\end{align}

\begin{align}
\mathbb{E}_k[(x_{k+1,2}^{(j)} - x_{k,2}^{(j)})^2] =  \frac{\tau^2\langle \gamma_2^2\rangle}{n} + o\left(\frac{1}{n}\right).
\end{align}

\subsection{Dynamics of the first estimate} \label{app:firstestimate}

For the initial estimate, where Gram-Schmidt orthogonalization is not applicable ($\tilde{\tilde{\bm{x}}}_{k,1} = \tilde{\bm{x}}_{k,1}$), we substitute the update expression directly into the normalization step. Applying a first-order Taylor expansion to the resulting normalization factor, we obtain

\begin{align} \label{equationxupdate}
\bm{x}_{k+1,1} &=  \frac{{\bm{x}_{k,1}+\bm{g}_{k,1}}}{\sqrt{1 + \|\bm{g}_{k,1}\|^2/n +2\bm{x}_{k,1}^\top \bm{g}_{k,1}/n}} ,\notag  \\ 
&= \bm{x}_{k,1} - \bm{x}_{k,1}\frac{\|\bm{g}_{k,1}\|^2}{2n} - \bm{x}_{k,1}\frac{\bm{x}_{k,1}^\top \bm{g}_{k,1}}{n} +\bm{g}_{k,1} \notag \\ & \qquad  - \bm{g}_{k,1}\frac{\|\bm{g}_{k,1}\|^2}{2n} - \bm{g}_{k,1}\frac{\bm{x}_{k,1}^\top \bm{g}_{k,1}}{n} + o\left(\frac{1}{n}\right).
\end{align}

Consistent with the analysis of the second estimate, the expectation of the  entry-wise update of the first estimate satisfies

\begin{align}\label{first_estimate_expectation_with_g}
&\mathbb{E}_k[x_{k+1,1}^{(j)} - x_{k,1}^{(j)} ]  \notag \\ & \qquad = \underbrace{\mathbb{E}_k[g_{k,1}^{(j)}]}_{P_1} - \underbrace{x_{k,1}^{(j)}\frac{\mathbb{E}_k[||\bm{g}_{k,1}||^2]}{2n}}_{P_2}  - \underbrace{x_{k,1}^{(j)}\frac{\mathbb{E}_k[\bm{x}_{k,1}^\top \bm{g}_{k,1}]}{n}}_{P_3} -  \underbrace{\mathbb{E}_k[g_{k,1}^{(j)}\frac{\bm{x}_{k,1}^\top \bm{g}_{k,1}}{n}]}_{P_4} +o\Big(\frac{1}{n}\Big),
\end{align}

where the terms labeled $P_1, P_2, P_3$ and $P_4$ evaluated separately. Leveraging the results established in Section \ref{app:secondestimate} for the expected gradients, it follows that:

\begin{align}
P_1 &= \mathbb{E}_k[g_{k,1}^{(j)}] =  \frac{\tau}{n}\Big(x_{k,1}^{(j)} \langle \gamma_1^{'} \rangle + u_1^{(j)}(\langle c_1\gamma_1 \rangle - Q_{k,1,1}\langle \gamma_1^{'} \rangle) +  u_2^{(j)}(\langle c_2\gamma_1 \rangle - Q_{k,1,2}\langle \gamma_1^{'} \rangle)\Big) ,\\
 P_2 &=  x_{k,1}^{(j)}\frac{\mathbb{E}_k[\|\bm{g}_{k,1}\|^2]}{2n}  = x_{k,1}^{(j)} \frac{\sum_{\alpha = 1}^n \mathbb{E}_k[{(g_{k,1}^{(\alpha)})^2}]}{2n} = \frac{x_{k,1}^{(j)}}{2n}\tau^2 \langle \gamma_1^2\rangle ,\\
P_3 &= \notag x_{k,1}^{(j)}\frac{\mathbb{E}_k[\bm{x}_{k,1}^\top \bm{g}_{k,1}]}{n}, \\ \notag
&= x_{k,1}^{(j)}\frac{\sum_{\alpha = 1}^n x_{k,1}^{(\alpha)}\mathbb{E}_k[g_{k,1}^{(\alpha)}]}{n},   
\\ \notag &=  \frac{x_{k,1}^{(j)}}{n}\sum_{\alpha = 1}^n x_{k,1}^{(\alpha)} \frac{\tau}{n}\Bigg(x_{k,1}^{(\alpha)} \langle \gamma_1^{'} \rangle + u_1^{(\alpha)}(\langle c_1\gamma_1 \rangle - Q_{k,1,1}\langle \gamma_1^{'} \rangle) +  u_2^{(\alpha)}(\langle c_2\gamma_1 \rangle - Q_{k,1,2}\langle \gamma_1^{'} \rangle)\Bigg),
\\ \notag 
&= \frac{x_{k,1}^{(j)}}{n}\tau\Bigg(\frac{1}{n}\sum_{\alpha = 1}^n x_{k,1}^{(\alpha)} x_{k,1}^{(\alpha)} \langle \gamma_1^{'} \rangle \notag \\ & \qquad \qquad \qquad+ \frac{\sum_{\alpha = 1}^n x_{k,1}^{(\alpha)}u_1^{(\alpha)}}{n}(\langle c_1\gamma_1 \rangle - Q_{k,1,1}\langle \gamma_1^{'} \rangle) \notag \\ 
& \qquad \qquad \qquad +  \frac{\sum_{\alpha = 1}^n x_{k,1}^{(\alpha)} u_2^{(\alpha)}}{n}(\langle c_2\gamma_1 \rangle - Q_{k,1,2}\langle \gamma_1^{'} \rangle)\Bigg),
\notag\\ 
&= \frac{x_{k,1}^{(j)}}{n}\tau\Big(\frac{1}{n}n\langle \gamma_1^{'} \rangle + \frac{Q_{k,1,1}n}{n}(\langle c_1\gamma_1 \rangle - Q_{k,1,1}\langle \gamma_1^{'} \rangle) +  \frac{Q_{k,1,2}n}{n}(\langle c_2\gamma_1 \rangle - Q_{k,1,2}\langle \gamma_1^{'} \rangle)\Big),
\notag \\
&= \frac{x_{k,1}^{(j)}}{n}\tau\Big(\langle \gamma_1^{'} \rangle + Q_{k,1,1}(\langle c_1\gamma_1 \rangle - Q_{k,1,1}\langle \gamma_1^{'} \rangle) + Q_{k,1,2}(\langle c_2\gamma_1 \rangle - Q_{k,1,2}\langle \gamma_1^{'} \rangle)\Big) , \\
P_4 &= \mathbb{E}_k[g_{k,1}^{(j)}\frac{\bm{x}_{k,1}^\top \bm{g}_{k,1}}{n}], \notag\\ &= \frac{1}{n} \mathbb{E}_k[\sum_{\alpha=1}^n x_{k,1}^{(\alpha)}g_{k,1}^{(\alpha)}g_{k,1}^{(j)}], \notag \\ &= \frac{1}{n} \sum_{\alpha=1}^n x_{k,1}^{(\alpha)}\mathbb{E}_k[g_{k,1}^{(\alpha)}g_{k,1}^{(j)}], \notag \\ &= \frac{1}{n}x_{k,1}^{(j)}\mathbb{E}_k[(g_{k,1}^{(j)})^2], \notag \\ &= \frac{1}{n^2}x_{k,1}^{(j)} \tau^2 \langle  \gamma_1^2\rangle .
\end{align}

Combining these results yields the final formulation for the  first estimate in the following:

\begin{align}
&\mathbb{E}_k[x_{k+1,1}^{(j)} - x_{k,1}^{(j)}]  \notag \\ & \qquad =  \frac{\tau}{\sqrt{n}}\Big(\frac{1}{\sqrt{n}}x_{k,1}^{(j)} \langle \gamma_1^{'} \rangle + \frac{u_1^{(j)}}{\sqrt{n}}(\langle c_1\gamma_1 \rangle - Q_{k,1,1}\langle \gamma_1^{'} \rangle) +  \frac{u_2^{(j)}}{\sqrt{n}}(\langle c_2\gamma_1 \rangle - Q_{k,1,2}\langle \gamma_1^{'} \rangle)\Big) \notag \\
&\qquad \qquad  - \frac{x_{k,1}^{(j)}}{n}\tau\Big(\langle \gamma_1^{'} \rangle + Q_{k,1,1}(\langle c_1\gamma_1 \rangle - Q_{k,1,1}\langle \gamma_1^{'} \rangle)  + Q_{k,1,2}(\langle c_2\gamma_1 \rangle - Q_{k,1,2}\langle \gamma_1^{'} \rangle)\Big)
  \notag \\ &\qquad \qquad - \frac{x_{k,1}^{(j)}}{2n}\tau^2 \langle \gamma_1^2\rangle + o\left(\frac{1}{n}\right), 
 \notag \\ & \qquad=  \frac{\tau}{\sqrt{n}}\Big(\frac{u_1^{(j)}}{\sqrt{n}}(\langle c_1\gamma_1 \rangle - Q_{k,1,1}\langle \gamma_1^{'} \rangle) +  \frac{u_2^{(j)}}{\sqrt{n}}(\langle c_2\gamma_1 \rangle - Q_{k,1,2}\langle \gamma_1^{'} \rangle)\Big) - \frac{x_{k,1}^{(j)}}{2n}\tau^2 \langle \gamma_1^2\rangle 
\notag \\
&\qquad \qquad - \frac{x_{k,1}^{(j)}}{n}\tau\Big( Q_{k,1,1}(\langle c_1\gamma_1 \rangle - Q_{k,1,1}\langle \gamma_1^{'} \rangle) + Q_{k,1,2}(\langle c_2\gamma_1 \rangle - Q_{k,1,2}\langle \gamma_1^{'} \rangle)\Big)+ o\left(\frac{1}{n}\right).
\end{align}

Proceeding analogously to the derivation in Section \ref{app:secondestimate}, only the leading-order term survives, consistent with the previous analysis we arrive at

\begin{align}
\mathbb{E}_k[(x_{k+1,1}^{(j)} - x_{k,1}^{(j)})^2] &= \frac{\tau^2\langle \gamma_1^2\rangle}{n} + o\left(\frac{1}{n}\right).
\end{align}

\subsection{Derivation of the $\mathbb{E}_k[(x_{k+1,2}^{(j)}-x_{k,2}^{(j)})(x_{k+1,1}^{(j)}-x_{k,1}^{(j)})]$} \label{app:crossterms}

Following the expressions in \eqref{equationx2-x2} and \eqref{equationxupdate}, expectation of the cross term satisfies

\begin{align}
    &\mathbb{E}_k[(x_{k+1,2}^{(j)}-x_{k,2}^{(j)})(x_{k+1,1}^{(j)}-x_{k,1}^{(j)})] \notag \\ & \qquad = \mathbb{E}_k \Bigg[\Big(\tilde{\tilde{\Delta}}_{k,2}^{(j)} - \frac{\|\tilde{\tilde{\bm{\Delta}}}_{k,2}\|^2  x_{k,2}^{(j)}}{2n} - x_{k,2}^{(j)} \frac{\bm{x}_{k,2}^\top\tilde{\tilde{\bm{\Delta}}}_{k,2}}{n}  
     \Big) \notag \\ & \qquad \qquad \qquad   \times
    \Big(g_{k,1}^{(j)}-x_{k,1}^{(j)}\frac{\|\bm{g}_{k,1}\|^2}{2n} - x_{k,1}^{(j)}\frac{\bm{x}_{k,1}^\top \bm{g}_{k,1}}{n} -g_{k,1}\frac{\bm{x}_{k,1}^\top \bm{g}_{k,1}}{n}\Big)\Bigg] +o\left(\frac{1}{n}\right).
\end{align}

Furthermore, substituting Equation \eqref{deltatilttil} into the expression above yields
\begin{align}
    &\mathbb{E}_k[(x_{k+1,2}^{(j)}-x_{k,2}^{(j)})(x_{k+1,1}^{(j)}-x_{k,1}^{(j)})] \notag \\ & \qquad = \mathbb{E}_k \Bigg[\Big( g_{k,2}^{(j)}-x_{k,1}^{(j)}\frac{\bm{x}_{k,1}^\top \bm{g}_{k,2} + \bm{x}_{k,2}^\top \bm{g}_{k,1}}{n} - \frac{x_{k,2}^{(j)}}{n}(\bm{x}_{k,2}^\top\bm{g}_{k,2})\Big) \notag \\ &  \qquad \qquad \qquad \times  
    \Big(g_{k,1}^{(j)}-x_{k,1}^{(j)}\frac{\bm{x}_{k,1}^\top \bm{g}_{k,1}}{n}\Big)\Bigg]+o\left(\frac{1}{n}\right),
\end{align}

\begin{align}
    &\mathbb{E}_k[(x_{k+1,2}^{(j)}-x_{k,2}^{(j)})(x_{k+1,1}^{(j)}-x_{k,1}^{(j)})]  \notag \\ & \qquad =  \mathbb{E}_k[g_{k,1}^{(j)}g_{k,2}^{(j)}] - \frac{x_{k,1}^{(j)}}{n}\mathbb{E}[g_{k,2}^{(j)}(\bm{x}_{k,1}^\top\bm{g}_{k,1})]  - \frac{x_{k,1}^{(j)}}{n}\mathbb{E}[g_{k,1}^{(j)}(\bm{x}_{k,1}^\top \bm{g}_{k,2} + \bm{x}_{k,2}^\top \bm{g}_{k,1})] \notag\\ & \qquad \qquad
    -\frac{x_{k,2}^{(j)}}{n}\mathbb{E}[g_{k,1}^{(j)}(\bm{x}_{k,2}^{\top}\bm{g}_{k,2}) ] 
    +\frac{x_{k,1}^{(j)}x_{k,2}^{(j)}}{n^2}\mathbb{E}[(\bm{x}_{k,1}^\top \bm{g}_{k,1})( \bm{x}_{k,2}^\top \bm{g}_{k,2})] \notag\\ & \qquad \qquad
    + \frac{x_{k,1}^{(j)2}}{n^2}\mathbb{E}[(\bm{x}_{k,1}^\top \bm{g}_{k,2} + \bm{x}_{k,2}^\top \bm{g}_{k,1})(\bm{x}_{k,1}^\top\bm{g}_{k,1})].
\end{align}

We leverage the results we obtained for the expectations over the gradients in the previous sections, thus we can simplify the expression as 

\begin{align}
    &\mathbb{E}_k[(x_{k+1,2}^{(j)}-x_{k,2}^{(j)})(x_{k+1,1}^{(j)}-x_{k,1}^{(j)})]  \notag \\ &\qquad = \mathbb{E}_k[g_{k,1}^{(j)}g_{k,2}^{(j)}] - \frac{x_{k,1}^{(j)2}}{n}\mathbb{E}[g_{k,2}^{(j)}g_{k,1}^{(j)}] - \frac{x_{k,1}^{(j)}}{n}(x_{k,1}^{(j)}\mathbb{E}[g_{k,1}^{(j)}g_{k,2}^{(j)}] + x_{k,2}^{(j)}\mathbb{E}[g_{k,1}^{(j)2}])
    \notag\\ &\qquad \qquad -\frac{x_{k,2}^{(j)2}}{n}\mathbb{E}[g_{k,1}^{(j)}g_{k,2}^{(j)} ] 
    +\frac{x_{k,1}^{(j)}x_{k,2}^{(j)}}{n^2}\sum_{\alpha = 1}^nx_{k,1}^{(\alpha)}x_{k,2}^{(\alpha)}\mathbb{E}[g_{k,1}^{(\alpha)}g_{k,2}^{(\alpha)} ] \notag\\ &\qquad \qquad
    + \frac{x_{k,1}^{(j)2}}{n^2}(\sum_{\alpha =1}^nx_{k,1}^{(\alpha)2}\mathbb{E}[g_{k,1}^{(\alpha)}g_{k,2}^{(\alpha)}] + \sum_{\alpha =1}^nx_{k,1}^{(\alpha)}x_{k,2}^{(\alpha)}\mathbb{E}[g_{k,1}^{(\alpha)2}]).
\end{align}

Recall that $\mathbb{E}[g_{k,1}^{(j)2}]$,$\mathbb{E}[g_{k,2}^{(j)2}]$ and $\mathbb{E}[g_{k,1}^{(j)}g_{k,2}^{(j)} ]$ were found to be independent of $j$, thus these expectations get out of the sum operator. We also recognize that $\sum_{\alpha =1}^nx_{k,1}^{(\alpha)2}=n$ and $\sum_{\alpha = 1}^nx_{k,1}^{(\alpha)}x_{k,2}^{(\alpha)} = 0 $, and arrive at the following:
\begin{align}
    & \mathbb{E}_k[(x_{k+1,2}^{(j)}-x_{k,2}^{(j)})(x_{k+1,1}^{(j)}-x_{k,1}^{(j)})] \notag \\ & \qquad = \frac{\tau^2 \langle\gamma_1 \gamma_2 \rangle}{n} - \frac{x_{k,1}^{(j)2}}{n^2}\tau^2 \langle\gamma_1 \gamma_2 \rangle  - \frac{x_{k,1}^{(j)}}{n^2}(x_{k,1}^{(j)}\tau^2 \langle\gamma_1 \gamma_2 \rangle + x_{k,2}^{(j)}\tau^2 \langle\gamma_1^2 \rangle)\notag\\ & \qquad \qquad
    -\frac{x_{k,2}^{(j)2}}{n^2}\tau^2 \langle\gamma_1 \gamma_2 \rangle 
    + \frac{x_{k,1}^{(j)2}}{n^2}\tau^2 \langle\gamma_1 \gamma_2 \rangle.
\end{align}

Observe that all the other terms other than the first, is of order $O\left(\frac{1}{n^2}\right)$. Consequently, we arrive at the final expression:
\begin{align}
    \mathbb{E}_k\left[(x_{k+1,2}^{(j)}-x_{k,2}^{(j)})(x_{k+1,1}^{(j)}-x_{k,1}^{(j)})\right] = \tau^2 \frac{\langle\gamma_1 \gamma_2 \rangle}{n} + o\left(\frac{1}{n}\right).
\end{align}

\subsection{Regularization Term}\label{app:regularization}
Regarding the regularization term ($-\frac{\tau}{n}\phi(x)$), we first define the macroscopic variable $R_{k,i,j}$ as 
\begin{align}
    R_{k,i,j} = \frac{\bm{x}_{k,i}^\top \phi(\bm{x}_{k,j})}{n}.
\end{align}
We derive the additional regularization terms by redefining the gradient expressions in \eqref{gradients_before_reg}. The resulting updates are given by
\begin{equation}
    \tilde{\bm{x}}_{k,1}
    = 
    \bm{x}_{k,1}
    +
    \underbrace{\frac{\tau}{\sqrt{n}}
    f\!\left(
        \frac{1}{\sqrt{n}} \bm{y}_k^\top \bm{x}_{k,1}
    \right)
    \bm{y}_k  - \frac{\tau}{n}\phi(\bm{x}_{k,1})}_{\coloneqq \bm{g}_{k,1}-\frac{\tau}{n}\phi(\bm{x}_{k,1})},
\end{equation}

\begin{equation}
    \tilde{\bm{x}}_{k,2}
    = 
    \bm{x}_{k,2}
    +
    \underbrace{\frac{\tau}{\sqrt{n}}
    f\!\left(
        \frac{1}{\sqrt{n}} \bm{y}_k^\top \bm{x}_{k,2}
    \right)
    \bm{y}_k  - \frac{\tau}{n}\phi(\bm{x}_{k,2})}_{\coloneqq \bm{g}_{k,2}-\frac{\tau}{n}\phi(\bm{x}_{k,2})}.
\end{equation}

Using our result \eqref{first_estimate_expectation_with_g} from Section \ref{app:firstestimate}, we obtain the following expression for the first estimate:

\begin{align}
\mathbb{E}_k[x_{k+1,1}^{(j)} - x_{k,1}^{(j)} ] &= \underbrace{\mathbb{E}_k[g_{k,1}^{(j)}-\frac{\tau}{n}\phi(\bm{x}_{k,1})^{(j)}]}_{P_1} - \underbrace{x_{k,1}^{(j)}\frac{\mathbb{E}_k[||\bm{g}_{k,1}-\frac{\tau}{n}\phi(\bm{x}_{k,1})||^2]}{2n}}_{P_2} \notag \\ 
& \qquad - \underbrace{x_{k,1}^{(j)}\frac{\mathbb{E}_k[\bm{x}_{k,1}^\top (\bm{g}_{k,1}-\frac{\tau}{n}\phi(\bm{x}_{k,1}))]}{n}}_{P_3} \notag \\ 
& \qquad -   \underbrace{\mathbb{E}_k[(g_{k,1}^{(j)}-\frac{\tau}{n}\phi(\bm{x}_{k,1})^{(j)})\frac{\bm{x}_{k,1}^\top (\bm{g}_{k,1}-\frac{\tau}{n}\phi(\bm{x}_{k,1}))}{n}]}_{P_4} +o\left(\frac{1}{n}\right).
\end{align}

Terms $P_1,P_2,P_3,P_4$ will be calculated separately in the following. For notational clarity and to avoid the repetition of lengthy baseline expressions, we use $(\mathbb{E}_k)_{\text{base}}$ as a generic placeholder for the terms present in the unregularized dynamics. It should be understood that the specific content of $(\mathbb{E}_k)_{\text{base}}$ varies depending on the quantity being evaluated, representing the respective unregularized expectations derived in Section \ref{app:firstestimate} and \ref{app:secondestimate}. This allows us to isolate and highlight the new contributions arising from the regularization term:

\begin{align}
    P_1 &= (\mathbb{E}_k)_{base} - \frac{\tau}{n}\phi(\bm{x}_{k,1})^{(j)}, \\
    P_2 &= (\mathbb{E}_k)_{base} - \frac{x_{k,1}^{(j)}}{2n^3}\tau^2 \| \phi(\bm{x}_{k,1})\|^2 +\frac{x_{k,1}^{(j)}}{n^2}\tau \mathbb{E}[\bm{g}_{k,1}^\top\phi(\bm{x}_{k,1})] ,\notag \\ &= (\mathbb{E}_k)_{base} +\frac{x_{k,1}^{(j)}}{n^2}\tau  \sum_{\alpha=1}^n \frac{\tau}{n}\Bigg(\langle\gamma_1^\prime\rangle x_{k,1}^{(\alpha)}\phi(\bm{x}_{k,1})^{(\alpha)} \notag \\ & \qquad \qquad \qquad \qquad \qquad \qquad \quad+ u_{1}^{(\alpha)}\phi(\bm{x}_{k,1})^{(\alpha)}\Big(\langle c_1\gamma_i \rangle - Q_{k,i,1}\langle \gamma_i^{\prime} \rangle\Big) \notag \\ & \qquad \qquad \qquad \qquad \qquad \qquad \quad+u_{2}^{(\alpha)}\phi(\bm{x}_{k,1})^{(\alpha)}\Big(\langle c_2\gamma_i \rangle - Q_{k,i,2}\langle \gamma_i^{\prime} \rangle\Big)\Bigg)+o\left(\frac{1}{n}\right), \notag \\ &= 
    (\mathbb{E}_k)_{base} + \frac{x_{k,1}^{(j)}}{n^2}\tau^2 \langle \gamma_1^\prime\rangle R_{k,1,1} +o\left(\frac{1}{n}\right),
 \notag \\ &= 
 (\mathbb{E}_k)_{base} +o\left(\frac{1}{n}\right), \\ 
    P_3 &= (\mathbb{E}_k)_{base} + \frac{x_{k,1}^{(j)}}{n^2} \tau \bm{x}_{k,1}^\top\phi(\bm{x}_{k,1}), \notag \\ &= (\mathbb{E}_k)_{base} +\frac{\tau}{n}x_{k,1}^{(j)} R_{k,1,1}, \\ 
    P_4 &= (\mathbb{E}_k)_{base} +o\left(\frac{1}{n}\right).
\end{align}

Finally, we arrive at the following result for the first estimate:

\begin{align}
\mathbb{E}_k[x_{k+1,1}^{(j)} - x_{k,1}^{(j)} ] &= (\mathbb{E}_k)_{base} - \frac{\tau}{n}\phi(\bm{x}_{k,1})^{(j)} + \frac{\tau}{n}x_{k,1}^{(j)}R_{k,1,1} +o\left(\frac{1}{n}\right).
\end{align}

Following Equation \eqref{beforecalculation,drift2}, we compute the regularization induced contributions for each term separately for the second estimate as follows:

\begin{align}
    \mathbb{E}_k[\tilde{\tilde{\Delta}}_{k,2}^{(j)}] 
    &= \underbrace{\mathbb{E}_k[g_{k,2}^{(j)}-\frac{\tau}{n}\phi(\bm{x_{k,2}})^{(j)}]}_{(\mathbb{E}_k)_{base} -\frac{\tau}{n}\phi(\bm{x_{k,2}})^{(j)} }
    - \underbrace{x_{k,1}^{(j)}\frac{\sum_{\alpha = 1}^n x_{k,2}^{(\alpha)}\mathbb{E}_k[g_{k,1}^{(\alpha)}-\frac{\tau}{n}\phi(\bm{x_{k,1}})^{(\alpha)}]}{n}}_{(\mathbb{E}_k)_{base}+ \frac{1}{n}x_{k,1}^{(j)}\tau R_{k,2,1} + o(1/n)} 
    \notag \\ & \qquad - \underbrace{x_{k,1}^{(j)}\frac{\sum_{\alpha = 1}^n x_{k,1}^{(\alpha)}\mathbb{E}_k[g_{k,2}^{(\alpha)}-\frac{\tau}{n}\phi(\bm{x_{k,2}})^{(\alpha)}]}{n}}_{(\mathbb{E}_k)_{base}+ \frac{1}{n}x_{k,1}^{(j)}\tau R_{k,1,2}+ o(1/n)} \notag \\
    &\qquad + \underbrace{2x_{k,1}^{(j)} \frac{\sum_{\alpha = 1}^n x_{k,2}^{(\alpha)}x_{k,1}^{(\alpha)}\mathbb{E}_k[(g_{k,1}^{(\alpha)}-\frac{\tau}{n}\phi(\bm{x_{k,1}})^{(\alpha)})^2]}{n^2}}_{(\mathbb{E}_k)_{base} +  o(1/n)} \notag \\ & \qquad 
    + \underbrace{2x_{k,1}^{(j)} \frac{\sum_{\alpha = 1}^n (x_{k,1}^{(\alpha)})^2\mathbb{E}_k[(g_{k,1}^{(\alpha)}-\frac{\tau}{n}\phi(\bm{x_{k,1}})^{(\alpha)})(g_{k,2}^{(\alpha)}-\frac{\tau}{n}\phi(\bm{x_{k,2}})^{(\alpha)})]}{n^2}}_{(\mathbb{E}_k)_{base} + o(1/n)} \notag \\
    &\qquad
    - \underbrace{\frac{x_{k,2}^{(j)}\mathbb{E}_k[(g_{k,1}^{(j)}-\frac{\tau}{n}\phi(\bm{x_{k,1}})^{(j)})^2]}{n} }_{(\mathbb{E}_k)_{base} + o(1/n)}\notag \\ & \qquad
    - \underbrace{\frac{x_{k,1}^{(j)}\mathbb{E}_k[(g_{k,2}^{(j)}-\frac{\tau}{n}\phi(\bm{x_{k,2}})^{(j)})(g_{k,1}^{(j)}-\frac{\tau}{n}\phi(\bm{x_{k,1}})^{(j)})]}{n}}_{(\mathbb{E}_k)_{base} + o(1/n)}  \notag \\ & \qquad - \underbrace{x_{k,1}^{(j)}\frac{\mathbb{E}_k[(\bm{g}_{k,2}-\frac{\tau}{n}\phi(\bm{x_{k,2}}))^\top (\bm{g}_{k,1}-\frac{\tau}{n}\phi(\bm{x_{k,1}})^{(j)})]}{n}}_{(\mathbb{E}_k)_{base} +o(1/n)},\notag \\ =& (\mathbb{E}_k)_{base} -\frac{\tau}{n}\phi(\bm{x}_{k,2})^{(j)} + \frac{1}{n}x_{k,1}^{(j)}\tau R_{k,2,1} + \frac{1}{n}x_{k,1}^{(j)}\tau R_{k,1,2}   +o\left(\frac{1}{n}\right),
\end{align}

\begin{align}
\mathbb{E}_k[(\tilde{\tilde{\Delta}}_{k,2}^{(j)})^2] &= 
\mathbb{E}_k[(g_{k,2}^{(j)})^2] - 2\frac{x_{k,1}^{(j)}(x_{k,2}^{(j)}\mathbb{E}_k[g_{k,2}^{(j)}g_{k,1}^{(j)}] + x_{k,1}^{(j)}\mathbb{E}_k[(g_{k,2}^{(j)})^2])}{n} \notag \\ 
&\qquad + x_{k,1}^{(j)2} \frac{1}{n^2} \Big(\mathbb{E}_k[(g_{k,1}^{(j)})^2]\underbrace{\sum_\alpha (x_{k,2}^{(\alpha)})^2}_{= \ n} + \mathbb{E}_k[g_{k,2}^{(j)2}] \underbrace{\sum_\alpha (x_{k,1}^{(\alpha)})^2}_{= \ n} \notag \\ & \qquad \qquad\qquad\qquad + 2 \mathbb{E}_k[g_{k,1}^{(j)}g_{k,2}^{(j)}] \underbrace{\sum_\alpha x_{k,1}^{(\alpha)}x_{k,2}^{(\alpha)}}_{ = \ 0} \Big) \ .
\end{align}

As established in the evaluation of $\mathbb{E}_k[\tilde{\tilde{\Delta}}_{k,2}^{(j)}]$, the second-order gradients exclusively contribute to the baseline dynamics $(\mathbb{E}_k)_{\text{base}}$ and higher-order terms. Consequently, the term $\mathbb{E}_k[(\tilde{\tilde{\Delta}}_{k,2}^{(j)})^2]$ does not introduce additional regularization components at order $O(1/n)$, leading to

\begin{align}
\mathbb{E}_k[(\tilde{\tilde{\Delta}}_{k,2}^{(j)})^2] &=  (\mathbb{E}_k)_{base} +o\left(\frac{1}{n}\right).   
\end{align}

Lastly, the final term is given by

\begin{align}
      x_{k,2}^{(j)} \Big(\sum_{\alpha = 1}^n x_{k,2}^{(\alpha)}\mathbb{E}_k[\tilde{\tilde{\Delta}}_{k,2}^{(\alpha)}]\Big)/n  &= \frac{x_{k,2}^{(j)}}{n} \sum_{\alpha =1}^n x_{k,2}^{(\alpha)}\Bigg((\mathbb{E}_k)_{base} -\frac{\tau}{n}\phi(\bm{x}_{k,2})^{(\alpha)} + \frac{1}{n}x_{k,1}^{(\alpha)}\tau R_{k,2,1}  \notag \\ & \qquad \qquad \qquad \qquad \quad+ \frac{1}{n}x_{k,1}^{(\alpha)}\tau R_{k,1,2}    +o\left(\frac{1}{n}\right)\Bigg), \notag \\ &= (\mathbb{E}_k)_{base} - \frac{x_{k,2}^{(j)}}{n} \sum_{\alpha =1}^n x_{k,2}^{(\alpha)}\frac{\tau}{n}\phi(\bm{x}_{k,2})^{(\alpha)} +o\left(\frac{1}{n}\right), \notag \\ &=
      (\mathbb{E}_k)_{base} - \frac{x_{k,2}^{(j)}}{n} \tau R_{k,2,2}+o\left(\frac{1}{n}\right).
\end{align}

Finally, for the second estimate's incremental update we arrive at

\begin{align}
\mathbb{E}_k[x_{k+1,2}^{(j)} - x_{k,2}^{(j)} ] &= (\mathbb{E}_k)_{base} - \frac{\tau}{n}\phi(\bm{x}_{k,2})^{(j)} + \frac{\tau}{n}x_{k,2}^{(j)}R_{k,2,2} + \frac{\tau}{n}x_{k,1}^{(j)}(R_{k,1,2}+R_{k,2,1}) \notag \\ &  \qquad +o\left(\frac{1}{n}\right).
\end{align}

\subsection{Generalization to multi-component case ($p \geq 2$)}\label{app:generalization}
To generalize our results to $p$ component vectors, we express the terms in vector-matrix notation, furthermore we define the following vectors first in $\mathbb{R}^p$:
\begin{subequations}
    \begin{align}
        \bm{c}_k &:= [c_{k,1}, \dots, c_{k,p}]^\top,  \\
        \bm{u}^{(j)}  &:= [u_1^{(j)}, \dots, u_p^{(j)}]^\top, \\
        \bm{x}_k^{(j)} &:= [x_{k,1}^{(j)}, \dots, x_{k,p}^{(j)}]^\top, \\
        \bm{q}_i &:= [Q_{k,i,1}, \dots, Q_{k,i,p}]^\top,  
    \end{align}
    \vspace{-1pt} 
    \begin{align}
        \boldsymbol{\psi}_{i} := 
        \begin{bmatrix}
            \langle c_{k,1} \gamma_i \rangle - Q_{k,i,1}\langle \gamma_i' \rangle \\[-2pt]
            \vdots \\[-2pt]
            \langle c_{k,p} \gamma_i \rangle - Q_{k,i,p}\langle \gamma_i' \rangle
        \end{bmatrix},
        \quad \text{where} \quad
        \gamma_i := f\Big(\bm{c}_k^\top \bm{q}_i + e_{i}\sqrt{1-\bm{q}_i^\top \bm{q}_i}\Big) \ .
    \end{align}
\end{subequations}

Having derived the explicit dynamics for the initial two components, we observe a consistent pattern in the interaction terms. By proceeding inductively, we generalize the formulation to an arbitrary $l-th$ component, obtaining the following expression for the drift term:

\begin{align}
\mathbb{E}_k[x_{k+1,l}^{(j)} - x_{k,l}^{(j)}]
&=
-\frac{\tau^2}{n}
\left(
\frac{1}{2} x_{k,l}^{(j)} \langle \gamma_l^2\rangle
+
\sum_{i=1}^{l-1} x_{k,i}^{(j)} \langle \gamma_i \gamma_l\rangle
\right)
+
\frac{\tau}{n}\,\bm{u}^\top \boldsymbol{\psi}_l
\notag \\ 
&\qquad
-  \frac{\tau}{n} x_{k,l}^{(j)}\mathbf{q}_l^\top \boldsymbol{\psi}_l
-\frac{\tau}{n}
\sum_{i=1}^{l-1} x_{k,i}^{(j)}
\Big(
\mathbf{q}_i^\top \boldsymbol{\psi}_l
+
\mathbf{q}_l^\top \boldsymbol{\psi}_i
\Big) \notag \\ & \qquad + \frac{\tau}{n}x_{k,l}^{(j)}R_{k,l,l} + \frac{\tau}{n} \sum_{i=1}^{l-1}x_{k,i}^{(j)}(R_{k,i,l}+R_{k,l,i}) - \frac{\tau}{n}\phi(\bm{x}_{k,l})^{(j)}\notag \\ & \qquad + o\left(\frac{1}{n}\right).
\label{drfit_appendix}
\end{align}

Furthermore, the diffusion term satisfies the following:

\begin{align}
    \mathbb{E}_k[(x_{l,k+1}^{(j)} - x_{l,k}^{(j)})(x_{m,k+1}^{(j)} - x_{m,k}^{(j)})] = \frac{\tau^2}{n} \langle \gamma_l \gamma_m \rangle +o\left(\frac{1}{n}\right).
\end{align}

For more convenient notation, we express everything in matrix form as in the main text. Define the matrices $\bm{\Psi} \coloneqq [\bm{\psi}_1, ...,\bm{\psi}_p]$, $\bm{Q} \coloneqq [\bm{q}_1, ...,\bm{q}_p]^\top$, $(M)_{i,j} = \bm{q}_i^\top\bm{\psi}_j$ and lastly $(C)_{i,j} = \langle \gamma_i \gamma_j\rangle $. Finally, we arrive at the following expectations for $\bm{\Delta}_k^{(j)} \coloneqq \bm{x}_{k+1}^{(j)}-\bm{x}_{k}^{(j)}$: 

\begin{align}\label{appendixDRIFT_final}
    \mathbb{E}[\bm{\Delta}_k^{(j)}] &= - \frac{\tau^2}{2n}\Big(
    \mathrm{tril}(\bm{C}_k + \bm{C}_k^\top ) -\mathrm{diag}(\bm{C}_k) 
    \Big)\bm{x}_k^{(j)} \notag \\ 
    &\qquad - \frac{\tau}{n} \Big(
    \mathrm{tril}(\bm{M}_k + \bm{M}_k^\top) -\mathrm{diag}(\bm{M}_k)
    \Big)\bm{x}_k^{(j)} \notag \\ 
    &\qquad + \frac{\tau}{n}
\Big(
    \mathrm{tril}(\bm{R}_k + \bm{R}_k^\top) -\mathrm{diag}(\bm{R}_k) 
    \Big)\bm{x}_k^{(j)} - \frac{\tau}{n}\phi(\bm{x}_k^{(j)})    
    \notag \\ 
    &\qquad+ \frac{\tau}{n}\,\bm{\Psi}_k^\top\,\bm{u}^{(j)} + o\left(\frac{1}{n}\right),
\end{align}

\begin{align}\label{appendixDIFF_final}
\mathbb{E}[\bm{\Delta}_k^{(j)}\bm{\Delta}_k^{(j)\top}] = \frac{\tau^2}{n}\bm{C}_k + o\left(\frac{1}{n}\right).
\end{align}

\section{Formal proof of Theorem \ref{theorem}}\label{app:proof-of-theorem}

The formal proof of Theorem \ref{theorem} utilizes the meta-theorem by \cite{wang2017scaling}. This theorem provides sufficient conditions for a sequence of measure-valued processes $\{(\mu_t^n)_{0 \leq k \leq T}\}_n$ to converge to a limiting partial differential equation (PDE) . Below, we restate these assumptions \ref{assumption1}-\ref{asssumption10}, generalizing the original scalar formulations to our multi-component vector and matrix notation. We subsequently verify that each condition is satisfied within our setting. 

\subsection{Assumption 1}\label{assumption1}

\textit{The Markov Chain $\{(\bm{X}_k, \bm{U}_k)\}_{k \geq 0}$ is exchangeable.}

In our multi-component online ICA setting, the state $(\bm{X}_k, \bm{U}_k)$ forms a Markov chain, where the iterates are $\bm{X}_k = [\bm{x}_{k,1}, \dots, \bm{x}_{k,p}]^\top \in \mathbb{R}^{p \times n}$ and the fixed true components are $\bm{U}_k = [\bm{u}_1, \dots, \bm{u}_p] \in \mathbb{R}^{n \times p}$. We formally include the trivial update $\bm{U}_{k+1} = \bm{U}_k$ for completeness, even though the true components are fixed. We show that the complete update step; comprising the online update, orthogonalization, and normalization preserves joint exchangeability over the $n$ spatial coordinates (the columns of $\bm{X}_k$ and the rows of $\bm{U}_k$).

The complete algorithm at each iteration $k$ operates as a sequence of distinct, deterministic transformations: the online update, Gram-Schmidt orthogonalization, and normalization. Because each procedure takes the output of the preceding step as its direct input, the overall update can be formally expressed as a composite mapping.

Without loss of generality, we can assume that the initial state $(\bm{X}_0, \bm{U}_0)$ forms an exchangeable Markov chain (see Remark 3 of \cite{wang2017scaling}). To establish exchangeability, we must demonstrate that the online update rule is equivariant to permutations of the $n$ spatial coordinates. Let $\bm{\pi} \in \mathbb{R}^{n \times n}$ be an arbitrary permutation matrix, $\bm{\pi}\bm{\pi}^\top = \bm{I}$. Thus, permuted matrices satisfy

\begin{align}
    \bm{U}^{\pi} &= [\bm{\pi}\bm{u}_1, \dots , \bm{\pi}\bm{u}_p] = \bm{\pi} \bm{U}, \qquad \
    \bm{X}_k^{\pi}=[\bm{\pi}\bm{x}_{k,1}, \dots , \bm{\pi}\bm{x}_{k,p}]^\top =\bm{X}_{k}\bm{\pi}^\top. \end{align}

We first demonstrate that the online update preserves exchangeability by showing equivariance to spatial permutations:

\begin{align}
\tilde{\bm{X}}_{k} (\bm{\pi}^\top)^\top  &=\tilde{\bm{X}}_{k} \bm{\pi}, \notag \\  &= \Bigg( \bm{X}_k\bm{\pi^\top}
    +
    \frac{\tau}{\sqrt{n}}\, f\Big(\frac{1}{\sqrt{n}}\bm{X}_k\bm{\pi}^\top(\frac{1}{\sqrt{n}}\bm{\pi U}\bm{c}_k+\bm{\pi}\bm{a_k})\Big)\, (\frac{1}{\sqrt{n}}\bm{\pi U}\bm{c}_k+\bm{\pi}\bm{a_k})^{\top} \notag \\ & \qquad \qquad -\frac{\tau}{n}\phi(\bm{X}_k\bm{\pi}^\top)  \Bigg)\bm{\pi}, \notag \\
    &=\bm{X}_k\bm{\pi^\top}\bm{\pi}
    +
    \frac{\tau}{\sqrt{n}}\ f\Big(\frac{1}{n}\bm{X}_k\bm{\pi}^\top\bm{\pi U}\bm{c}_k+\frac{1}{\sqrt{n}}\bm{X}_k\bm{\pi}^\top \bm{\pi}\bm{a_k}\Big)\, (\frac{1}{\sqrt{n}}\bm{c_k}^\top\bm{U}^\top \bm{\pi}^\top+\bm{a_k}^{\top}\bm{\pi}^\top) \bm{\pi}  \notag \\ & \qquad \qquad -\frac{\tau}{n}\phi(\bm{X}_k\bm{\pi}^\top)\bm{\pi}, \notag \\
    &= \bm{X}_k \bm{I}+ \frac{\tau}{\sqrt{n}}\ f\Big(\frac{1}{n}\bm{X}_k\bm{IU}\bm{c}_k+\frac{1}{\sqrt{n}}\bm{X}_k\bm{I}\bm{a_k}\Big)\, (\frac{1}{\sqrt{n}}\bm{c_k}^\top\bm{U}^\top +\bm{a_k}^{\top}\bm{\pi}^\top\bm{\pi})\bm{\pi}^\top \bm{\pi} \notag \\ & \qquad \qquad  -\frac{\tau}{n}\phi(\bm{X}_k) \bm{\pi}^\top\bm{\pi}, \notag \\
    &=\bm{X}_k+ \frac{\tau}{\sqrt{n}}\ f\Big(\frac{1}{n}\bm{X}_k\bm{U}\bm{c}_k+\frac{1}{\sqrt{n}}\bm{X}_k\bm{a_k}\Big)\, (\frac{1}{\sqrt{n}}\bm{U} \bm{c_k}+\bm{a_k})^\top -\frac{\tau}{n}\phi(\bm{X}_k) \bm{I}, \notag \\
    &=\bm{X}_k+ \frac{\tau}{\sqrt{n}}\ f\Big(\frac{1}{\sqrt{n}}\bm{X}_k\bm{y_k}\Big)\,\bm{y_k}^\top -\frac{\tau}{n}\phi(\bm{X}_k) .
\end{align}

Above, we exploit the fact that the regularization $\phi(\cdot)$ acts entrywise on the matrix. Consequently, it is equivariant with respect to spatial permutations applied to the columns: $\phi(\bm{X}_k \bm{\pi}^\top) = \phi(\bm{X}_k) \bm{\pi}^\top$. Furthermore, for the Gram-Schmidt process we have

\begin{align}
    \tilde{\tilde{\bm{X}}}_{k} &= [\tilde{\tilde{\bm{x}}}_{k,1},\dots, \tilde{\tilde{\bm{x}}}_{k,p}]^\top ,\notag \\
    &= [\tilde{\bm{x}}_{k,1},\dots , \tilde{\tilde{\bm{x}}}_{k,i} = \tilde{\bm{x}}_{k,i} - \sum_{j=1}^{i-1} \frac{\tilde{\bm{x}}_{k,i}^{\top}\tilde{\bm{x}}_{k,j}}{\|\tilde{\bm{x}}_{k,j}\|^2}\,\tilde{\bm{x}}_{k,j}, \dots,\tilde{\bm{x}}_{k,p} - \sum_{j=1}^{p-1} \frac{\tilde{\bm{x}}_{k,p}^{\top}\tilde{\bm{x}}_{k,j}}{\|\tilde{\bm{x}}_{k,j}\|^2}\,\tilde{\bm{x}}_{k,j} ],
\end{align}
To demonstrate exchangeability, let us consider the $i-th$ row vector, yielding
\begin{align}
\Longrightarrow\bm{\pi}^\top\tilde{\tilde{\bm{x}}}_{k,i} &=\bm{\pi}^\top \Bigg(\bm{\pi}\tilde{\bm{x}}_{k,i} - \sum_{j=1}^{i-1} \frac{(\bm{\pi}\tilde{\bm{x}}_{k,i})^{\top}(\bm{\pi}\tilde{\bm{x}}_{k,j})}{\|\bm{\pi}\tilde{\bm{x}}_{k,j}\|^2}\,(\bm{\pi}\tilde{\bm{x}}_{k,j}) \Bigg) \notag \\
    &= \bm{\pi}^\top \bm{\pi}\tilde{\bm{x}}_{k,i} - \sum_{j=1}^{i-1} \bm{\pi}^\top\frac{(\bm{\pi}\tilde{\bm{x}}_{k,i})^{\top}(\bm{\pi}\tilde{\bm{x}}_{k,j})}{\|\bm{\pi}\tilde{\bm{x}}_{k,j}\|^2}\,(\bm{\pi}\tilde{\bm{x}}_{k,j}) \notag \\
    &= \bm{I}\tilde{\bm{x}}_{k,i} - \sum_{j=1}^{i-1} \frac{(\tilde{\bm{x}}_{k,i}^\top\bm{\pi}^\top\bm{\pi}\tilde{\bm{x}}_{k,j})}{\tilde{\bm{x}}_{k,j}^\top\bm{\pi}^\top \bm{\pi}\tilde{\bm{x}}_{k,j}} \bm{\pi}^\top\bm{\pi}\tilde{\bm{x}}_{k,j}  \notag \\
    &= \tilde{\bm{x}}_{k,i} - \sum_{j=1}^{i-1} \frac{(\tilde{\bm{x}}_{k,i}^\top\tilde{\bm{x}}_{k,j})}{\tilde{\bm{x}}_{k,j}^\top\tilde{\bm{x}}_{k,j}} \tilde{\bm{x}}_{k,j}.
\end{align}

Thus, the Gram-Schmidt process used to obtain $\tilde{\tilde{\bm{X}}}_{k}$ from $\tilde{\bm{X}}_{k}$ also preserves exchangeability. Finally, to show that the normalization step is exchangeable we write 
\begin{align}
    \bm{\pi}^\top\bm{x}_{k+1,i} &= \bm{\pi}^\top \Big(\frac{\bm{\pi}\tilde{\tilde{\bm{x}}}_{k,i}}{\|\bm{\pi}\tilde{\tilde{\bm{x}}}_{k,i} \|} \sqrt{n} \Big), \notag \\
    &=\frac{\bm{\pi}^\top \bm{\pi}\tilde{\tilde{\bm{x}}}_{k,i}}{\sqrt{\tilde{\tilde{\bm{x}}}_{k,i}^\top\bm{\pi}^\top\bm{\pi}\tilde{\tilde{\bm{x}}}_{k,i} }} \sqrt{n} , \notag \\
    &=\frac{\tilde{\tilde{\bm{x}}}_{k,i}}{\|\tilde{\tilde{\bm{x}}}_{k,i} \|} \sqrt{n}.
\end{align}
Because the full update is a composition of these sequentially applied steps, demonstrating that the online update, orthogonalization and normalization are individually exchangeable is sufficient to conclude that the overall process forms an exchangeable Markov chain.

\subsection{Assumption 2}\label{assumption2}

\textit{The initial empirical measure $\mu_0^n(\bm{x}, \bm{u})$ converges weakly to a deterministic measure $\mu_0$.}

Our Theorem \ref{theorem} is explicitly conditioned on this assumption.

\subsection{Assumption 3}\label{assumption3}

\textit{There is some finite constant $C$ such that}

\begin{align*}
    \sup_n \Big\langle \mu_0^n, \sum_{i=1}^p x_i^4 + \sum_{j=1}^p u_j^4 \Big\rangle \leq C.
\end{align*}

Our theorem assumes that all moments of the \emph{initial} ground truth and estimate vectors are bounded; consequently, Assumption 3 is satisfied directly.

\subsection{Assumption 4} \label{assumption4}

\textit{Let $\bm{\Delta}_k^{(j)} = \bm{x}_{k+1}^{(j)} - \bm{x}_k^{(j)}$. There exists a deterministic function $\bm{\mathcal{G}}: \mathbb{R}^p \times \mathbb{R}^p \times \mathbb{R}^{r_1 \times r_2} \xrightarrow{} \mathbb{R}^p$  for some $r_1,r_2 \geq 0$, such that for each  $T > 0$}:

\begin{align}
    \max_{k \le nT} \mathbb{E}\Bigg[\Big\|\mathbb{E}_k[\bm{\Delta}_k^{(j)}] - \frac{1}{n} \bm{\mathcal{G}}_k^{(j)}\Big\| \Bigg] \le \frac{C(T)}{n^{1+\gamma}}\ ,
\end{align}

\textit{where $\gamma >0$ is some positive constant. In the above expression,}

\begin{align}
    \bm{\mathcal{G}}_k^{(j)} =  \bm{\mathcal{G}}(\bm{x}_k^{(j)},\bm{u}^{(j)}, \bm{\Theta}_k^n),
\end{align}

\textit{and $\bm{\Theta}_k^n $ is a matrix in $ \mathbb{R}^{r_1 \times r_2}$. The $l,m$-th entry of $\bm{\Theta}_k^n $ is defined as}

\begin{align}
    (\bm{\Theta}_k^n)_{l,m} = \langle \mu_k^n, \varphi_{l,m}(\bm{x}, \bm{u})\rangle ,
\end{align}

\textit{where $\varphi_{l,m}(\bm{x}, \bm{u})$ is some deterministic function.}

In our setting, we use the overlap matrix $\bm{Q}_t$ and matrix $\bm{R}_t$ as our macroscopic variables. We define $\varphi(\bm{x}, \bm{u}) = x^{(l)} u^{(m)}$ to obtain  $(Q_k^n)_{l,m} = \langle \mu_k^n, x^{(l)} u^{(m)} \rangle$ and $\varphi(\bm{x}, \bm{u}) = x^{(l)} \phi(x)^{(m)}$ to obtain $(R_k^n)_{l,m} = \langle \mu_k^n, x^{(l)} \phi(x)^{(m)}\rangle$. Thus, we define $\bm{\Theta}$ as follows: 
\begin{align}
    \bm{\Theta}_k^n \coloneqq \begin{bmatrix}
        \bm{Q}_k^n | \bm{R}_k^n
    \end{bmatrix}.
\end{align}
We refer to the derivation of the expected incremental update (drift) in Appendix \ref{app:driftanddiffusion}, restating Equation \eqref{appendixDRIFT_final} for convenience:
\begin{align}
\mathbb{E}_k[\bm{\Delta}_k^{(j)}] &= - \frac{\tau^2}{2n}\Big(
    \mathrm{tril}(\bm{C}_k + \bm{C}_k^\top ) -\mathrm{diag}(\bm{C}_k) 
    \Big)\bm{x}_k^{(j)} \notag \\ 
    &\qquad - \frac{\tau}{n} \Big(
    \mathrm{tril}(\bm{M}_k + \bm{M}_k^\top) -\mathrm{diag}(\bm{M}_k)
    \Big)\bm{x}_k^{(j)} \notag \\ 
    &\qquad + \frac{\tau}{n}
\Big(
    \mathrm{tril}(\bm{R}_k + \bm{R}_k^\top) -\mathrm{diag}(\bm{R}_k) 
    \Big)\bm{x}_k^{(j)} - \frac{\tau}{n}\phi(\bm{x}_k^{(j)})    
    \notag \\ 
    &\qquad+ \frac{\tau}{n}\,\bm{\Psi}_k^\top\,\bm{u}^{(j)} + o\left(\frac{1}{n}\right) .
\end{align}

We define the coefficient $\bm{\mathcal{G}}$ as
\begin{align}\label{equationG}
    \bm{\mathcal{G}}^{(j)} &\coloneqq - \frac{\tau^2}{2}\Big(
    \mathrm{tril}(\bm{C}_k + \bm{C}_k^\top ) -\mathrm{diag}(\bm{C}_k) 
    \Big)\bm{x}_k^{(j)} \notag \\ 
    &\qquad - \tau \Big(
    \mathrm{tril}(\bm{M}_k + \bm{M}_k^\top) -\mathrm{diag}(\bm{M}_k)
    \Big)\bm{x}_k^{(j)} \notag \\ 
    &\qquad +\tau
\Big(
    \mathrm{tril}(\bm{R}_k + \bm{R}_k^\top) -\mathrm{diag}(\bm{R}_k) 
    \Big)\bm{x}_k^{(j)} - \tau\phi(\bm{x}_k^{(j)})    
    \notag \\ 
    &\qquad+ \tau\,\bm{\Psi}_k^\top\,\bm{u}^{(j)}.
\end{align}

Consequently, Assumption 4 is satisfied.

\subsection{Assumption 5} \label{assumption5}

\textit{There exists some deterministic function $\bm{\Lambda} : \mathbb{R}^{r_1 \times r_2} \xrightarrow{} \mathbb{R}^{p\times p}$ such that, for each $T > 0$,}

\begin{align}
\max_{k \leq n T} \mathbb{E}\left[\left\|\mathbb{E}_k\left[\bm{\Delta}_k^{(j)} \bm{\Delta}_k^{(j)\top}\right]-\frac{1}{n}  \bm{\Lambda}_k \right\|_F\right] \leq \frac{C(T)}{n^{1+\gamma}},
\end{align}

\textit{where $\gamma >0$ is some positive constant, and}

\begin{align}
    \bm{\Lambda}_k = \bm{\Lambda}(\bm{\Theta}_k^n).
\end{align}

We refer to the derivation of the second moment of the incremental update (diffusion) in Appendix~\ref{app:driftanddiffusion}, restating Equation \eqref{appendixDIFF_final} for convenience:
\begin{align}
\mathbb{E}_k\left[\bm{\Delta}_k^{(j)} \bm{\Delta}_k^{(j)\top}\right] = \frac{1}{n} \tau^2  (\mathbf{C}_k) + o\Big(\frac{1}{n}\Big).
\end{align}
Then we define the coefficient as follows:
\begin{align}
\bm{\Lambda}(\bm{Q}_k^n) \coloneqq  \tau^2 \mathbf{C}_k(\bm{Q}_k^n).
\end{align}

The diffusion term is independent of the matrix $\bm{R}$; therefore, we define it solely in terms of the $\bm{Q}$ matrix. Consequently, Assumption 5 is satisfied.

\subsection{Assumption 6}\label{asumption6}

\textit{For any $T > 0$, there exists a finite constant $B(T)$ such that}

\begin{align}
    \lim_{n \to \infty} \mathbb{P} \left( \max_{k \leq nT} \| \bm{\Theta}_k^n \|_{\text{max}} > B(T) \right) = 0,
\end{align}

\textit{where $\| \cdot\|_{\text{max}}$ is the max-norm of the matrix.}

$\bm{Q}_k^n$'s every entry $(Q_k^n)_{j,l}$ is bounded by definition:  $(Q_k^n)_{j,l} \in [-1,1]$; $\bm{R}_k^n$'s every entry is bounded by the assumption on $\phi(x)$ being Lipschitz: $|(R_{k}^n)_{i,j}| \leq \frac{1}{\sqrt{n}}(L \|\bm{x}_{k,j}\| + \sqrt{n}|\phi(0)|) = L + |\phi(0)|$. Consequently, Assumption 6 holds.

\subsection{Assumption 7}\label{assumption7}

\textit{Define $(\Theta_k^n)_{\ell,m}(h) = \langle \mu_k^n, \varphi_{\ell,m}(\bm{x}, \bm{u}) \sqcap h \rangle$, where $\sqcap$ is the projection operation defined as the projection of $x$ onto the interval $[-b,b]$ : $x \sqcap b = \min \{|x|, b\} \operatorname{sgn}(x)$. For any $b > B(T)$ and $T > 0$, we have}

\begin{equation}
    \limsup_{h \to \infty} \sup_n \max_{k \leq nT} \mathbf{E} \left\| \bm{\mathcal{G}}(\bm{x}_k^{(j)}, \bm{u}^{(j)}, \bm{\Theta}_k^n) - \bm{\mathcal{G}}(\bm{x}_k^{(j)}, \bm{u}^{(j)}, \bm{\Theta}_k^n(h) \sqcap b) \right\| = 0,
\end{equation}
\textit{and}
\begin{equation}
    \limsup_{h \to \infty} \sup_n \max_{k \leq nT} \mathbf{E} \left\| \bm{\Lambda}(\bm{Q}_k^n) - \bm{\Lambda}(\bm{Q}_k^n(h) \sqcap b) \right\|_F = 0.
\end{equation}

Assumption 7 holds by the same reasoning applied to Assumption \ref{asumption6}.

\subsection{Assumption 8} \label{assumption8}
\textit{For each $T>0$, there exists $C(T)< \infty$ such that}

\begin{align}
    \max_{k \leq nT} \mathbb{E}[ \|\bm{\mathcal{G}}_k^{(j)}\|^2] \leq C(T) \quad \text{and} \quad \max_{k \leq nT} \mathbb{E} [\|\bm{\Lambda}_k\|_F^2] \leq C(T).
\end{align}

We first establish several supplementary lemmas needed to prove these bounds.
\begin{lemma}\label{lemmaboundonupsilon}
For $q > 0$, the vector $\bm{\upsilon}_t$ satisfies the following bound on its $q$-th absolute moment, where $H_q$ and $L_q$ are positive constants depending on $q$:

 \begin{align}
\mathbb{E}\left[|\upsilon_i|^q\right] \leq H_q(1+L_q) .
 \end{align}
\end{lemma}

\begin{proof}
We begin by recalling the growth assumption on the nonlinearity: $|f(x)| \le B(1+|x|^q)$ for constants $B \geq 0$ and $q \geq 0$. We first establish that $\upsilon_i$ is zero-mean. By the the zero-mean properties of the components $c_j$ and $e_i$, we have

\begin{align}
    \mathbb{E}_{\bm{c},\bm{e}}[\upsilon_i] =  \mathbb{E}_{\bm{c},\bm{e}}[c_1Q_{i,1}+...+c_pQ_{i,p} + e_i\sqrt{1-Q_{i,1}^2...-Q_{i,p}^2}] = 0.
\end{align}

Next, we compute the second moment. Conditioning on $c_i$ and applying the tower property yields

\begin{align}
    \mathbb{E}_{\bm{c},\bm{e}}
[\upsilon_i^2] &=  \mathbb{E}_{\bm{c},\bm{e}}[(c_1Q_{i,1}+...+c_pQ_{i,p} + e_i\sqrt{1-Q_{i,1}^2...-Q_{i,p}^2})^2] ,\notag \\ 
    &= \mathbb{E}_{\bm{c}}[\mathbb{E}_{\bm{e}}[(c_1Q_{i,1}+...+c_pQ_{i,p} + e_i\sqrt{1-Q_{i,1}^2...-Q_{i,p}^2})^2 | c_1,c_2 ...,c_p]], \notag\\
    &= \mathbb{E}_c[(c_1Q_{i,1} + ...+c_pQ_{i,p})^2 + 1 - Q_{i,1}^2...-Q_{i,p}^2], \notag\\
    &= 1.
\end{align}

To bound higher-order moments, we appeal to a Rosenthal-type inequality \cite{Rosenthal_1970} in the form given by \cite{chen2020rosenthal}, which states the following: for a sequence of independent random variables $X_1,\ldots,X_n$ with zero mean satisfying $\mathbb{E}[|X_j|^q] < \infty$ for some $q > 2$, there exists a positive constant $H_q$ depending only on $q$ such that

\begin{align} \label{inequaltyiused}
    \mathbb{E}\left[\left|\sum_{j=1}^{n}X_j\right|^q\right] \leq H_q \text{max}\left\{\sum_{j=1}^n\mathbb{E}[|X_j|^q], \Bigg( \sum_{j=1}^n\mathbb{E}[X^2]\Bigg)^{q/2}\right\}.
\end{align}

We apply this inequality by decomposing $\upsilon_i$ into a sum of $n = p+1$ independent, zero-mean terms: $X_1 = c_1 Q_{i,1},\;\ldots,\;X_p = c_p Q_{i,p}$, and $X_{p+1} = e_i\sqrt{1 - Q_{i,1}^2 - \cdots}$. Each summand satisfies $\mathbb{E}_{\bm{c},\bm{e}}[X_j] = 0$, and finiteness of all $q$-th moments follows from the assumption that the moments of $c_i$ are bounded and $e_i$ is Gaussian. We observe that

\begin{align}
    \sum_{j=1}^{p+1}X_j = \upsilon_i, \qquad 
    \Bigg(\sum_{j=1}^{p+1}\mathbb{E}_{\bm{c},\bm{e}}[X^2]\Bigg)^{q/2}= 1.
\end{align}

Thus, we arrive at the following:
\begin{align}
\mathbb{E}_{\bm{c},\bm{e}}\left[\left|\upsilon\right|_i^q\right] \leq H_q \text{max}\left\{\sum_{j=1}^{p+1}\mathbb{E}_{\bm{c},\bm{e}}[|X_j|^q], 1\right\}.
\end{align}

It remains to control the sum of individual $q$-th moments. Expanding, we obtain

\begin{align}
\sum_{j=1}^{p+1}\mathbb{E}_{\bm{c},\bm{e}}[|X_j|^q] &= \mathbb{E}_{\bm{c},\bm{e}}[|X_1|^q]+...+\mathbb{E}_{\bm{c},\bm{e}}[|X_{p+1}|^q], \notag \\
&= \mathbb{E}_{\bm{c},\bm{e}}[|c_1Q_{i,1}|^q] + ... +\mathbb{E}_{\bm{c},\bm{e}}[\Big|e_i \sqrt{1-Q_{i,1}^2\dots - Q_{i,p}^2}\Big|^q], \notag \\
&\leq\mathbb{E}_{\bm{c},\bm{e}}[|c_1|^q|Q_{i,q}|^q]+...+\mathbb{E}_{\bm{c},\bm{e}}[|e_i|^q\Big|\sqrt{1-Q_{i,1}^2\dots - Q_{i,p}^2}\Big|^q]  ,\notag \\
&=|Q_{i,1}|^q \mathbb{E}_{\bm{c},\bm{e}}[|c_1|^q] + ... +\Big|\sqrt{1-Q_{i,1}^2\dots - Q_{i,p}^2}\Big|^q\mathbb{E}_{\bm{c},\bm{e}}[|e_i|^q] .
\end{align}

Since the overlaps satisfy $|Q_{i,j}| \leq 1$ and $\Big|\sqrt{1 - Q_{i,1}^2  \dots -Q_{i,p}^2}\Big| = \Big|\sqrt{1 - \sum_{\alpha=1}^pQ_{i,\alpha}^2}\Big| \leq 1$, each factor involving $Q$ can be bounded by unity, giving

\begin{align}
\sum_{j=1}^{p+1}\mathbb{E}_{\bm{c},\bm{e}}[|X_j|^q]
&=|Q_{i,1}|^q \underbrace{\mathbb{E}_{\bm{c},\bm{e}}[|c_1|^q]}_{ = Z_{1,q}} + ... +\left|\sqrt{1 - \sum_{\alpha=1}^pQ_{i,\alpha}^2}\right|^q\underbrace{\mathbb{E}_{\bm{c},\bm{e}}[|e|^q]}_{= Z_{e,q}}, \notag \\ & \leq Z_{1,q}+Z_{2,q}+...+Z_{e,q}.
\end{align}

Substituting this back into the Rosenthal bound yields

\begin{align}
\mathbb{E}_{\bm{c},\bm{e}}\left[\left|\upsilon\right|_i^q\right] \leq  H_q \text{max}\left\{Z_{1,q}+Z_{2,q}+\cdots+Z_{e,q}, 1\right\} .
\end{align}

To obtain a cleaner expression, we define $L_q \coloneqq Z_{1,q}+Z_{2,q}+\dots+Z_{e,q}$ and we relax the maximum using the inequality $\max\{a, 1\} \leq 1 + a$ for $a \geq 0$:

\begin{align}
H_q \text{max}\left\{L_q, 1\right\} \leq H_q(1+L_q),
\end{align}

which establishes the claimed bound and finalizes the proof:

\begin{align}
\mathbb{E}_{\bm{c},\bm{e}} \left[\left|\upsilon_i\right|^q\right]  = \mathbb{E} \left[\left|\upsilon_i\right|^q\right] \leq  H_q(1+L_q).
\end{align}

\end{proof}

\begin{lemma}\label{lemmaCmatrix}
For $q > 0$, the entries of the matrix $\bm{C}_t$ satisfy the following elementwise bound, where $B$ is the growth constant of the nonlinearity $f$ and $R_q$ and $R_{2q}$ are positive constants depending on $q$:

    \begin{align}
        |C_{i,j}| \leq B^2(1+2R_q+R_{2q}).
    \end{align}
\end{lemma}
\begin{proof}
We bound the absolute value of each entry by applying Jensen's inequality followed by the growth condition $|f(x)| \leq B(1+|x|^q)$:
\begin{align}
|C_{i,j}| &=|\mathbb{E}_{c,e}[f(\upsilon_i)f(\upsilon_j)]|, \notag\\ &\leq \mathbb{E}_{c,e}[|f(\upsilon_i)f(\upsilon_j)|] ,\notag\\
&\leq\mathbb{E}_{c,e}[|f(\upsilon_i)||f(\upsilon_j)|], \notag\\
&\leq \mathbb{E}_{c,e}[B^2(1+|\upsilon_i|^q)(1+|\upsilon_j|^q)], \notag\\
&=B^2(1+\mathbb{E}_{c,e}|v_i|^q+\mathbb{E}_{c,e}|v_j|^q + \mathbb{E}[|v_i|^q|v_j|^q]) .
\end{align}

Expanding the product and taking expectations termwise, we invoke Lemma~\ref{lemmaboundonupsilon} to substitute $\mathbb{E}_{c,e}[|\upsilon_i|^q] \leq R_q$, where we define $R_q \coloneqq H_q(1+L_q)$. This yields

\begin{align}
|C_{i,j}|
\leq B^2(1+2R_q+\mathbb{E}_{\bm{c},\bm{e}}[|v_i|^q|v_j|^q]) .
\end{align}

To handle the remaining cross-moment, we apply the Cauchy-Schwarz inequality:

\begin{align}
\mathbb{E}_{\bm{c},\bm{e}}\big[|\upsilon_i|^q|\upsilon_j|^q\big]
&\le
\Big(\mathbb{E}_{\bm{c},\bm{e}}|\upsilon_i|^{2q}\Big)^{1/2}
\Big(\mathbb{E}_{\bm{c},\bm{e}}|\upsilon_j|^{2q}\Big)^{1/2}.
\end{align}

Since Lemma~\ref{lemmaboundonupsilon} applied at order $2q$ gives $\mathbb{E}_{\bm{c},\bm{e}}[|\upsilon_i|^{2q}] \leq R_{2q}$ for all $i$, both factors are bounded by $R_{2q}^{1/2}$. Consequently, we arrive at
\begin{align}
\mathbb{E}_{\bm{c},\bm{e}}\big[|\upsilon_i|^q|\upsilon_j|^q\big]
\le R_{2q}.
\end{align}
Substituting back completes the proof:
\begin{align}
|C_{i,j}|
\leq B^2(1+2R_q+R_{2q}) .
\end{align}
\end{proof}

\begin{lemma}\label{lemmaPsibound}
For $q > 0$, the entries of the matrix $\bm{\Psi}_t$ satisfy the following elementwise bound:

\begin{align}
 |\Psi_{i,j}| \leq B\sqrt{2(1+R_{2q})} +K(1+R_q^{\frac{q-1}{q}}),
\end{align}

where $B$ and $K$ are the growth constants of the nonlinearity and its derivative, $R_{q}$ and $R_{2q}$ are positive constants depending on $q$.
\end{lemma}
\begin{proof}
In addition to the growth condition on $f$,we recall the assumed bound on its derivative: $|f'(x)| \leq K(1 + |x|^{q-1})$ for some $q \geq 0$ and $K \geq 0$. We begin by applying the triangle inequality:
    
\begin{align}
|\Psi_{i,j}| &= \left| \mathbb{E}_{\bm{c},\bm{e}}[c_i f(\upsilon_j)] - Q_{j,i}\mathbb{E}_{\bm{c},\bm{e}}[f'(\upsilon_j)] \right|, \notag \\
&\leq \left| \mathbb{E}_{\bm{c},\bm{e}}[c_i f(\upsilon_j)] \right| + \left| Q_{j,i} \mathbb{E}_{\bm{c},\bm{e}}[f'(\upsilon_j)] \right|.
\end{align}
For the first term, we apply the Cauchy--Schwarz together with $\mathbb{E}_{\bm{c},\bm{e}}[c_i^2] = 1$. For the second term, we use $|Q_{j,i}| \leq 1$ and Jensen's inequality:

\begin{align}
|\Psi_{i,j}|  &\leq \sqrt{\mathbb{E}_{\bm{c},\bm{e}}[c_i^2]\,\mathbb{E}_{\bm{c},\bm{e}}[f(\upsilon_j)^2]} + |Q_{j,i}|\,\Big|\mathbb{E}_{\bm{c},\bm{e}}[f'(\upsilon_j)]\Big|, \notag\\
&\leq \sqrt{\mathbb{E}_{\bm{c},\bm{e}}[f(\upsilon_j)^2]} + \mathbb{E}_{\bm{c},\bm{e}}\!\Big[|f'(\upsilon_j)|\Big].
\end{align}
We now substitute the growth bounds $|f(x)| \leq B(1+|x|^q)$ and $|f'(x)| \leq K(1+|x|^{q-1})$, and use $(1+a)^2 \leq 2(1+a^2)$ for $a \geq 0$ to simplify

\begin{align}
|\Psi_{i,j}|  &\leq \sqrt{\mathbb{E}_{\bm{c},\bm{e}}[B^2(1+|\upsilon_j|^q)^2]} + \mathbb{E}_{\bm{c},\bm{e}}\!\Big[\Big| K(1+|\upsilon_j|^{q-1})\Big| \Big], \notag\\
&= \sqrt{\mathbb{E}_{\bm{c},\bm{e}}[B^2(1+|\upsilon_j|^q)^2]} + \mathbb{E}_{\bm{c},\bm{e}}\!\Big[ K(1+|\upsilon_j|^{q-1}) \Big],
\notag\\
&\leq \sqrt{2B^2\big(1+\mathbb{E}_{\bm{c},\bm{e}}[|\upsilon_j|^{2q}]\big)} + K\big(1+\mathbb{E}_{\bm{c},\bm{e}}[|\upsilon_j|^{q-1}]\big).
\end{align}
For the second term, we apply Jensen's in the form $\mathbb{E}[|X|^{q-1}] \leq (\mathbb{E}[|X|^q])^{(q-1)/q}$. Substituting the moment bounds from Lemma~\ref{lemmaboundonupsilon} then yields the claimed result:
\begin{align}
|\Psi_{i,j}| &\leq B\sqrt{2(1+R_{2q})} + K\big(1+R_q^{\frac{q-1}{q}}\big).
\end{align}
\end{proof}

\begin{lemma}\label{lemmaMmatrixbound}
For $q > 0$, the entries of the matrix $\bm{M}_t$ satisfy the following elementwise bound:

\begin{align}
    |M_{i,j}| &\leq  p\left(B\sqrt{2(1+R_{2q})} +K(1+R_q^{\frac{q-1}{q}})\right),
\end{align}
\end{lemma}
where $B$ and $K$ are the growth constants of the nonlinearity and its derivative, $R_{q}$ and $R_{2q}$ are positive constants depending on $q$ and $p$ is the number of components (the dimension of the matrix $\bm{M}_t \in \mathbb{R}^{p \times p}$)
\begin{proof}
Expanding the matrix product and applying the triangle inequality, we bound each entry as a sum of $p$ terms. Using $|Q_{i,l}| \leq 1$ and substituting the elementwise bound on $\bm{\Psi}$ from Lemma~\ref{lemmaPsibound}, we obtain

\begin{align}
    |M_{i,j}| &\leq \sum_{l=1}^p |Q_{i,l}| |\Psi_{l,j}| ,\notag\\ &
    \leq \sum_{\alpha=1}^p |\Psi_{l,j}|, \notag\\ &
    \leq  \sum_{\alpha=1}^p \left(B\sqrt{2(1+R_{2q})} +K(1+R_q^{\frac{q-1}{q}})\right) ,\notag\\ & = p\left(B\sqrt{2(1+R_{2q})} +K(1+R_q^{\frac{q-1}{q}})\right),
\end{align}

where the final equality follows since the summand is independent of the index $l$, which concludes the proof.

\end{proof}

\begin{lemma}\label{lemmaXbound}
    For each coordinate $j$, the following bound holds for $k\leq nT$:
\begin{align}\label{expec_x_sq_bound}
    \mathbb{E}[\|\bm{x}_k^{(j)}\|^2] \leq A_1(T),
\end{align}
\begin{align}\label{expec_x_fourth_bound}
    \mathbb{E}[\|\bm{x}_k^{(j)}\|^4] \leq A_2(T),
\end{align}
where $A_1(T)$ and $A_2(T)$ are finite constants depending on $T$.   
\end{lemma}
\begin{proof}

Recall that $\bm{\Delta}_k^{(j)} \coloneqq \bm{x}_{k+1}^{(j)} - \bm{x}_{k}^{(j)}$. Expanding the squared norm of the next iterate yields

\begin{align} \label{eq:boundingxsqr}
   \mathbb{E}[\bm{x}_{k+1}^{(j)\top}\bm{x}_{k+1}^{(j)}] &= \mathbb{E}[(\bm{x}_k^{(j)}+\bm{\Delta}_k^{(j)})^\top(\bm{x}_k^{(j)}+\bm{\Delta}_k^{(j)})], \notag \\ &= \mathbb{E}[\|\bm{x}_k^{(j)}\|^2] +\mathbb{E}[\mathbb{E}_k[\|\bm{\Delta}_k^{(j)}\|^2]] +2 \mathbb{E}[\bm{x}_k^{(j)\top}\mathbb{E}_k[\bm{\Delta}_k^{(j)}]], \notag\\ 
   &=\mathbb{E}[\|\bm{x}_k^{(j)}\|^2] +\mathbb{E}[\text{Tr}\left(\mathbb{E}_k[\bm{\Delta_k}^{(j)}\bm{\Delta_k}^{(j)\top}]\right)] +2 \mathbb{E}[\bm{x}_k^{(j)\top} \frac{\bm{\mathcal{G}}^{(j)}}{n}] + o\left(\frac{1}{n}\right), \notag \\ 
   &= \mathbb{E}[\|\bm{x}_k^{(j)}\|^2] +\mathbb{E}[\text{Tr}\left(\frac{\bm{\Lambda}_k^{(j)}}{n}\right)] +2 \mathbb{E}[\bm{x}_k^{(j)\top} \frac{\bm{\mathcal{G}}^{(j)}}{n}] + o\left(\frac{1}{n}\right).
\end{align}

We now bound each of the terms on the right-hand side separately. For bounding the cross term we expand it by substituting the expression for $\bm{\mathcal{G}}^{(j)}$. We  let $\tilde{\bm{C}} \coloneqq \text{tril}(\bm{C}+\bm{C^\top})-\text{diag}(\bm{C})$, and similarly for  $\tilde{\bm{M}}$ and $\tilde{\bm{R}}$ for notational simplicity, then using Equation \eqref{equationG} we arrive at
\begin{align}\label{squared_expectation_ofx_third_term_forthe_bound}
\Bigg|\mathbb{E}\Big[\bm{x}_k^{(j)\top} \frac{\bm{\mathcal{G}}^{(j)}}{n}\Big]\Bigg| 
&\leq \frac{1}{n}\mathbb{E}\Bigg[\Bigg|\bm{x}_k^{(j)\top}\Big(
- \frac{\tau^2}{2}\bm{\tilde{C}}\,\bm{x}_k^{(j)}  
- \tau\,\bm{\tilde{M}}\,\bm{x}_k^{(j)} \notag \\ & \qquad \qquad \qquad \quad + \tau\bm{\tilde{R}}\bm{x}_k^{(j)}
+ \tau\,\bm{\Psi}_t\,\bm{u}^{(j)}-\tau \phi(\bm{x}_k^{(j)}) \Big)\Bigg|\Bigg], \notag \\ 
&\leq \frac{1}{n}\Bigg(\frac{\tau^2}{2}\mathbb{E}[\|\bm{x}_k^{(j)}\| \|\bm{\tilde{C}}\bm{x}_k^{(j)}\|] + \tau \mathbb{E}[\|\bm{x}_k^{(j)}\| \|\bm{\tilde{M}}\bm{x}_k^{(j)}\|] \notag \\ & \qquad  \qquad \qquad+ \tau \mathbb{E}[\|\bm{x}_k^{(j)}\| \|\bm{\tilde{R}}\bm{x}_k^{(j)}\| \notag \\ & \qquad \qquad \qquad + \tau \mathbb{E}[\|\bm{x}_k^{(j)}\| \|\bm{\Psi}\bm{u}^{(j)}\|]+\tau \mathbb{E}[\|\bm{x}_k^{(j)}\| \|\phi(\bm{x}_k)^{(j)}\|\Bigg).
\end{align}

Since $\phi$ is Lipschitz, we can establish the element-wise bound on $\bm{R}$ matrix as  $|R_{i,j}| = | \frac{\bm{x}_{k,i}^\top \phi(\bm{x}_{k,j})}{n}| \leq \frac{1}{n}  \|\bm{x}_{k,i}\| \|\phi(\bm{x}_{k,j})\| = \frac{1}{\sqrt{n}} 
\|\phi(\bm{x}_{k,j})\| $ where $|\phi(x)| \leq L|x|+|\phi(0)|$. Thus $|R_{i,j}| \leq \frac{1}{\sqrt{n}}(L \|\bm{x}_{k,j}\| + \sqrt{n}|\phi(0)|) = L + |\phi(0)|$. Here, we emphasize the distinction between $\bm{x}_{k,i} \in \mathbb{R}^n$ and $\bm{x}_k^{(j)} \in \mathbb{R}^p$. The former represents the $i$-th estimate vector, which has a norm of $\sqrt{n}$, while the latter is a $p$-dimensional vector constructed from the $j$-th elements of all estimate vectors.
Furthermore, we can obtain the following bounds:
\begin{align}
\mathbb{E}[\|\bm{x}_k^{(j)}\|\|\bm{\tilde{C}}\bm{x}_k^{(j)}\|] \leq \mathbb{E}[\|\bm{x}_k^{(j)}\|\|\bm{\tilde{C}}\|_2\|\bm{x}_k^{(j)}\|] = \mathbb{E}[\|\bm{\tilde{C}}\|_2\|\bm{x}_k^{(j)}\|^2],
\end{align}

\begin{align}
\mathbb{E}[\|\bm{x}_k^{(j)}\|\|\bm{\tilde{M}}\bm{x}_k^{(j)}\|] \leq \mathbb{E}[\|\bm{\tilde{M}}\|_2\|\bm{x}_k^{(j)}\|^2],
\end{align}

\begin{align}
\mathbb{E}[\|\bm{x}_k^{(j)}\|\|\bm{\tilde{R}}\bm{x}_k^{(j)}\|] \leq \mathbb{E}[\|\bm{\tilde{R}}\|_2\|\bm{x}_k^{(j)}\|^2],
\end{align}

\begin{align}
    \mathbb{E}[\|\bm{x}_k^{(j)\top}\|\|\bm{\Psi}_t\,\bm{u}^{(j)}\|] &\leq \mathbb{E}\big[\|\bm{\Psi}_t\|_2\|\bm{x}_k^{(j)}\|\,\|\bm{u}^{(j)}\|\big],
\end{align}

For a $p \times p$ matrix with element-wise bound $|A_{ij}| \leq \alpha$, the spectral norm satisfies $\|\bm{A}\|_2 \leq \|\bm{A}\|_F \leq p\,\alpha$. First we note that since $\bm{\tilde{C}} = \text{tril}(\bm{C} + \bm{C}^\top) - \text{diag}(\bm{C})$, its entries take the form

\begin{align}
    \tilde{C}_{ij} = 
    \begin{cases}
        C_{ii}, & i = j, \\
        C_{ij} + C_{ji}, & i > j, \\
        0, & i < j.
    \end{cases}
\end{align}

Applying the triangle inequality together with the element-wise bound on $|C_{ij}|$ from Lemma~\ref{lemmaCmatrix}, and analogously for $\bm{\tilde{M}}$ and $\bm{\tilde{R}}$ yields

\begin{align}
    |\tilde{C}_{ij}| &\leq 2B^2(1+2R_q+R_{2q}), \qquad  \\[2ex]
    |\tilde{M}_{ij}| &\leq  2p\left(B\sqrt{2(1+R_{2q})} +K(1+R_q^{\frac{q-1}{q}})\right) , \\[2ex]
    |\tilde{R}_{ij}| &\leq 2(L + |\phi(0)|).
\end{align}

Finally, substituting these results into the elementwise bounds leads to the following:
\begin{align}
\mathbb{E}[\|\bm{x}_k^{(j)}\|\|\bm{\tilde{C}}\bm{x}_k^{(j)}\|] &\leq  \mathbb{E}[\|\bm{\tilde{C}}\|_2\|\bm{x}_k^{(j)}\|^2], \notag \\ & \leq \mathbb{E}[2pB^2(1+2R_q+R_{2q})\|\bm{x}_k^{(j)}\|^2], \notag \\ &= 2pB^2(1+2R_q+R_{2q})\mathbb{E}[\|\bm{x}_k^{(j)}\|^2], \notag \\[2ex] 
\mathbb{E}[\|\bm{x}_k^{(j)}\|\|\bm{\tilde{M}}\bm{x}_k^{(j)}\|] &\leq 2p^2\left(B\sqrt{2(1+R_{2q})} +K(1+R_q^{\frac{q-1}{q}})\right)\mathbb{E}[\|\bm{x}_k^{(j)}\|^2], \notag \\[2ex]
\mathbb{E}[\|\bm{x}_k^{(j)}\|\|\bm{\tilde{R}}\bm{x}_k^{(j)}\|] &\leq 2p(L + |\phi(0)|)\mathbb{E}[\|\bm{x}_k^{(j)}\|^2], \notag \\[2ex]
\mathbb{E}\big[\|\bm{\Psi}_t\|_2\|\bm{x}_k^{(j)}\|\,\|\bm{u}^{(j)}\|\big] &\leq p\Bigg(B\sqrt{2(1+R_{2q})} + K\big(1+R_q^{\frac{q-1}{q}}\big)\Bigg)\,\mathbb{E}\big[\|\bm{x}_k^{(j)}\|\,\|\bm{u}^{(j)}\|\big] ,\notag \\
    &\leq p\Bigg(B\sqrt{2(1+R_{2q})} + K\big(1+R_q^{\frac{q-1}{q}}\big)\Bigg)\sqrt{\mathbb{E}[\|\bm{x}_k^{(j)}\|^2]\,\mathbb{E}[\|\bm{u}^{(j)}\|^2]}, \notag \\ 
    & \leq p\Bigg(B\sqrt{2(1+R_{2q})} + K\big(1+R_q^{\frac{q-1}{q}}\big)\Bigg) \frac{\mathbb{E}[\|\bm{x}_k^{(j)}\|^2]+\mathbb{E}[\|\bm{u}^{(j)}\|^2]}{2},\notag \\[2ex]
    \big|\mathbb{E}[\bm{x}_{k}^{(j)\top}\phi(\bm{x}_k^{(j)})]\big|  
&\leq \mathbb{E}\big[\|\bm{x}_k^{(j)}\|\,\|\phi(\bm{x}_k^{(j)})\|\big],
\notag \\ & \leq \sqrt{\mathbb{E}[\| \bm{x}_k^{(j)}\|^2] \mathbb{E}[\|\phi(\bm{x}_k^{(j)})\|^2]}, \notag \\ &\leq \frac{\mathbb{E}[\| \bm{x}_k^{(j)}\|^2] +\mathbb{E}[\|\phi(\bm{x}_k^{(j)})\|^2]}{2}, \notag
\end{align}

where we apply the inequality $\sqrt{ab} \leq \frac{a+b}{2}$ to eliminate the square roots. Combining these bounds and substituting them back into \eqref{squared_expectation_ofx_third_term_forthe_bound} yields
\begin{align}
    \Big|\mathbb{E}\Big[\bm{x}_k^{(j)\top} \frac{\bm{\mathcal{G}}^{(j)}}{n}\Big]\Big| &\leq \frac{p}{n}\Bigg(
\frac{\tau^2}{2}(2B^2(1+2R_q+R_{2q})) \notag \\ 
& \qquad \qquad + \tau 2p\Big(B\sqrt{2(1+R_{2q})} +K(1+R_q^{\frac{q-1}{q}})\Big)  \notag \\ & \qquad \qquad  + 2\tau(L + |\phi(0)|)
\Bigg) \mathbb{E}[\|\bm{x}_k^{(j)}\|^2] \notag \\ & \qquad
+ \frac{\tau p}{n}\Big(B\sqrt{2(1+R_{2q})} + K\big(1+R_q^{\frac{q-1}{q}}\big)\Big)\frac{\mathbb{E}[\|\bm{x}_k^{(j)}\|^2]+\mathbb{E}[\|\bm{u}^{(j)}\|^2]}{2} \notag \\ & \qquad +\frac{\tau}{n} \frac{\mathbb{E}[\| \bm{x}_k^{(j)}\|^2] +  \mathbb{E}[\|\phi(\bm{x}_k^{(j)})\|^2]}{2},  \notag \\
    &\leq \frac{p}{n}\Bigg(
\frac{\tau^2}{2}(2B^2(1+2R_q+R_{2q})) \notag \\ & \qquad \qquad + \tau (2p+\frac{1}{2})\left(B\sqrt{2(1+R_{2q})} +K(1+R_q^{\frac{q-1}{q}})\right) \notag \\ & \qquad \qquad + 2\tau(L + |\phi(0)|) +\frac{\tau}{2p} 
\Bigg)\mathbb{E}[\|\bm{x}_k^{(j)}\|^2] \notag \\ & \qquad
+ \frac{\tau p}{2n}(B\sqrt{2(1+R_{2q})} + K\big(1+R_q^{\frac{q-1}{q}}\big))\mathbb{E}[\|\bm{u}^{(j)}\|^2] \notag \\ & \qquad+ \frac{\tau}{2n} \mathbb{E}[\|\phi(\bm{x}_k^{(j)})\|^2].
\end{align}

Furthermore, the Lipschitz assumption on $\phi$ yields $\mathbb{E}[\|\phi(\bm{x}_k^{(j)})\|^2]\leq 2L^2 \mathbb{E}[\|\bm{x}_k^{(j)}\|^2]+2p\phi(0)^2$ Consequently, we arrive at the following bound:
\begin{align}
    \Big|\mathbb{E}\Big[\bm{x}_k^{(j)\top} \frac{\bm{\mathcal{G}}^{(j)}}{n}\Big]\Big| 
    &\leq \frac{p}{n}\Bigg(
\frac{\tau^2}{2}(2B^2(1+2R_q+R_{2q})) \notag \\ & \qquad \qquad + \tau (2p+\frac{1}{2})\left(B\sqrt{2(1+R_{2q})} +K(1+R_q^{\frac{q-1}{q}})\right) \notag \\ & \qquad \qquad  + 2\tau(L + |\phi(0)|) +\frac{\tau}{2p}
\Bigg)\mathbb{E}[\|\bm{x}_k^{(j)}\|^2] \notag \\ & \qquad
+ \frac{\tau p}{2n}(B\sqrt{2(1+R_{2q})} + K\big(1+R_q^{\frac{q-1}{q}}\big))\mathbb{E}[\|\bm{u}^{(j)}\|^2] \notag \\ & \qquad + \frac{\tau}{2n} (2L^2 \mathbb{E}[\|\bm{x}_k^{(j)}\|^2]+2p\phi(0)^2), \notag \\ & =  \frac{p}{n}\Bigg(
\frac{\tau^2}{2}(2B^2(1+2R_q+R_{2q})) \notag \\ & \qquad \qquad + \tau (2p+\frac{1}{2})\left(B\sqrt{2(1+R_{2q})} +K(1+R_q^{\frac{q-1}{q}})\right) \notag \\ & \qquad \qquad  + 2\tau(L + |\phi(0)|) +\frac{\tau}{2p} + \frac{\tau}{2p}2L^2
\Bigg)\mathbb{E}[\|\bm{x}_k^{(j)}\|^2] \notag \\ & \qquad
+ \frac{\tau p}{2n}(B\sqrt{2(1+R_{2q})} + K\big(1+R_q^{\frac{q-1}{q}}\big))\mathbb{E}[\|\bm{u}^{(j)}\|^2] + \frac{\tau}{n} p\phi(0)^2.
\end{align}

The second term in Equation \eqref{eq:boundingxsqr} can be bounded directly using Lemma~\ref{lemmaCmatrix} as follows:
\begin{align}
    \mathbb{E}[\text{Tr}(\frac{\bm{\Lambda}_k^{(j)}}{n})] &= \mathbb{E}[\text{Tr}(\tau^2 \frac{\bm{C}_t}{n})], \notag \\ & =\frac{\tau^2}{n}\mathbb{E}[\sum_{i=1}^pC_{i,i}] ,\notag \\ & \leq \frac{\tau^2}{n}p B^2(1+2R_q + R_{2q}).
\end{align}

Combining the bounds on all three terms, we arrive at a recursion of the form $a_{k+1} \leq a_k J_1 + J_2$, where $a_k \coloneqq \mathbb{E}[\|\bm{x}_k^{(j)}\|^2]$. Unrolling the recursion and summing the resulting geometric series yields

\begin{align}
    a_{k+1} &\leq a_{k}J_1 + J_2 ,\notag \\ & =(a_{k-1}J_1 + J_2)J_1 +J_2 ,\notag\\
    &=a_0J_1^{k+1} + J_2(1+J_1 + J_1^2 + \dots + J_1^{k}), \notag\\ 
    &= a_0J_1^{k+1} + J_2\frac{J_1^{k+1}-1}{J_1 -1}, \\[3ex]
    a_{k} &\leq a_0J_1^{k} + J_2\frac{J_1^{k}-1}{J_1 -1}, \\[2ex]
    \mathbb{E}[\| \bm{x}_{k}^{(j)}\|^2] &\leq \mathbb{E}[\| \bm{x}_{0}^{(j)}\|^2]J_1^{k} + J_2\frac{J_1^{k}-1}{J_1 -1},
\end{align}

where we can define $J_1$ and $J_2$ as
\begin{align}
    J_1 &= 2\frac{p}{n}\Bigg(
\frac{\tau^2}{2}\Big(2B^2(1+2R_q+R_{2q})\Big) + \tau \Big(2p+\frac{1}{2}\Big)\left(B\sqrt{2(1+R_{2q})} +K(1+R_q^{\frac{q-1}{q}})\right) \notag \\ & \qquad \qquad + 2\tau(L + |\phi(0)|) +\frac{\tau}{2p} + \frac{\tau}{2p}2L^2
\Bigg) +1 ,\notag \\ & = \frac{V(q,\tau,p)}{n} +1, \\[2ex]
    J_2 &= \frac{\tau^2}{n}pB^2(1+2R_q+R_{2q}) \notag \\ & \qquad   + 2\Bigg(\frac{\tau p}{2n}(B\sqrt{2(1+R_{2q})} + K\big(1+R_q^{\frac{q-1}{q}}\big))\mathbb{E}[\|\bm{u}^{(j)}\|^2] + \frac{\tau}{n} p\phi(0)^2\Bigg) ,\notag \\ &= \frac{Z(q,\tau,p)}{n}.
\end{align}

Recall that the initial moments of the estimates ($\bm{x}_{0,i}$) and all moments of the true components ($\bm{u}_i$) are assumed to be bounded by the assumptions of our theorem. Thus, for $k \leq nT$, we use the classical limit $(1 + a/n)^n \to e^a$ to conclude that $(1 + V/n)^{nT} \leq e^{VT}$ uniformly in $n$, arriving at
\begin{align}
    \mathbb{E}[\| \bm{x}_{k}^{(j)}\|^2] &\leq \mathbb{E}[\| \bm{x}_{0}^{(j)}\|^2]\Big(\frac{V(q,\tau,p)}{n} +1\Big)^{nT} + \Big(\frac{Z(q,\tau,p)}{n}\Big)\frac{\Big(\frac{V(q,\tau,p)}{n} +1\Big)^{nT}-1}{\Big(\frac{V(q,\tau,p)}{n} +1\Big) -1} ,\notag \\ & \leq A_1(T),
\end{align}
where $A_1(T)$ is a finite constant depending on $T$. We have established \eqref{expec_x_sq_bound}; the proof of \eqref{expec_x_fourth_bound} is analogous. This completes the proof.
\end{proof}
For the squared norm of $\bm{\mathcal{G}}_k^{(j)}$, we can write the following expression: 
\begin{align}
 \|\bm{\mathcal{G}}_k^{(j)}\|^2  
 &= \frac{\tau^4}{4}\bm{x}_k^{(j)\top }\bm{\tilde{C}}^\top\bm{\tilde{C}}\bm{x}_k^{(j)} + \tau^2\bm{x}_k^{(j)\top }\bm{\tilde{M}}^\top\bm{\tilde{M}}\bm{x}_k^{(j)} + \tau^2 \bm{u}^{(j)\top }\bm{\Psi}^\top\bm{\Psi}\bm{u}^{(j)} \notag \\ & \qquad + \tau^2 \bm{x}_k^{(j)\top }\bm{\tilde{R}}^\top\bm{\tilde{R}}\bm{x}_k^{(j)} + \tau^2 \phi(\bm{x}_k^{(j)})^\top\phi(\bm{x}_k^{(j)})  \notag \\ 
 &  \qquad +\tau^3\bm{x}_k^{(j)\top }\bm{\tilde{C}}^\top\bm{\tilde{M}}\bm{x}_k^{(j)} -\tau^3\bm{x}_k^{(j)\top }\bm{\tilde{C}}^\top\bm{\Psi}\bm{u}^{(j)}  -2 \tau^2\bm{x}_k^{(j)\top }\bm{\tilde{M}}^\top\bm{\Psi}\bm{u}^{(j)}  \notag \\ & \qquad - \tau^3 \bm{x}_k^{(j)\top }\bm{\tilde{C}}^\top\bm{\tilde{R}}\bm{x}_k^{(j)} -2\tau^2\bm{x}_k^{(j)\top }\bm{\tilde{M}}^\top\bm{\tilde{R}}\bm{x}_k^{(j)} + 2\tau^2\bm{x}_k^{(j)\top }\bm{\tilde{R}}^\top\bm{\Psi}\bm{u}^{(j)} \notag \\ & \qquad + \tau^3\bm{x}_k^{(j)\top }\bm{\tilde{C}}^\top\phi(\bm{x}_k^{(j)}) + 2\tau^2\bm{x}_k^{(j)\top }\bm{\tilde{M}}^\top\phi(\bm{x}_k^{(j)}) \notag \\ & \qquad - 2\tau^2\bm{x}_k^{(j)\top }\bm{\tilde{R}}^\top\phi(\bm{x}_k^{(j)}) -2\tau^2\bm{u}^{(j)\top }\bm{\Psi}^\top\phi(\bm{x}_k^{(j)}) .
\end{align}
Taking expectations and applying the triangle inequality, we have to establish bounds on every term. We start by obtaining
\begin{align}
\mathbb{E}[\bm{x}^\top \bm{\tilde{C}}^\top\bm{\tilde{C}}\bm{x}] &\leq \mathbb{E}[\|\bm{\tilde{C}}\|_2^2\,\|\bm{x}\|^2]\leq \Bigg(2pB^2(1+2R_q+R_{2q}) \Bigg)^2\mathbb{E}[\|\bm{x}_k^{(j)}\|^2],\\[2ex]
\mathbb{E}[\bm{x}^\top \bm{\tilde{M}}^\top\bm{\tilde{M}}\bm{x}] &\leq \mathbb{E}[\|\bm{\tilde{M}}\|_2^2\,\|\bm{x}\|^2] \leq 4p^4\Bigg(B\sqrt{2(1+R_{2q})} +K(1+R_q^{\frac{q-1}{q}})\Bigg)^2 \mathbb{E}[\|\bm{x}_k^{(j)}\|^2],\\[2ex]
\mathbb{E}[\bm{u}^\top \bm{\Psi}^\top\bm{\Psi}\bm{u}] &\leq \mathbb{E}[\|\bm{\Psi}\|_2^2\,\|\bm{u}\|^2] \leq p^2\Bigg(B\sqrt{2(1+R_{2q})} + K\big(1+R_q^{\frac{q-1}{q}}\big)\Bigg)^2 \mathbb{E}[\|\bm{u}^{(j)}\|^2]  ,\\[2ex]
\mathbb{E}[\bm{x}^\top \bm{\tilde{R}}^\top\bm{\tilde{R}}\bm{x}] &\leq \mathbb{E}[\|\bm{\tilde{R}}\|_2^2\,\|\bm{x}\|^2] \leq 4p^2(L + |\phi(0)|)^2 \mathbb{E}[\|\bm{x}_k^{(j)}\|^2], \\[2ex]
\mathbb{E}[\|\phi(\bm{x}_k^{(j)})\|^2] &\leq 2L^2 \mathbb{E}[\|\bm{x}_k^{(j)}\|^2]+2p\phi(0)^2.
\end{align}

The same bounding logic can be applied to the cross terms by substituting the spectral norm bounds we obtained in the previous sections. All these terms, with the exception of $\mathbb{E}[\|\bm{x}_k^{(j)}\|^2]$, are independent of $n$. Having established that $\max_{k \leq nT} \mathbb{E}[\|\bm{x}_k^{(j)}\|^2] \leq A_1(T)$ in Lemma \ref{lemmaXbound}, we can conclude:
\begin{align}
    \max_{k \leq nT} \mathbb{E} \|\bm{\mathcal{G}}_k^{(j)}\|_2^2 \leq C(T).
\end{align}
The second condition of Assumption \ref{assumption8} requires bounding the Frobenius norm of $\bm{\Lambda}_k$, we obtain
\begin{align}
    \| \bm{\Lambda}_k \|_F^2 = \tau^4\sum_i\sum_j ((\bm{C}_t)_{i,j})^2.
\end{align}
Applying the element-wise bound we obtained from Lemma \ref{lemmaCmatrix} yields
\begin{align}
    \mathbb{E}[\| \bm{\Lambda}_k \|_F^2] &= \mathbb{E}[\tau^4\sum_{i=1}^p\sum_{j=1}^p ((\bm{C}_k)_{i,j})^2], \notag  \\& \leq \tau^4 p^2 B^4(1+2R_q+R_{2q})^2,
\end{align}
which is again a finite constant independent of $n$ and $k$. Consequently, we conclude that 
\begin{align}
    \max_{k \leq nT} \mathbb{E} \|\bm{\Lambda}_k\|_F^2 \leq C(T).
\end{align}
\subsection{Assumption 9 } \label{assumption9}
\textit{For each $T > 0,$ there exists $C(T)< \infty$ such that $\text{max}_{k \leq nT} \mathbb{E}[\|\bm{\Delta}_k^{(j)}\|^4] \leq C(T)n^{-2}$, and for any $i \neq j$}:

\begin{align}
    \text{max}_{k \leq nT} \mathbb{E}\Bigg[\Bigg | \mathbb{E}_k\Big[\big(\bm{\Delta}_k^{(i)}-\mathbb{E}_k[\bm{\Delta}_k^{(i)}]\big)^\top\big(\bm{\Delta}_k^{(j)}-\mathbb{E}_k[\bm{\Delta}_k^{(j)}]\big)\Big]\Bigg | \Bigg] \leq \frac{C(T)}{n^2}.
\end{align}

First, by Lemma \ref{lemmaXbound}, the bound $\text{max}_{k \leq nT} \mathbb{E}[\|\bm{\Delta}_k^{(j)}\|^4] \leq C(T)n^{-2}$ follows from the same argument used to verify Assumption \ref{assumption8}. For the latter condition, we expand the product:
\begin{align}
\mathbb{E}_k\Big[\big(\bm{\Delta}_k^{(i)}-\mathbb{E}_k[\bm{\Delta}_k^{(i)}]\big)^\top\big(\bm{\Delta}_k^{(j)}-\mathbb{E}_k[\bm{\Delta}_k^{(j)}]\big)\Big] = \mathbb{E}_k[\bm{\Delta}_k^{(i)\top}\bm{\Delta}_k^{(j)}]-\mathbb{E}_k[\bm{\Delta}_k^{(i)\top}]\mathbb{E}_k[\bm{\Delta}_k^{(j)}],
\end{align}

From Equation \ref{g1g2transpose}, the cross terms are of order $O(1/n^2)$ when $i \neq j$. Combined with the definition of $\bm{\Delta}_k^{(j)}$, this yields $\mathbb{E}[\bm{\Delta}_k^{(j)\top}\bm{\Delta}_k^{(i)}] = O(1/n^2)$. Following the definition of $\bm{\mathcal{G}}_k^{(j)}$, we arrive at
\begin{align}
\mathbb{E}_k\Big[\big(\bm{\Delta}_k^{(i)}-\mathbb{E}_k[\bm{\Delta}_k^{(i)}]\big)^\top\big(\bm{\Delta}_k^{(j)}-\mathbb{E}_k[\bm{\Delta}_k^{(j)}]\big)\Big] &= - \frac{\bm{\mathcal{G}}_k^{(i)\top}\bm{\mathcal{G}}_k^{(j)}}{n^2} + O(1/n^2).
\end{align}
Finally, taking the outer expectation on both sides yields
\begin{align}
\mathbb{E}\Bigg[\Bigg|\mathbb{E}_k\Big[\big(\bm{\Delta}_k^{(i)}-\mathbb{E}_k[\bm{\Delta}_k^{(i)}]\big)^\top\big(\bm{\Delta}_k^{(j)}-\mathbb{E}_k[\bm{\Delta}_k^{(j)}]\big)\Big]\Bigg|\Bigg] \leq \mathbb{E} \Bigg[\Bigg| - \frac{\bm{\mathcal{G}}_k^{(i)\top}\bm{\mathcal{G}}_k^{(j)}}{n^2} + O(1/n^2)\Bigg|\Bigg].
\end{align}
For the right hand side, we can establish the following bound as
\begin{align} \label{ass_9_equation_bound}
    \mathbb{E} \Bigg[\Bigg| - \frac{\bm{\mathcal{G}}_k^{(i)\top}\bm{\mathcal{G}}_k^{(j)}}{n^2} + O(1/n^2)\Bigg|\Bigg]&  \leq  \mathbb{E} \Bigg[\Bigg |  \frac{[\bm{\mathcal{G}}_k^{(i)\top}\bm{\mathcal{G}}_k^{(j)}]}{n^2} \Bigg |\Bigg] + O(1/n^2), \notag\\ 
    & \leq  \frac{\sqrt{\mathbb{E}[\|\bm{\mathcal{G}}_k^{(i)}\|^2]\mathbb{E}[\|\bm{\mathcal{G}}_k^{(j)}]\|^2]}}{n^2} + O(1/n^2), \notag\\ & 
    \leq \frac{\mathbb{E}[\|\bm{\mathcal{G}}_k^{(i)}\|^2]+ \mathbb{E}[\|\bm{\mathcal{G}}_k^{(j)}\|^2]}{2n^2}+ O(1/n^2).
\end{align}

In the verification of Assumption \ref{assumption8}, we showed that for some $C(T)<\infty$ we have
\begin{align}
    \max_{k \leq nT} \mathbb{E} \|\bm{\mathcal{G}}_k^{(j)}\|^2 \leq C(T).
\end{align}
Consequently, Assumption 9 is satisfied directly from Equation \ref{ass_9_equation_bound}, which establishes our claim: 
    \begin{align}
    \text{max}_{k\leq nT}\mathbb{E}\Bigg|[\mathbb{E}_k\Big[\big(\bm{\Delta}_k^{(i)}-\mathbb{E}_k[\bm{\Delta}_k^{(i)}]\big)^\top\big(\bm{\Delta}_k^{(j)}-\mathbb{E}_k[\bm{\Delta}_k^{(j)}]\big)\Big]\Bigg| \leq \frac{C(T)}{n^2}.
\end{align}
\subsection{Assumption 10}\label{asssumption10}
For each $b > 0$ and $T > 0$, the following PDE (in weak form) has a unique solution in $D([0,T], \mathcal{M}(\mathbb{R}^{2p}))$ for all bounded test functions $\varphi(\bm{x}, \bm{u}) \in \mathcal{C}^3(\mathbb{R}^{2p})$:
\begin{align}
    \langle \varphi, \mu_t \rangle &= \langle \varphi, \mu_0 \rangle + \int_0^t \left\langle \nabla_{\bm{x}}^\top \varphi \  \bm{\mathcal{G}}(\bm{x}_{\hat{t}}, \bm{u}_{\hat{t}}, \bm{\Theta}_{\hat{t}} \sqcap b) , \mu_{\hat{t}} \right\rangle \mathrm{d}\hat{t} \notag \\ & \qquad + \frac{1}{2} \int_0^t \left\langle \operatorname{Tr}(\bm{\Lambda}(\bm{x}_{\hat{t}}, \bm{u}_{\hat{t}}, \bm{\Theta}_{\hat{t}} \sqcap b) \nabla_{\bm{x}}^2 \varphi), \mu_{\hat{t}} \right\rangle \mathrm{d}\hat{t}.
\end{align}
The sufficient conditions for this uniqueness assumption are:

\textbf{10.1}\ \ $\langle \|\bm{u}\|^2, \mu_0 \rangle \le L$ and $\langle \|\bm{x}\|^2, \mu_0 \rangle \le V$, where $L, V$ are two generic constants.

\textbf{10.2}\ \ For any $\bm{x},\bm{\tilde{x}}, \bm{u} $, and $\bm{\Theta}$, we have $\| \bm{\mathcal{G}}(\bm{x}, \bm{u}, \bm{\Theta}) - \bm{\mathcal{G}}(\tilde{\bm{x}}, \bm{u}, \bm{\Theta})\| \le L (1 + \|\bm{\Theta}\|_F) \|(\bm{x} - \tilde{\bm{x}})\|$.

\textbf{10.3}\ \ $|\bm{\mathcal{G}}(\bm{x}, \bm{u}, \bm{\Theta}) - \bm{\mathcal{G}}(\bm{x}, \bm{u}, \tilde{\bm{\Theta}})| \le L (1 + \|\bm{u}\| + \|\bm{x}\|) \|\bm{\Theta} - \tilde{\bm{\Theta}}\|_F$.

\textbf{10.4}\ \ $|\bm{\mathcal{G}}(\bm{x}, \bm{u}, \tilde{\bm{\Theta}})| \le L (1 + |\bm{u}| + |\bm{x}|) (\|\bm{\Theta}\|_F + 1)$.

\textbf{10.5} \ \ For any $\bm{\Theta}$ and $\tilde{\bm{\Theta}}$, we have $|\bm{\Lambda}^{\frac{1}{2}}(\bm{\Theta}) - \bm{\Lambda}^{\frac{1}{2}}(\tilde{\bm{\Theta}})| \le L \|\bm{\Theta} - \tilde{\bm{\Theta}}\|_F$.

\textbf{10.6} \ \ $\Lambda^{\frac{1}{2}}(\bm{\Theta}) \le L (1 + \|\bm{\Theta}\|_F)$.

Given our definitions of $\bm{\mathcal{G}}$ and $\bm{\Lambda}$, the verification of conditions (10.1)–(10.6) follows from standard algebraic arguments and is omitted for brevity. In particular, to satisfy Assumption 10.4, we required the function $\phi(x)$ to be Lipschitz.

Following that we verified assumptions \ref{assumption1}-\ref{asssumption10} hold in our setting, the deterministic measure $\mu_t$ is the  unique solution of the following partial differential equation: 

\begin{align}
\langle \varphi,\mu_t\rangle
=
\langle \varphi,\mu_0\rangle
+ \int_0^t
\Big\langle
\nabla_{\bm{x}}\varphi^{\top}\,\boldsymbol{\mathcal{G}}_s, \mu_s \Big\rangle\,\mathrm{d}s + \int_0^t\frac{1}{2}\,\Big\langle\mathrm{Tr}\!\left(\mathbf{\Lambda}_s\,\nabla_{\bm{x}}^2\varphi\right),
\mu_s
\Big\rangle\,\mathrm{d}s.
\end{align}

Furthermore, if $\mu_t$ admits a density $P_t(\bm{x},\bm{u})$, the following strong form PDE holds: 

\begin{equation}
\frac{\partial P_t}{\partial t}
=
-\nabla_{\bm{x}}\big(\boldsymbol{\mathcal{G}}_t\,P_t\big)
+ \frac{1}{2}\,\mathrm{Tr}\!\left(\mathbf{\Lambda}_t\,\nabla_{\bm{x}}^2 P_t\right).
\end{equation}

\subsection{Marginal probability densities} \label{marginaldensities}

For completeness, we provide the marginal probability densities of the first and second component estimates in Figure \ref{fig:pdf_x1} and Figure \ref{fig:pdf_x2}. These results supplement the joint density analysis of $(x_1, x_2)$ presented in the main text (Figure \ref{fig:1a} \ref{fig:1b}, \ref{fig:1c}) offering a detailed view of the individual coordinate distributions.
 \begin{figure}[t]
    \centering
    \begin{subfigure}[t]{\textwidth}
        \centering
        \includegraphics[width=\textwidth,height=0.20\textheight,keepaspectratio]{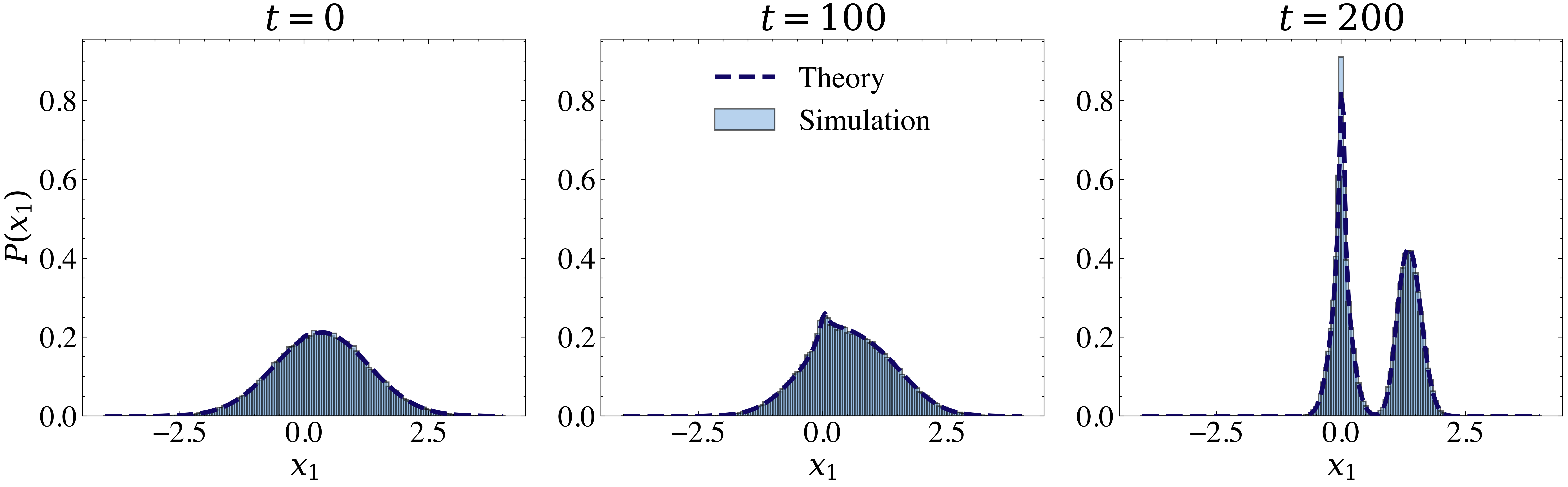}
        \caption{Marginal density of the first component \(x_1\) at times \(t = 0, 100,\) and \(200\).}
        \label{fig:pdf_x1}
    \end{subfigure}

    \vspace{0.8em}

    \begin{subfigure}[t]{\textwidth}
        \centering
        \includegraphics[width=\textwidth,height=0.20\textheight,keepaspectratio]{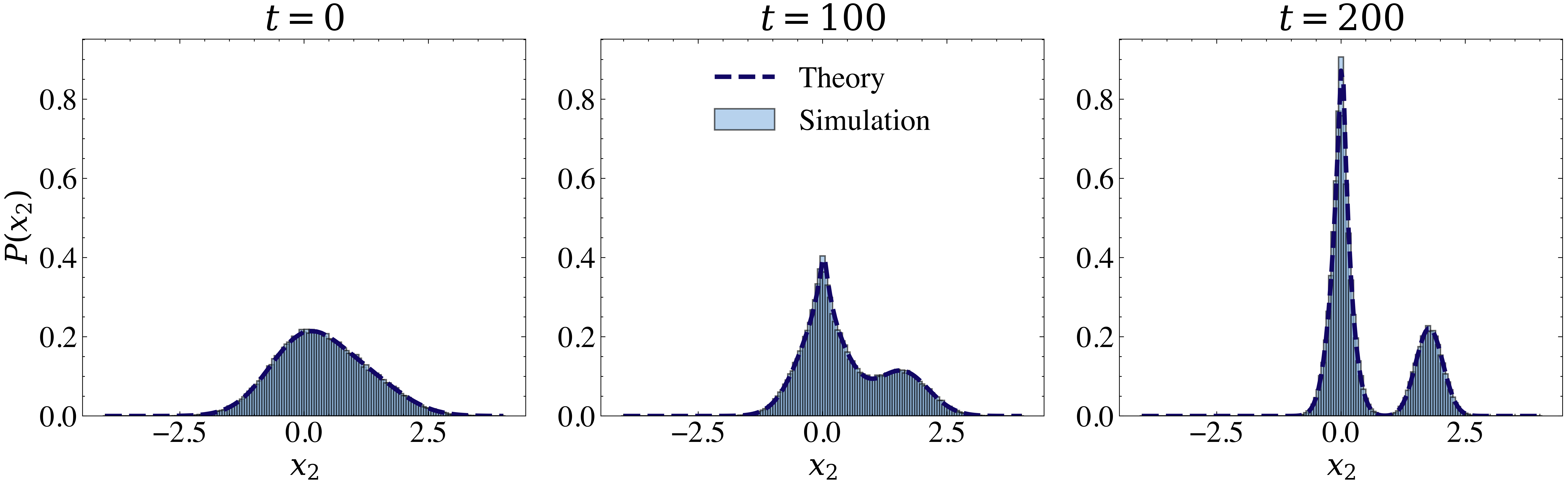}
        \caption{Marginal density of the second component \(x_2\) at times \(t = 0, 100,\) and \(200\).}
        \label{fig:pdf_x2}
    \end{subfigure}

    \caption{
        Evolution of the limiting marginal densities, corresponding to the total density in Figure \ref{figure1} in the main text. The setting corresponds to sparse component vectors $\mathbf{u}_1$ and $\mathbf{u}_2$, for $\phi(x) = 0.1sgn(x)$. Here we use the setting of Example~\ref{example:cubic}: we set $\beta_1=1$ for $c_1$, and set $\beta_2=0$ for $c_2$. In Monte Carlo simulations, component vectors are drawn from distributions ($\mathbb{P}(u=1/\sqrt{\rho_i}) = \rho_i$ and $\mathbb{P}(u=0) = 1-\rho_i$) with sparsity levels $\rho_1 = 0.5$ and $\rho_2 = 0.3$. The dashed dark blue curves correspond to the PDE prediction, while the light blue histograms represent Monte Carlo simulations. Here $n=5000$ and $\tau = 0.01$.}
    \label{fig:pdf_comparison}
\end{figure}

\begin{figure}[h!]
  \centering
  \begin{subfigure}[t]{0.30\textwidth}
    \centering
    \includegraphics[width=\linewidth]{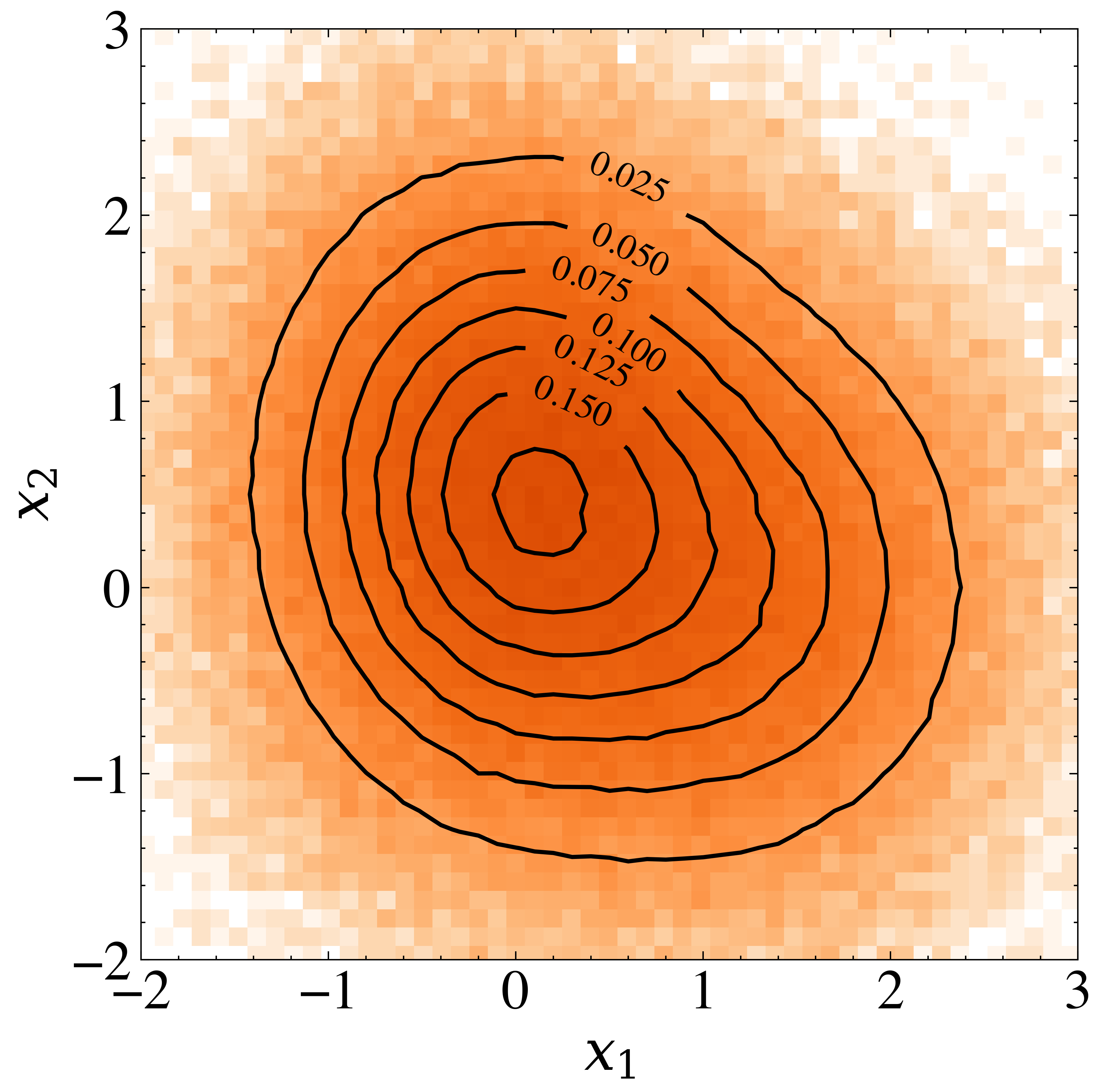}
    \caption{$t=0$}
    \label{fig:nonregul_heatmap_1}
  \end{subfigure}\hfill
  \begin{subfigure}[t]{0.30\textwidth}
    \centering
    \includegraphics[width=\linewidth]{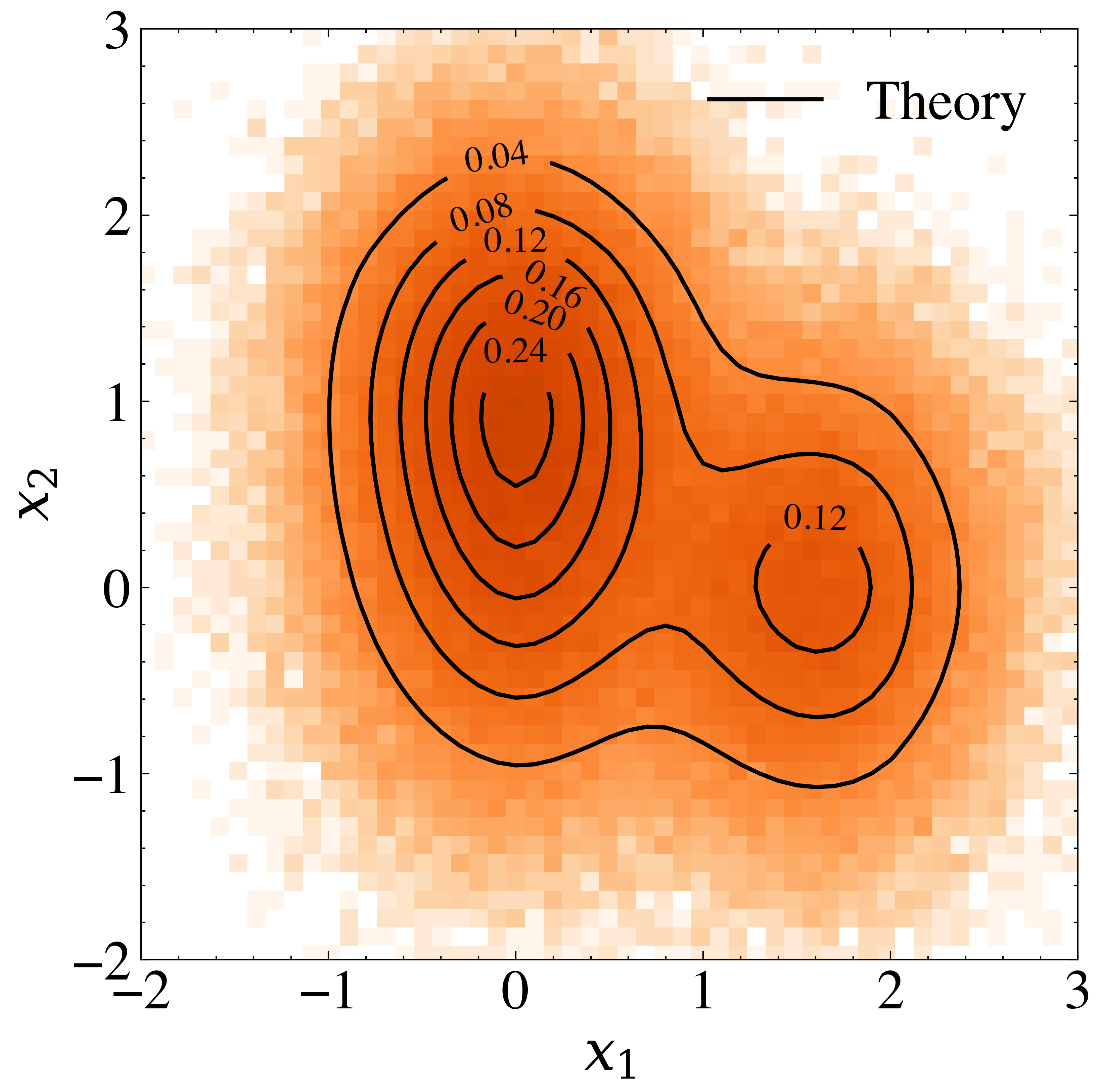}
    \caption{$t=100$}
    \label{fig:nonregul_heatmap_2}
  \end{subfigure}\hfill
  \begin{subfigure}[t]{0.385\textwidth}
    \centering
    \includegraphics[width=\linewidth]{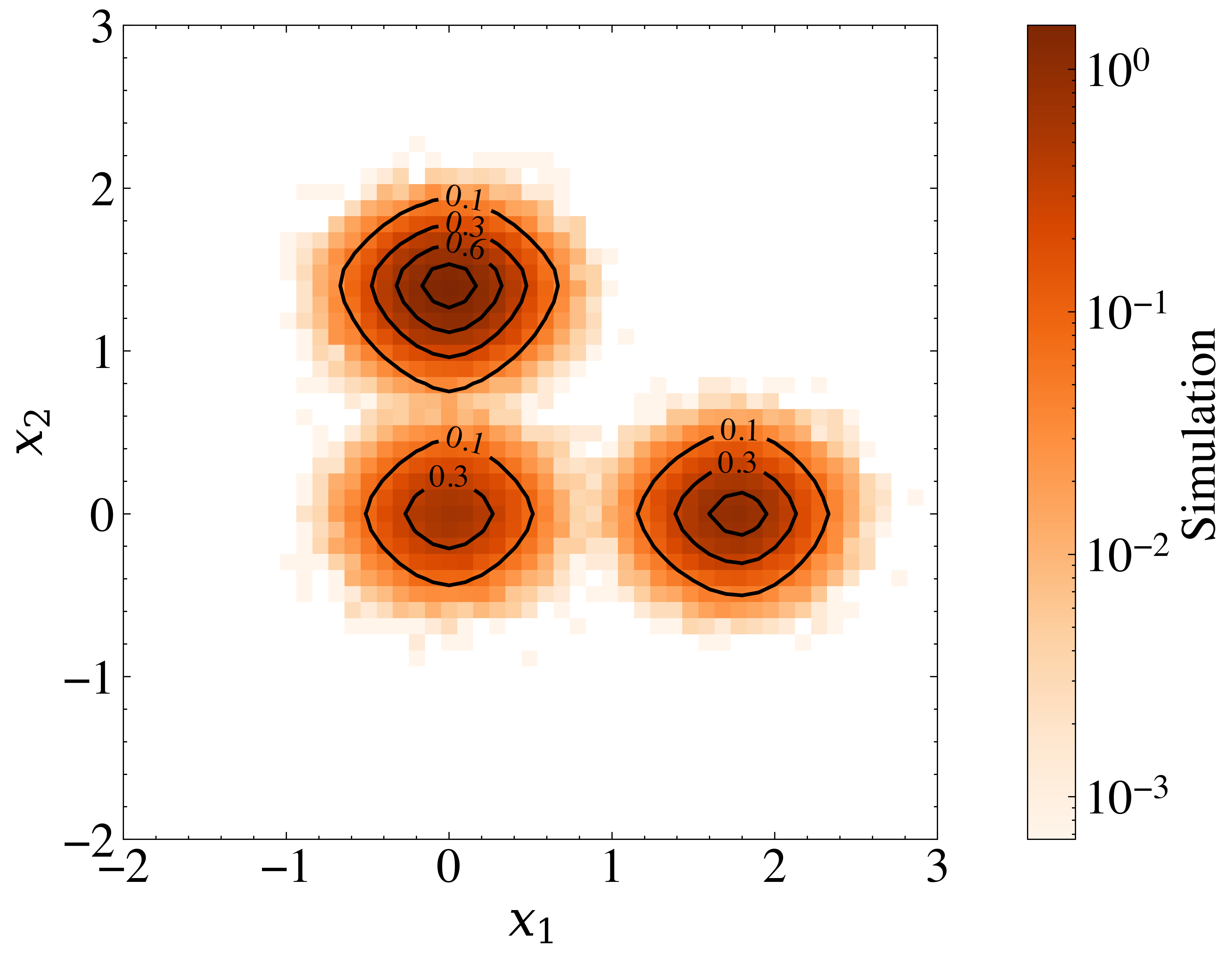}
    \caption{$t=200$}
    \label{fig:nonregul_heatmap_3}
  \end{subfigure}
  \caption{Evolution of the joint limiting probability density for $p=2$. Comparison between theoretical predictions (contours) and Monte Carlo simulations (heatmaps) for $P_t(x_1,x_2,u_1,u_2)$ at times $t=0,100,200$. The setting corresponds to sparse component vectors $\bm{u}_1$ and $\bm{u}_2$ for $\phi(x)= 0$. Here we use the setting of Example~\ref{example:cubic}. For $c_1$, we set $\beta_1=1$, for $c_2$, we set $\beta_2=0$. In simulations, component vectors are drawn from distributions ($\mathbb{P}(u=1/\sqrt{\rho_i}) = \rho_i$ and $\mathbb{P}(u=0) = 1-\rho_i$) with sparsity levels $\rho_1 = 0.5$ and $\rho_2 = 0.3$. This choice enables a clear visualization of the convergence behavior , with the components evolving toward three distinct regions in the $(\mathbf{x}_1,\mathbf{x}_2)$ space. Here, $n=1000$, $\tau = 0.01$}
  \label{nonregul_heatmap}
\end{figure}

\subsection{Derivation of Corollary \ref{corrolaryode}}\label{proofofcorrolary}
 \begin{figure}[h!]
    \centering
    \begin{subfigure}[t]{\textwidth}
        \centering
\includegraphics[width=\textwidth,height=0.20\textheight,keepaspectratio]{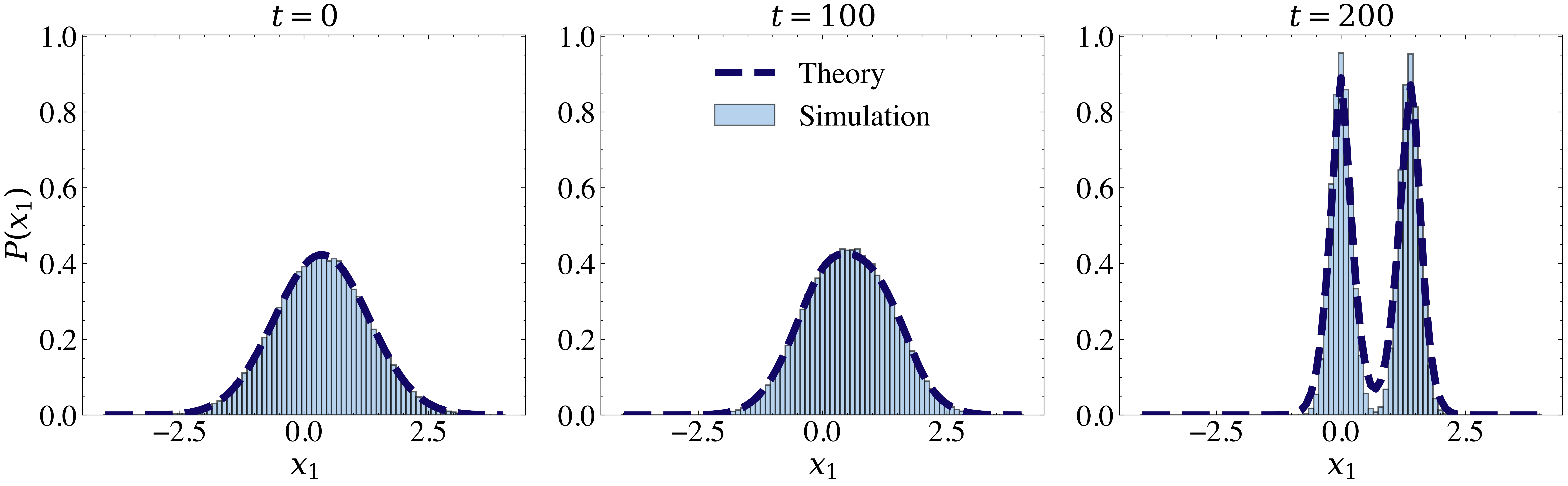}
        \caption{Marginal density of the first component \(x_1\) at times \(t = 0, 100,\) and \(200\).}
        \label{fig:non_regul_marginal1}
    \end{subfigure}

    \vspace{0.8em}

    \begin{subfigure}[t]{\textwidth}
        \centering
        \includegraphics[width=\textwidth,height=0.20\textheight,keepaspectratio]{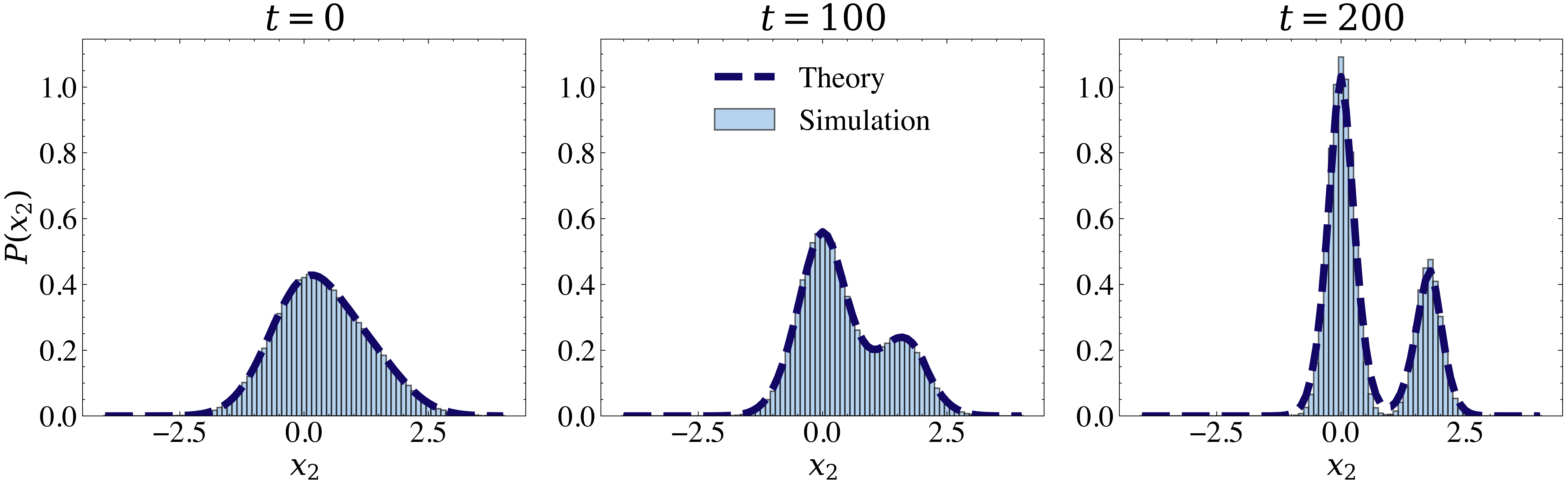}
        \caption{Marginal density of the second component \(x_2\) at times \(t = 0, 100,\) and \(200\).}
        \label{fig:non_regul_marginal2}
    \end{subfigure}
    \caption{
        Evolution of the limiting marginal densities corresponding to the total density in Figure \ref{nonregul_heatmap}. The setting corresponds to sparse component vectors $\mathbf{u}_1$ and $\mathbf{u}_2$, for $\phi(x) = 0$. Here we use the setting of Example~\ref{example:cubic}. For $c_1$, we set $\beta_1=1$, for $c_2$, we set $\beta_2=0$. In Monte Carlo simulations, component vectors are drawn from sparse distributions ($\mathbb{P}(u=1/\sqrt{\rho_i}) = \rho_i$ and $\mathbb{P}(u=0) = 1-\rho_i$) with sparsity levels $\rho_1 = 0.5$ and $\rho_2 = 0.3$. The dashed dark blue curves correspond to the PDE prediction, while the light blue histograms represent Monte Carlo simulations. Here $n = 1000$ and $\tau = 0.01$.}
    \label{fig:regularized_pdf}
\end{figure}

By selecting an appropriate test function $\varphi$ in \eqref{eq:weak_pde_main}, we can derive a closed-form system of ODEs. Specifically, we consider the case $\phi(x) = 0$, which consequently enforces $R_{i,j} = 0$. We provide additional total and marginal trajectories for the limiting PDE under $\phi(x) = 0$ in Figure \ref{nonregul_heatmap} and \ref{fig:regularized_pdf}, where our theory demonstrates close agreement with numerical simulations. To explicitly obtain the macroscopic evolution for the indices $(i,j)$, we instantiate the test function as

\begin{equation}
    \varphi_{l,m}(\bm{x},\bm{u}) := x_l u_m.
\end{equation}

Then the weak pairing recovers the overlap matrix: 
\begin{equation}
    \langle \varphi_{l,m},\mu_t\rangle = \int x_l u_m\,\mu_t(\mathrm{d}\bm{x},\mathrm{d}\bm{u}) = Q_{t,l,m}.
\end{equation}

Moreover, since $\varphi_{l,m}$ is linear in $\bm{x}$, we have $\nabla_{\bm{x}}^2\varphi_{l,m} \equiv 0$, so the diffusion term in \eqref{eq:weak_pde_main} vanishes. Therefore, the evolution reduces to

\begin{align}
\frac{d}{dt} Q_{t,l,m} &= \int u_m \, \omega_l \, \mu_t(\mathrm{d}\bm{x},\mathrm{d}\bm{u}).
\end{align}

This follows directly from \eqref{drfit_appendix} by setting $\phi(x) = 0 $ and $ R_{i,j} = 0$, yielding the ODE in summation form:
\begin{align}
\frac{d}{dt} Q_{lj}
&=
-\tau^2
\left(
\frac{1}{2} Q_{lj}\,\langle \gamma_l^2\rangle
+
\sum_{i<l} Q_{i,j}\,\langle \gamma_i \gamma_l\rangle
\right)
+\tau\,\psi_{l}^{(j)} -  \tau Q_{lj} \bm{Q}_l^\top \bm{\psi}_l  \notag \\
&\qquad \ \ \
-\tau
\sum_{i< l} Q_{i,j}
\Big(
\bm{Q}_i^\top \bm{\psi}_l
+
\bm{Q}_l^\top \bm{\psi}_i
\Big).
\end{align}

Expressing the above expression in matrix form using the previously defined matrices $\bm{C}_t$, $\bm{M}_t$ and $\bm{\Psi}_t$  yields the final closed-form ODE for the matrix $\bm{Q}_t$:
\begin{align}\label{eq:Q_ODE_appendix}
\frac{d}{dt}\,\bm{Q}
&=
-\frac{\tau^2}{2}
\mathcal{T}(\bm{C}_t)\bm{Q}
-\tau \mathcal{T}(\bm{M}_t)\bm{Q}+ \tau\,\bm{\Psi}_t^\top.
\end{align}

\section{Derivation of the ODE for cubic nonlinearity: $f(x) = \pm x^3$}\label{app:derivationforcube}

For the cubic nonlinearity $f(x)= \pm x^3$ and $p = 2$,  we explicitly derive the system of coupled ODEs. We begin by computing the expectations with respect to the random variables $c_{i,t}$ and $e_i$ for $i \in \{1,2\}$. For notational brevity, we omit the index $t$ and focus on the positive case; the derivation for $f(x)=-x^3$ follows immediately by symmetry. We start with the expression for $\gamma_i$:

\begin{align}
\gamma_i &= \Bigg(c_{1}Q_{i,1} + c_{2}Q_{i,2} + e_{1}\sqrt{1-Q_{i,1}^2 -Q_{i,2}^2}\Bigg)^3 ,\notag \\
&=
(c_{1}Q_{i,1})^3
+ (c_{2}Q_{i,2})^3
+ \Big(e_{1}\sqrt{1-Q_{i,1}^{2}-Q_{i,2}^{2}}\Big)^3 \notag \\[4pt]
&\qquad
+ 3(c_{1}Q_{i,1})^2(c_{2}Q_{i,2})
+ 3(c_{1}Q_{i,1})^2 e_{1}\sqrt{1-Q_{i,1}^{2}-Q_{i,2}^{2}} \notag \\[4pt]
&\qquad
+ 3(c_{2}Q_{i,2})^2(c_{1}Q_{i,1})
+ 3(c_{2}Q_{i,2})^2 e_{1}\sqrt{1-Q_{i,1}^{2}-Q_{i,2}^{2}} \notag \\[4pt]
&\qquad
+ 3\Big(e_{1}\sqrt{1-Q_{i,1}^{2}-Q_{i,2}^{2}}\Big)^2(c_{1}Q_{i,1})
+ 3\Big(e_{1}\sqrt{1-Q_{i,1}^{2}-Q_{i,2}^{2}}\Big)^2(c_{2}Q_{i,2}) \notag \\[4pt]
&\qquad
+ 6(c_{1}Q_{i,1})(c_{2}Q_{i,2})e_{1}\sqrt{1-Q_{i,1}^{2}-Q_{i,2}^{2}}.
\end{align}

For the derivative $\gamma_i^{'}$ we obtain the following:
\begin{align}
\gamma_i^{'} = 3\Bigg(c_{1}Q_{i,1} + c_{2}Q_{i,2} + e_{1}\sqrt{1-Q_{i,1}^2 -Q_{i,2}^2}\Bigg)^2.
\end{align}
Let $m_{i,j} := \langle c_i^j \rangle$ denote the $j$-th moment of the variable $c_i$, also recall that $\langle c_i \rangle = 0$  and $\langle c_i^2 \rangle = 1$, this yields the following expectations:
\begin{align}
\langle \gamma_i^{'} \rangle =  3(Q_{i,1}^2 + Q_{i,2}^2 + 1-Q_{i,1}^2 -Q_{i,2}^2) = 3,
\end{align}
\begin{align}\label{cgamma_i}
\langle c_1\gamma_i \rangle =  Q_{i,1}^3m_{1,4}  + 3Q_{i,2}^2Q_{i,1} + 3(1-Q_{i,1}^2 -Q_{i,2}^2)Q_{i,1} = Q_{i,1}^3(m_{1,4} - 3) + 3 Q_{i,1},
\end{align}
\begin{align}
\langle c_2\gamma_i \rangle =  Q_{i,2}^3(m_{2,4} - 3) + 3 Q_{i,2}.
\end{align}
Furthermore, we require the expectations of the squared terms $\gamma_i^2$ and cross term $\gamma_1\gamma_2$, which we calculate as follows:
\begin{align}
\langle \gamma_1^2 \rangle &=Q_{1,1}^6m_{1,6} + Q_{1,2}^6 m_{2,6} + 15(1-Q_{1,1}^2 -Q_{1,2}^2)^3 + 15Q_{1,2}^4(1-Q_{1,2}^2)m_{2,4} \notag \\
&\qquad + 15Q_{1,1}^4(1-Q_{1,1}^2)m_{1,4} 
+ 45(1-Q_{1,1}^2-Q_{1,2}^2)^2(Q_{1,1}^2 + Q_{1,2}^2) \notag \\
&\qquad 
+90Q_{1,1}^2Q_{1,2}^2(1-Q_{1,1}^2-Q_{1,2}^2) + 20Q_{1,2}^3Q_{1,1}^3 m_{1,3}m_{2,3} \  ,
\end{align}

\begin{align}
\langle \gamma_2^2 \rangle &=Q_{2,1}^6m_{1,6} + Q_{2,2}^6 m_{2,6} + 15(1-Q_{2,1}^2 -Q_{2,2}^2)^3 + 15Q_{2,2}^4(1-Q_{2,2}^2)m_{2,4} \notag \\
&\qquad
+ 15Q_{2,1}^4(1-Q_{2,1}^2)m_{1,4} 
 + 45(1-Q_{2,1}^2-Q_{2,2}^2)^2(Q_{2,1}^2 + Q_{2,2}^2)  \notag \\
&\qquad +90Q_{2,1}^2Q_{2,2}^2(1-Q_{2,1}^2-Q_{2,2}^2) +  20Q_{2,2}^3Q_{2,1}^3 m_{1,3}m_{2,3} \ ,
\end{align}

\begin{align}
\langle \gamma_1 \gamma_2 \rangle &= \Bigg\langle \Big(c_{1}Q_{1,1} + c_{2}Q_{1,2} + e_{1}\sqrt{1-Q_{1,1}^{2}-Q_{1,2}^{2}}\Big)^3 \notag \\ & \qquad \qquad  \times\Big(c_{1}Q_{2,1} + c_{2}Q_{2,2} + e_{2}\sqrt{1-Q_{2,1}^{2}-Q_{2,2}^{2}}\Big)^3 \Bigg\rangle ,\notag  \\
&= Q_{1,1}^3 \Bigg(Q_{2,1}^3m_{1,6}+Q_{2,2}^3m_{1,3}m_{2,3}+ 3(1-Q_{2,1}^2)Q_{2,1}m_{1,4}\Bigg) \notag \\
&\qquad + Q_{1,2}^3\Bigg(Q_{2,2}^3m_{2,6} +Q_{2,1}^3m_{1,3}m_{2,3} + 3(1-Q_{2,2}^2)Q_{2,2}m_{2,4}\Bigg) \notag \\
&\qquad
+3Q_{1,1}^2Q_{1,2}\Bigg(Q_{2,2}^3m_{2,4}+3Q_{2,1}^2Q_{2,2}m_{1,4} \notag \\ & \qquad \qquad \qquad \qquad \quad + 3Q_{2,2}^2Q_{2,1}m_{1,3}m_{2,3} + 3(1-Q_{2,1}^2-Q_{2,2}^2)Q_{2,2}\Bigg) \notag \\
&\qquad
+3Q_{1,2}^2Q_{1,1}\Bigg(Q_{2,1}^3m_{1,4}+3Q_{2,2}^2Q_{2,1}m_{2,4} \notag \\ & \qquad \qquad \qquad \qquad \quad + 3Q_{2,1}^2Q_{2,2}m_{1,3}m_{2,3} + 3(1-Q_{2,1}^2-Q_{2,2}^2)Q_{2,1}\Bigg) \notag \\
&\qquad
+ 3Q_{1,1}\Bigg(1-Q_{1,1}^2-Q_{1,2}^2\Bigg)\Bigg(Q_{2,1}^3m_{1,4}+3(1-Q_{2,1}^2)Q_{2,1}\Bigg) \notag \\
&\qquad
+3Q_{1,2}\Bigg(1-Q_{1,1}^2-Q_{1,2}^2\Bigg)\Bigg(Q_{2,2}^3m_{2,4} + 3(1-Q_{2,2}^2)Q_{2,2}\Bigg) \ .
\end{align}

To streamline the presentation of the ODEs, we define the following functions representing the expectations of $\gamma_i\gamma_j$:
\begin{align}
    F(Q_{1,1},Q_{2,1},Q_{1,2},Q_{2,2}) &\coloneqq \langle \gamma_1 \gamma_2 \rangle, \\[1em]
    W(Q_{i,1},Q_{i,2}) &\coloneqq \langle \gamma_i^2 \rangle.
\end{align}
We also define the function $D$ for notational convenience, which appears in the drift coefficient of the second estimate:
\begin{align}
    D(Q_{1,1},Q_{2,1},Q_{1,2},Q_{2,2}) &= Q_{1,1}Q_{2,1}^3(m_{1,4} -3) + Q_{1,2}Q_{2,2}^3(m_{2,4}-3) \notag \\ 
      &\qquad + Q_{2,1}Q_{1,1}^3(m_{1,4}-3) + Q_{2,2}Q_{1,2}^3(m_{2,4}-3).
\end{align}
Finally, for the drift and diffusion terms we arrive at the following final expressions:
\begin{align}
\omega_{2} &= \tau\Big(3x_{2}^{(j)} + u_1^{(j)}Q_{2,1}^3(m_{1,4}-3) + u_2^{(j)}Q_{2,2}^3(m_{2,4}-3)\Big) - \frac{x_{2}^{(j)}\tau^2}{2}W(Q_{2,1},Q_{2,2}) \notag \\
&\qquad - x_{1}^{(j)}\tau^2F(Q_{1,1},Q_{2,1},Q_{1,2},Q_{2,2}) -x_{1}^{(j)}\tau D(Q_{1,1},Q_{2,1},Q_{1,2},Q_{2,2}) \notag \\
&\qquad -x_{2}^{(j)}\tau\Big(3 + Q_{2,1}^4[m_{1,4}-3] +  Q_{2,2}^4[m_{2,4}-3]\Big) ,\notag \\
&= \tau\Big(u_1^{(j)}Q_{2,1}^3(m_{1,4}-3) + u_2^{(j)}Q_{2,2}^3(m_{2,4}-3)\Big) - \frac{x_{2}^{(j)}\tau^2}{2}W(Q_{2,1},Q_{2,2}) \notag \\
&\qquad - x_{1}^{(j)}\tau^2F(Q_{1,1},Q_{2,1},Q_{1,2},Q_{2,2}) 
-x_{1}^{(j)}\tau D(Q_{1,1},Q_{2,1},Q_{1,2},Q_{2,2}) \notag \\
&\qquad -x_{2}^{(j)}\tau\Big(Q_{2,1}^4[m_{1,4}-3] +  Q_{2,2}^4[m_{2,4}-3]\Big) ,
\end{align}
Consequently, for the first estimate we get
\begin{align}
\omega_{1} &= \tau\Big(u_1^{(j)}Q_{1,1}^3(m_{1,4}-3) + u_2^{(j)}Q_{1,2}^3(m_{2,4}-3)\Big) \notag \\
&\qquad - \frac{x_{1}^{(j)}\tau^2}{2}W(Q_{1,1}, Q_{1,2}) -x_{1}^{(j)}\tau\Big( Q_{1,1}^4(m_{1,4} - 3) + Q_{1,2}^4(m_{2,4}-3)\Big).
\end{align}
Combining these results, we arrive at the following system of coupled ODEs for $f(x)= +x^3$:
\begin{align}
\frac{d}{dt} Q_{1,1}
&= \tau Q_{1,1}^3(m_{1,4}-3) - \frac{Q_{1,1}\tau^2}{2}W(Q_{1,1}, Q_{1,2}) \notag \\ & \qquad - \tau Q_{1,1}\Big(Q_{1,1}^4(m_{1,4} - 3) + Q_{1,2}^4(m_{2,4}-3)\Big) ,\\
\frac{d}{dt} Q_{1,2}
&=\tau Q_{1,2}^3(m_{2,4}-3) - \frac{Q_{1,2}\tau^2}{2}W(Q_{1,1}, Q_{1,2}) \notag \\ & \qquad - \tau Q_{1,2}\Big(Q_{1,1}^4(m_{1,4} - 3) + Q_{1,2}^4(m_{2,4}-3)\Big) ,\\
\frac{d}{dt} Q_{2,1}
&= \tau Q_{2,1}^3 (m_{1,4}-3)
   - \frac{\tau^2}{2} Q_{2,1}\, W(Q_{2,1}, Q_{2,2})
   \notag \\ & \qquad - \tau^2 Q_{1,1} F(Q_{1,1}, Q_{2,1}, Q_{1,2}, Q_{2,2}) \notag \\
&\qquad
   - \tau Q_{1,1}\, D(Q_{1,1}, Q_{2,1}, Q_{1,2}, Q_{2,2}) \notag \\ & \qquad -\tau Q_{2,1}\Big(Q_{2,1}^4(m_{1,4}-3) +  Q_{2,2}^4(m_{2,4}-3)\Big),\\ 
\frac{d}{dt} Q_{2,2}
&= \tau\big(Q_{2,2}^3 (m_{2,4}-3)\big)
   - \frac{\tau^2}{2} Q_{2,2}\, W(Q_{2,1}, Q_{2,2}) \notag \\ & \qquad
   - \tau^2 Q_{1,2}\, F(Q_{1,1}, Q_{2,1}, Q_{1,2}, Q_{2,2}) \notag\\
&\qquad
   - \tau Q_{1,2}\, D(Q_{1,1}, Q_{2,1}, Q_{1,2}, Q_{2,2}) \notag\\ & \qquad -\tau Q_{2,2}\Big(Q_{2,1}^4(m_{1,4}-3) +  Q_{2,2}^4(m_{2,4}-3)\Big).
\end{align}

Consequently, it follows directly for $f(x)= \pm x^3$:
\begin{align}\label{app:ODE_CUBIC}
\frac{dQ_{1,1}}{dt} 
&=  \pm \tau Q_{1,1}^3(m_{1,4}-3) - \frac{Q_{1,1}\tau^2}{2}W(Q_{1,1}, Q_{1,2}) \notag\\ & \qquad \mp \tau Q_{1,1}\Big(Q_{1,1}^4(m_{1,4} - 3) + Q_{1,2}^4(m_{2,4}-3)\Big) ,\\
\frac{dQ_{1,2}}{dt} 
&=\pm \tau Q_{1,2}^3(m_{2,4}-3) - \frac{Q_{1,2}\tau^2}{2}W(Q_{1,1}, Q_{1,2}) \notag\\& \qquad \mp \tau Q_{1,2}\Big(Q_{1,1}^4(m_{1,4} - 3) + Q_{1,2}^4(m_{2,4}-3)\Big),\\
\frac{dQ_{2,1}}{dt} 
&= \pm \tau Q_{2,1}^3 (m_{1,4}-3)
   - \frac{\tau^2}{2} Q_{2,1}\, W(Q_{2,1}, Q_{2,2}) \notag\\ & \qquad
   - \tau^2 Q_{1,1}\, F(Q_{1,1}, Q_{2,1}, Q_{1,2}, Q_{2,2}) \notag\\
&\qquad
   \mp \tau Q_{1,1}\, D(Q_{1,1}, Q_{2,1}, Q_{1,2}, Q_{2,2}) \notag\\ & \qquad\mp \tau Q_{2,1} \Big(Q_{2,1}^4(m_{1,4}-3) +  Q_{2,2}^4(m_{2,4}-3)\Big) ,\\
\frac{dQ_{2,2}}{dt} 
&= \pm \tau Q_{2,2}^3 (m_{2,4}-3)
   - \frac{\tau^2}{2} Q_{2,2}\, W(Q_{2,1}, Q_{2,2}) \notag\\ & \qquad
   - \tau^2 Q_{1,2}\, F(Q_{1,1}, Q_{2,1}, Q_{1,2}, Q_{2,2}) \notag \\
&\qquad
   \mp \tau Q_{1,2}\, D(Q_{1,1}, Q_{2,1}, Q_{1,2}, Q_{2,2}) \notag\\ & \qquad \mp \tau Q_{2,2} \Big(Q_{2,1}^4(m_{1,4}-3) +  Q_{2,2}^4(m_{2,4}-3)\Big).
\end{align}

\section{Steady state analysis}\label{app:steadystate}

\subsection{Learnability boundary} \label{app:boundaryproof}

For the cubic branch, the steady-state condition $\dot{\bm{Q}}_t = 0$ in \eqref{app:ODE_CUBIC} yields a seventh-order polynomial in $Q_{j,j}$. By substituting $q = Q_{j,j}^2$ for the diagonal entries---consistent with the first regime---the dynamics reduce to the form $\dot{Q}_{j,j} = Q_{j,j} P(q)$, where $P(q) = \xi q^3 + \eta  q^2 + \zeta q + \varpi$ is a cubic polynomial with $\varpi =15$. A non-trivial steady state, representing the learning phase where $\bm{Q} \neq 0$, exists if and only if $P(q) = 0$ admits at least one positive real root.

For a cubic polynomial $ P(q) = \xi q^3 + \eta  q^2 + \zeta q +  \varpi = 0$  where $\varpi  > 0 $ , to have at least one positive root, we have to consider two cases : $\xi > 0 $ vs. $\xi < 0$. 
 \paragraph{Case1 ;  $\xi > 0 $}

Since $P(0) = \varpi > 0$ and $\lim_{q \to \infty} P(q) = \infty$, the curve begins positive and ends positive. For a positive real root to exist, the local minimum must dip below the q-axis. The stationary points are found where $P'(q) = 0$:
    \begin{equation}
        3\xi q^2 + 2\eta q + \zeta = 0 \implies q_{\text{min}} = \frac{-\eta  + \sqrt{\eta ^2 - 3\xi\zeta}}{3\xi} \ .
    \end{equation}
    (Requires $\eta ^2 - 3\xi\zeta > 0$ for real turning points).

    For roots to exist, the value at the minimum must be non-positive:
    \begin{equation}
        P(q_{\text{min}}) \leq 0 \ .
    \end{equation}
    Substituting $q_{\text{min}}$ into $P(q)$ yields the condition that the cubic discriminant $\Delta$ must be non-negative.
    \begin{equation}
        \Delta =         \eta ^2 \zeta^2 - 4\xi\zeta^3 - 4\eta ^3\varpi - 27\xi^2\varpi^2 + 18\xi \eta \zeta\varpi \geq 0 \ .
    \end{equation}
We also require that this real positive solution lies between our physical meaningful space.  We analyze this by transforming the system into the reciprocal function $y = 1/q$.

    Let $R(y) = q^3 P(1/q)$. The coefficients reverse order:
    \begin{equation}
        R(y) =\varpi y^3 + \zeta y^2 + \eta  y + \xi \ .
    \end{equation}

    The slope of the reciprocal function is
    \begin{align}
        R'(y) &= 3\varpi y^2 + 2\zeta y + \eta  , \\
    R'(1) &= 3\varpi + 2\zeta + \eta  \ .
    \end{align}

    The condition $3\varpi + 2\zeta + \eta  < 0$ implies $R'(1) < 0$.
    Since $\varpi> 0$, $R(y) \to +\infty$ as $y \to \infty$. A negative slope at $y=1$ indicates that the local dip (and thus the root $y^*$) must occur at a value greater than 1:
    \begin{equation}
        y^* > 1 \implies \frac{1}{q^*} > 1 \implies q^* < 1 \ .
    \end{equation}
    Thus, the condition defines the region where the positive root is confined to the interval $(0, 1)$.

\paragraph{Case 2: $\xi< 0$}
In this regime, since $P(0) = \varpi > 0$ and $\lim_{q \to \infty} P(q) = -\infty$, the Intermediate Value Theorem guarantees the existence of at least one positive real root $q^* > 0$. 

To ensure this root lies within the meaningful interval $q^* \in (0, 1)$, we examine the reciprocal polynomial. Since $R(y) \to \infty$ as $y \to \infty$, the condition $R(1) < 0$ is sufficient to ensure that the root $y^*$ (where $R(y^*) = 0$) satisfies $y^* > 1$, which implies $q^* < 1$.

The condition $R(1) < 0$ translates to:
\begin{equation}
    \varpi + \zeta+ \eta  + \xi < 0
\end{equation}
However, \textbf{Case 2 is physically inadmissible in our context}. Given that $\xi+\eta +\zeta+\varpi = (m_6 - 15) + 15 < 0$, it follows that $m_6 < 0$. Since the sixth moment $m_6$ represents an even power expectation, it is non-negative by definition ($m_6 \geq 0$). This contradiction eliminates Case 2, leaving Case 1 as the sole determinant of our learnability boundary.

\subsection{Competition boundary} \label{app:subsec_sompeition_boundary}
To investigate the effect of initialization on competition dynamics, we begin by analyzing the dynamics of the first row of $\bm{Q} \in \mathbb{R}^{2\times2}$, namely $Q_{1,1}$ and $Q_{1,2}$, as these govern the evolution of the second row ($Q_{2,1}$ and $Q_{2,2}$). The first row satisfies the ODE system:
\begin{align}
    \frac{dQ_{1,1}}{dt} &=   \tau\Big( \langle c_1\gamma_1 \rangle- Q_{1,1}\langle \gamma_1'\rangle\Big) - \frac{\tau^2}{2}Q_{1,1}\langle \gamma_1^2\rangle \notag \\ & \qquad - \tau Q_{1,1} \Big(Q_{1,1}(\langle c_1\gamma_1 \rangle - Q_{1,1}\langle \gamma_1^{'} \rangle) + Q_{1,2}(\langle c_2\gamma_1 \rangle - Q_{1,2}\langle \gamma_1^{'} \rangle)\Big),
\end{align}
\begin{align}
    \frac{dQ_{1,2}}{dt} &= \tau\Big( \langle c_2\gamma_1 \rangle- Q_{1,2}\langle \gamma_1'\rangle\Big) - \frac{\tau^2}{2}Q_{1,2}\langle \gamma_1^2\rangle \notag \\ & \qquad - \tau Q_{1,2} \Big(Q_{1,1}(\langle c_1\gamma_1 \rangle - Q_{1,1}\langle \gamma_1^{'} \rangle) + Q_{1,2}(\langle c_2\gamma_1 \rangle - Q_{1,2}\langle \gamma_1^{'} \rangle)\Big).
\end{align}

To understand the boundary between the basins of attraction for the two components, we analyze the relative growth rates of the order parameters $Q_{1,1}$ and $Q_{1,2}$. Dividing the update equations for $\dot{Q}_{1,1}$ and $\dot{Q}_{1,2}$ by $Q_{1,1}$ and $Q_{1,2}$ respectively yields the logarithmic derivatives. Physically, the term $\frac{\dot{Q}}{Q}$ represents the relative growth rate of the component, rather than its absolute speed:
\begin{equation}
    \frac{d}{dt} \ln(Q_{1,i}) = \frac{\dot{Q}_{1,i}}{Q_{1,i}}.
\end{equation}
This normalization allows us to compare the competitive advantage of one component over the other, independent of their current magnitudes. To find the separatrix, we examine the evolution of the ratio between the two components. We compute the difference between their relative growth rates:
\begin{equation}
    \frac{d}{dt} \ln\left(\frac{Q_{1,2}}{Q_{1,1}}\right) = \frac{\dot{Q}_{1,2}}{Q_{1,2}} - \frac{\dot{Q}_{1,1}}{Q_{1,1}}.
\end{equation}
Substituting the explicit expressions for $\frac{\dot{Q}_{1,1}}{Q_{1,1}}$ and $\frac{\dot{Q}_{1,2}}{Q_{1,2}}$:
\begin{align}
    \frac{\dot{Q}_{1,1}}{Q_{1,1}} &= \tau \Big( \frac{\langle c_1 \gamma_1\rangle}{Q_{1,1}}-\langle \gamma_1^\prime \rangle\Big) \notag \\ & \qquad -\tau\Big(Q_{1,1}(\langle c_1\gamma_1 \rangle - Q_{1,1}\langle \gamma_1^{'} \rangle) + Q_{1,2}(\langle c_2\gamma_1 \rangle - Q_{1,2}\langle \gamma_1^{'} \rangle)\Big) - \frac{\tau^2}{2}\langle \gamma_1^2\rangle ,\\
    \frac{\dot{Q}_{1,2}}{Q_{1,2}} &= \tau \Big(\frac{\langle  c_2 \gamma_1\rangle}{Q_{1,2}} - \langle \gamma_1^\prime \rangle\Big) \notag \\ & \qquad- \tau \Big(Q_{1,1}(\langle c_1\gamma_1 \rangle - Q_{1,1}\langle \gamma_1^{'} \rangle) + Q_{1,2}(\langle c_2\gamma_1 \rangle - Q_{1,2}\langle \gamma_1^{'} \rangle)\Big) - \frac{\tau^2}{2}\langle \gamma_1^2\rangle.
\end{align}
Subtracting these two equations eliminates the second and the third term, yielding
\begin{equation}
    \frac{\dot{Q}_{1,2}}{Q_{1,2}} - \frac{\dot{Q}_{1,1}}{Q_{1,1}} = \tau \left( \frac{\langle c_2\gamma_1 \rangle}{Q_{1,2}} - \frac{\langle c_1 \gamma_1\rangle}{Q_{1,1}} \right).
\end{equation}
The separatrix is defined as the boundary where the ratio between the components is stationary ($\frac{d}{dt} (Q_{1,2}/Q_{1,1}) = 0$). Setting the difference to zero yields
\begin{equation}
    \tau \left( \frac{\langle c_2\gamma_1 \rangle}{Q_{1,2}} - \frac{\langle c_1 \gamma_1\rangle}{Q_{1,1}} \right) = 0 \implies \frac{\langle c_2\gamma_1 \rangle}{Q_{1,2}} =\frac{\langle c_1 \gamma_1\rangle}{Q_{1,1}}.
\label{sep_line}
\end{equation}
This equation remains valid for an arbitrary nonlinearity, as each $\langle c_i\gamma_1\rangle$ is governed by its respective function $\gamma(\cdot)$. Consequently, this relationship defines the separatrix of the system. For cubic nonlinearity $f(x) = \pm x^3$, we can substitute our result in Equation \eqref{cgamma_i}, obtaining
\begin{align}
    \frac{Q_{1,1}^3(m_{1,4}-3)}{Q_{1,1}} = \frac{Q_{1,2}^3(m_{2,4}-3)}{Q_{1,2}},
\end{align}

Finally, we arrive at the following line equation for the competition boundary: 
\begin{align} 
    Q_{1,2} = \pm \sqrt{\frac{(m_{1,4}-3)}{(m_{2,4}-3)}}Q_{1,1}.
\end{align}

Any deviation from this line creates a positive feedback loop where the dominant component grows faster, driving the system away from the line and towards one of the stable fixed points. Thus, this line acts as the separatrix between the two outcomes.

\subsection{Further analysis of the decoupled regime}
In the decoupled initialization regime with $p=2$ and cubic nonlinearity, we further analyze the stability of the fixed points using Jacobian analysis, numerically. In this regime, Jacobian matrix is diagonal with all off-diagonal entries equal to zero. Thus, the eigenvalues are simply the partial derivatives of the functions governing the ODEs, evaluated at the fixed points. Choosing a representative entry $Q_{1,1}$ we evaluate the partial derivative of the governing function ($h_{1,1}  = \dot{Q}_{1,1}$) for the first component:
\begin{align}
\frac{1}{\tau}\frac{\partial h_{1,1}}{\partial Q_{1,1}} &= -3Q_{1,1}^2 (m_{1,4} - 3 )- \frac{\tau}{2} \left( W(Q_{1,1}, 0) + Q_{1,1}\left. \frac{\partial W(Q_{1,1},Q_{1,2})}{\partial Q_{1,1}} \right|_{Q_{1,2} = 0}\right)\notag \\ & \qquad + 5Q_{1,1}^4(m_{1,4}-3) \ ,
\end{align}
where 
\begin{align}
\frac{\partial W(Q_{1,1},Q_{1,2})}{\partial Q_{1,1}}\Bigg|_{Q_{1,2} = 0} =  - 90 Q_{1,1}^{5} m_{1,4} + 6 Q_{1,1}^{5} m_{1,6} + 180 Q_{1,1}^{5} + 60 Q_{1,1}^3 m_{1,4} - 180 Q_{1,1}^3 ,
\end{align}
and
\begin{align}
 W(Q_{1,1}, 0) &=Q_{1,1}^6m_{1,6} + 15(1-Q_{1,1}^2)^3  + 15Q_{1,1}^4(1-Q_{1,1}^2)m_{1,4} \notag \\ & \qquad
+ 45(1-Q_{1,1}^2)^2(Q_{1,1}^2).
\end{align}

Finally, for the partial derivative expression we arrive at 
\begin{align}
\frac{1}{\tau}\frac{\partial h_{1,1}}{\partial Q_{1,1}} &= Q_{1,1}^6 \frac{7\tau}{2} (-m_{1,6} +15 m_{1,4} - 30 ) + Q_{1,1}^4 (m_{1,4} - 3)  (-\frac{75\tau}{2}   + 5 )  \notag \\ & \qquad- 3Q_{1,1}^2(m_{1,4}-3) - \frac{15\tau}{2} \ .
\end{align}
We numerically determine the fixed points by identifying the roots where the evolution equations vanish. To assess stability, we evaluate the sign of the partial derivative $\frac{\partial h_{1,1}}{\partial Q_{1,1}}$ at these points. Consequently, a fixed point is stable if
\begin{equation}
    \left. \frac{\partial h_{1,1}}{\partial Q_{1,1}} \right|_{Q_{1,1}^*} < 0 \ .
\end{equation}

\section{Experimental details: ICA for hyperspectral remote sensing }\label{app:experimental}
\subsection{Pre-processing}
We utilize the Indian Pines hyperspectral dataset, denoted as a tensor $\mathcal{Y} \in \mathbb{R}^{H \times W \times B}$, where $B=224$. We use the \emph{corrected} version of the dataset, in which water absorption and high-noise spectral bands have been removed to ensure signal quality, resulting in an spectral dimension of $n=200$. The data are flattened into a matrix $\bm{Y}_{\text{raw}} \in \mathbb{R}^{N \times n}$, where $N$ is the total number of pixels, and centered by subtracting the mean, yielding $\bm{Y}_c$. To whiten the data, we compute the empirical covariance matrix $\bm{\Sigma} = \frac{1}{N} \bm{Y}_{\text{c}}^\top \bm{Y}_{\text{c}}$ and perform Singular Value Decomposition (SVD):
\begin{equation}
    \bm{\Sigma} = \bm{V}_s \bm{S} \bm{V}_s^\top.
\end{equation}
where $\bm{V}_s$ is an orthogonal matrix containing the singular vectors (representing the principal components of the data), and $\bm{S}$ is a diagonal matrix of the corresponding singular values (representing the variance along each component). The whitening matrix is constructed as $\bm{W} = \mathbf{S} ^{-1/2} \bm{V}_s^\top$. The whitened dataset is obtained via the linear transformation $\tilde{\bm{Y}} = \mathbf{Y}_{\text{c}} \bm{W}^\top$, ensuring that the covariance of $\tilde{\bm{Y}}$ is the identity matrix $\bm{I}$.

We derive a set of ground truth vectors $\bm{u}_i$ from the labeled hyperspectral classes. Let $\mathcal{C}_i$ denote the set of pixel indices belonging to class $i$, and $\tilde{\bm{y}}_l$ denote $l-th$ row of the matrix $\bm{\tilde{Y}}$. We compute the class centroids in the whitened space:
\begin{equation}
    \bm{u}_i = \frac{1}{|\mathcal{C}_i|} \sum_{l \in \mathcal{C}_i} \tilde{\mathbf{y}}_l.
\end{equation}
Since raw class centroids may be correlated, we enforce orthogonality to satisfy the ICA independence assumption. For this, we construct the matrix $\bm{U} = [\bm{u}_1, \dots, \bm{u}_p]$ and apply QR decomposition. The resulting orthogonal components are then re-scaled to maintain a norm of $\sqrt{n}$. Having established the independent ground truth components, we construct the data as described in~\eqref{data_generation}. For $p=2$, we select Classes~6 and~13, corresponding to Grass Trees and Wheat in the Indian Pines scene, as our two ground-truth components. In the setting of Example~\ref{example:cubic}, we set $\beta_1 = 1$ for Class 13 and $\beta_2 = 0.6$ for Class 6, with learning rate $\tau = 0.01$.

\subsection{Post-processing}
After obtaining the learned components which we reshape back to the spatial grid to yield activation maps $\mathbf{A} \in \mathbb{R}^{H \times W \times 2}$. Since online ICA recovers the independent components only up to permutation and sign, we align each learned component with a ground-truth class by computing a matching score against the centroid-derived targets. To enable qualitative assessment of recovery, the activation maps are plotted using a globally shared color scale. The corresponding ground-truth binary mask is then displayed as a localized inset overlaid directly on the learned activation map to facilitate direct visual comparison.

\section{Competition in other orthogonalization schemes and nonlinearities} \label{app:sym_orth}

\begin{figure}[h]
  \centering
  \begin{subfigure}[t]{0.33\textwidth}
    \centering
    \includegraphics[width=\linewidth]{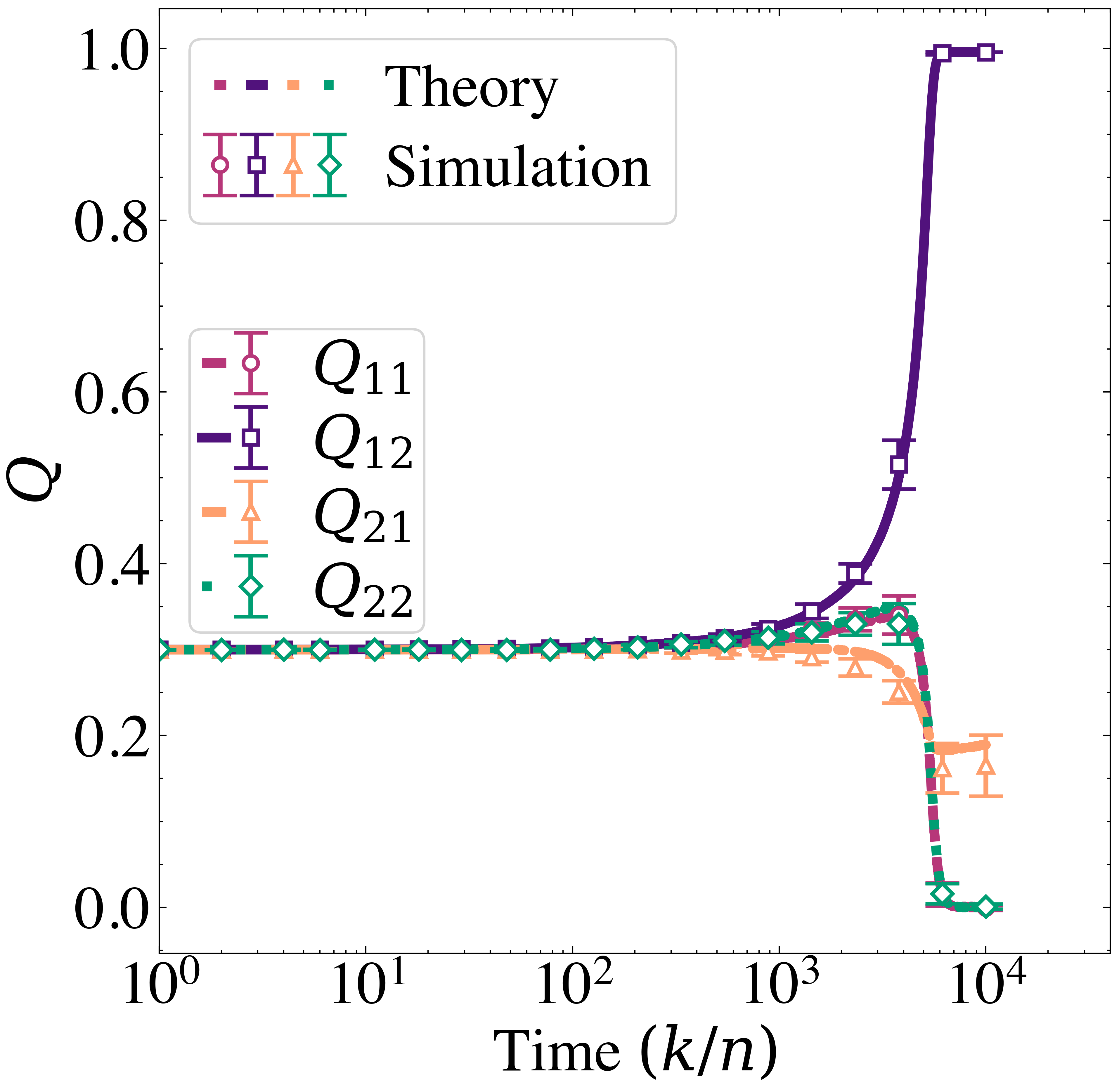}
    \caption{$f(x) =\tanh(x)$ with GS} 
    \label{fig:tanh}
  \end{subfigure}\hfill
    \begin{subfigure}[t]{0.33\textwidth}
    \centering
    \includegraphics[width=\linewidth]{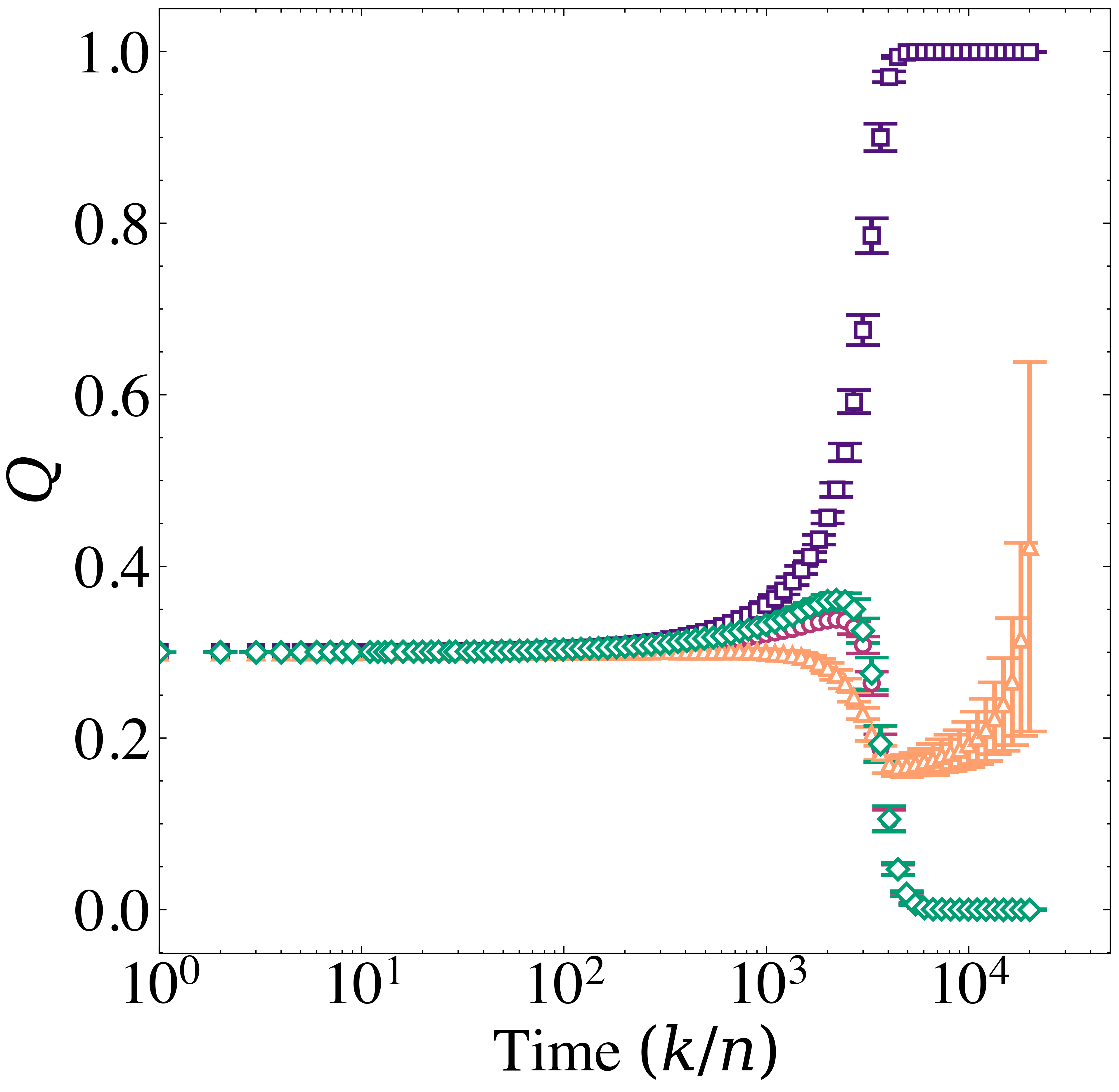}
    \caption{QR with Householder (sequential)} 
    \label{fig:QR}
  \end{subfigure}\hfill
  \begin{subfigure}[t]{0.33\textwidth}
    \centering
    \includegraphics[width=\linewidth]{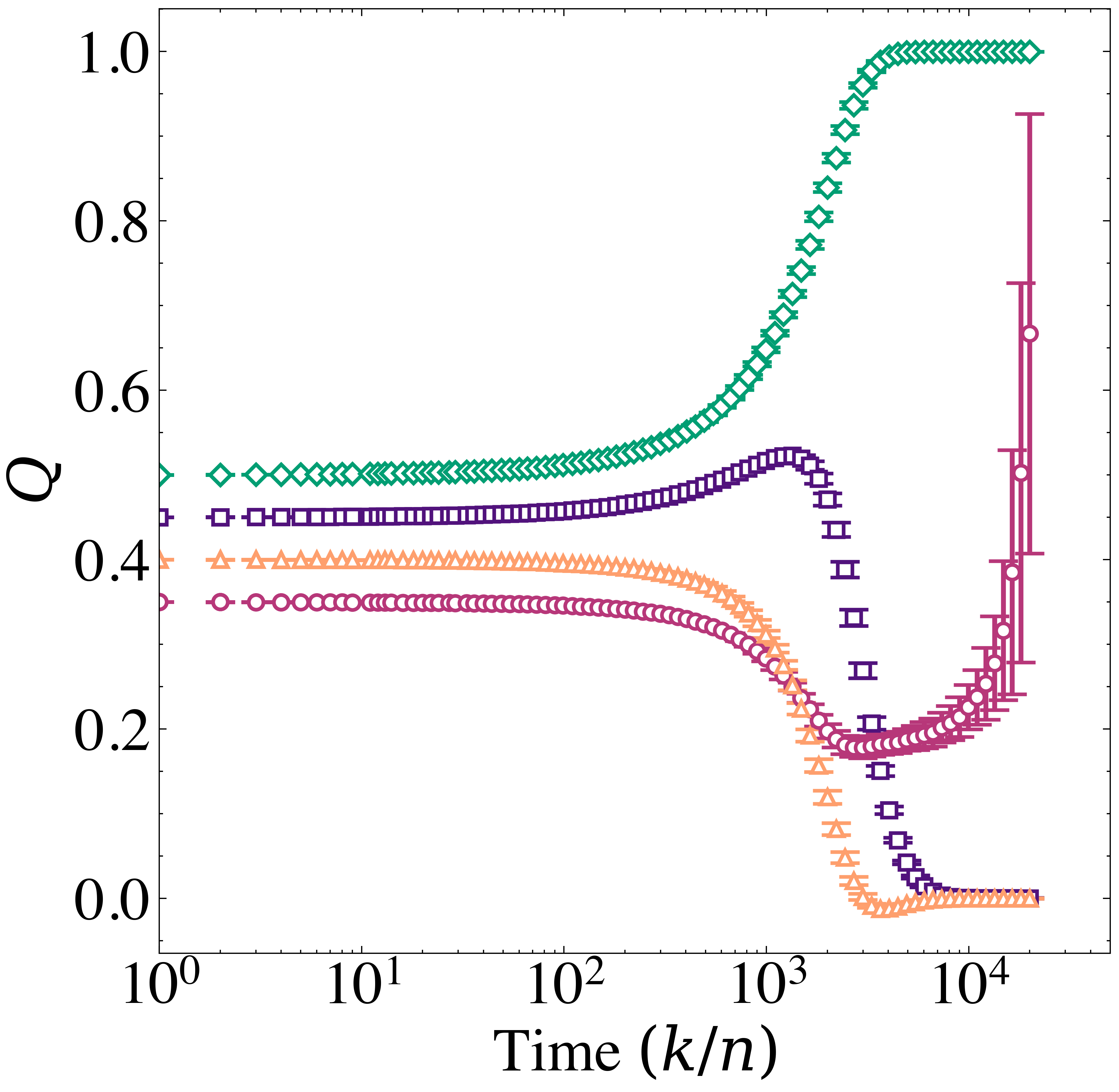}
    \caption{Löwdin via polar (symmetric)} 
    \label{fig:löwdin}
  \end{subfigure}\hfill
\caption{Competition is also visible across other nonlinearities and orthogonalization schemes. In a) our limiting ODE and the simulations are plotted with $f(x) = \tanh(x)$ with Gram-Schmidt process, with $\tau = 0.01$. In b), QR via Householder transformations, with $ \tau = 0.001, Q_{0,i,j} = 0.3 $ In c) Löwdin orthogonalization was used with $\tau = 0.001$ and $Q_{0,1,1}, Q_{0,1,2},Q_{0,2,1},Q_{0,2,2} =  0.35,0.45,0.4,0.5$. In all of the above simulations $n=1000,\beta_1 = 0.2, \beta_2 = 1$ were used, with 10 Monte Carlo averages, $2$ standard deviation error bars for simulations.} 
\end{figure}

The identified regimes of decoupling and competition are not unique to specific configurations; rather, they are observable across other orthogonalization schemes and nonlinearities. In this section, we provide supplementary plots of the learning dynamics to demonstrate that our results are not restricted to Gram-Schmidt orthogonalization or to cubic nonlinearities. 

Orthogonality may be enforced through several methodologies: sequential schemes such as QR decomposition (utilizing Householder or Gram-Schmidt), symmetric approaches via Löwdin (polar) orthogonalization, Cayley-transform updates, or Riemannian retractions on the Stiefel manifold.\cite{lodvin,Edelman1998,WenYin2013,Boumal2023}. 
We simulate two orthogonalization schemes, specifically QR with Householder, which is a sequential method, also Löwdin orthogonalization, where estimate vectors are orthogonalized concurrently rather than sequentially as in the QR approach. In Figure \ref{fig:QR}, and Figure \ref{fig:löwdin} we can clearly observe the competition between components, as one of the estimates is first forced to unlearn, then eventually wins the competition. 

Next, given that our Theorem \ref{theorem} holds for any nonlinearity satisfying the underlying assumptions, we can evaluate other score functions and directly compare theoretical predictions with empirical simulations. Using Equation \eqref{eq:Q_ODE}, we analyze the case where $f(x) = \tanh(x)$. Under the competition regime initializations, we obtain the results shown in Figure \ref{fig:tanh}. We observe strong agreement between our theoretical ODE trajectories and empirical Monte Carlo simulations, further validating that our theory and regime identification hold for arbitrary nonlinearities. 

\end{document}